\documentclass{article}

\usepackage[final]{neurips_2024}

\usepackage[utf8]{inputenc} 
\usepackage[T1]{fontenc}    
\usepackage{hyperref}       
\usepackage{url}            
\usepackage{booktabs}       
\usepackage{amsfonts}       
\usepackage{nicefrac}       
\usepackage{microtype}      
\usepackage{xcolor}         

\usepackage{amsthm, thm-restate}

\newtheorem{lemma}{Lemma}

\usepackage{wrapfig}

\newtheorem{definition}{Definition}
\usepackage[noend]{algpseudocode}
\usepackage{subfigure}
\usepackage{amsfonts}
\usepackage{algorithm}
\usepackage{amsmath}
\usepackage{graphicx}
\usepackage{amssymb}
\usepackage{listings}
\usepackage{xcolor}

\usepackage{tabularx, makecell, booktabs}

\usepackage{mathtools}

\usepackage{siunitx}

\newcommand\res[3][round-precision=1]{
    \SI[round-mode=places, scientific-notation=fixed, fixed-exponent=0, output-decimal-marker={.},#1]{#2}{}%
    \pm%
    \SI[round-mode=places, scientific-notation=fixed, fixed-exponent=0, output-decimal-marker={.},#1]{#3}{}
}

\definecolor{codegreen}{rgb}{0,0.6,0}
\definecolor{codegray}{rgb}{0.5,0.5,0.5}
\definecolor{codepurple}{rgb}{0.58,0,0.82}
\definecolor{backcolour}{rgb}{0.95,0.95,0.92}

\lstdefinestyle{mystyle}{
    backgroundcolor=\color{backcolour},   
    commentstyle=\color{codegreen},
    keywordstyle=\color{magenta},
    numberstyle=\tiny\color{codegray},
    stringstyle=\color{codepurple},
    basicstyle=\ttfamily\footnotesize,
    breakatwhitespace=false,         
    breaklines=true,                 
    captionpos=b,                    
    keepspaces=true,                 
    numbers=left,                    
    numbersep=5pt,                  
    showspaces=false,                
    showstringspaces=false,
    showtabs=false,                  
    tabsize=2
}

\usepackage{xspace}
\makeatletter
\DeclareRobustCommand\onedot{\futurelet\@let@token\@onedot}
\def\@onedot{\ifx\@let@token.\else.\null\fi\xspace}

\def\iid{{i.i.d}\onedot}
\def\eg{{e.g}\onedot} 
\def\ie{{i.e}\onedot}

\makeatother

\newcommand{\Method}{Banded Square Root Factorization\xspace}
\newcommand{\method}{banded square root factorization\xspace}
\newcommand{\SR}{SR\xspace}
\newcommand{\acronym}{BSR\xspace}
\newcommand{\AOF}{AOF\xspace}

\newcommand{\AOFfull}{approximately optimal factorization\xspace}

\newcommand{\E}{\mathbb{E}}
\newcommand{\Ecal}{\mathcal{E}}

\newcommand{\R}{\mathbb{R}}

\newcommand{\sens}{\operatorname{sens}}
\newcommand{\diag}{\operatorname{diag}}
\newcommand{\Fr}{F}
\newcommand{\trace}{\operatorname{trace}}
\newcommand{\Id}{\operatorname{Id}}

\renewcommand{\b}[1]{{|#1|}} 
\newcommand{\toep}{\operatorname{LDToep}}

\renewcommand{\paragraph}[1]{\noindent\textbf{#1}\quad}

\title{Banded Square Root Matrix Factorization for Differentially Private Model Training} 

\author{%
  Nikita  Kalinin  \\
  Institute of Science and Technology (ISTA)\\
  Klosterneuburg, Austria \\
  \texttt{nikita.kalinin@ist.ac.at} \\
   \And
   Christoph Lampert \\
   Institute of Science and Technology (ISTA) \\
   Klosterneuburg, Austria \\
   \texttt{chl@ist.ac.at} \\
}

\begin{document}
\renewcommand{\proofname}{Proof sketch}

\maketitle

\begin{abstract}
    Current state-of-the-art methods for differentially private model 
    training are based on matrix factorization techniques. 
    However, these methods suffer from high computational overhead 
    because they require numerically solving a demanding optimization 
    problem to determine an approximately optimal factorization prior 
    to the actual model training. 
    In this work, we present a new matrix factorization approach, 
    \acronym, which overcomes this computational bottleneck. 
    By exploiting properties of the standard matrix square root,
    \acronym allows to efficiently handle also large-scale problems. 
    For the key scenario of stochastic gradient descent with momentum 
    and weight decay, we even derive analytical expressions for \acronym
    that render the computational overhead negligible. 
    We prove bounds on the approximation quality that hold both in the 
    centralized and in the federated learning setting.
    Our numerical experiments demonstrate that models trained 
    using \acronym perform on par with the best existing methods,
    while completely avoiding their computational overhead. 
\end{abstract}

\section{Introduction}
We study the problem of \emph{differentially private (DP) model training with 
stochastic gradient descent (SGD)} in the setting of either federated or 
centralized learning. 
This task has recently emerged as one of the most promising ways to train powerful
machine learning models but nevertheless guarantee the privacy of the used data,
which led to a number of studies, both theoretical as well as application-driven~\citep{Abadi,yugradient,zhang2021adaptive,Kairouz,Denisov}.
The state of the art in the field are approaches based on the \emph{matrix factorization (MF) mechanism}~\citep{Li_MF,Henzinger}, which combines theoretical guarantees with practical applicability~\citep{BandedMatrix,MultiEpoch,PAMM, CNPBINDPL}\footnote{Note that this 
specific type of MF should not be confused with other occurrences of matrix factorization 
in, potentially private, machine learning, such as in recommender systems~\citep{shin2018privacy,li2021federated}.}
It is based on the observation that all iterates of SGD are 
simply linear combinations of model gradients, which are 
computed at intermediate time steps. 
Consequently, the iterates can be written formally as the result of 
multiplying the matrix of coefficients, called \emph{workload matrix}, 
with the row-stacked gradient vectors.  
To preserve the privacy of the training data in this process 
one adds suitably scaled Gaussian noise at intermediate 
steps of the computation. 
The MF mechanism provides a way to select the noise 
covariance structure based on a factorization of the workload 
matrix into two matrices.

Identifying the minimal amount of noise necessary to achieve 
a desired privacy level requires solving an optimization 
problem over all possible factorizations, subject to 
\emph{data participation constraints}.
%
For some specific settings, the optimal solutions have 
been characterized: for \emph{streaming learning}, 
when each data batch contributes at most once to the 
gradients, \citet{Li_MF} presented a formulation of 
this problem as a semi-definite program. 
\citet{Henzinger} proved that a square root factorization 
of the workload matrix is asymptotically optimal for different 
linear workloads, including continual summation and decaying 
sums. 

In this work, our focus lies on the settings that are most 
relevant to machine learning tasks: workload matrices that 
reflect SGD-like optimization, and participation schemes in 
which each data batch can potentially contribute to more than 
one gradient vector, as it is the case for standard multi-epoch 
training. 
Assuming that this happens at most once every $b$ steps, for 
some value $b\geq 1$, leads to the problem of optimal matrix 
factorization in the context of \emph{$b$-min-separated participation} 
sensitivity. 
Unfortunately, as shown in~\citet{BandedMatrix}, finding the 
optimal matrix factorization in this setting is computationally
intractable. 
Instead, the authors proposed an approximate solution by 
posing additional constraints on the solution set.  
The result is a semi-definite program that is tractable, but 
still has high computational cost, making it practical only
for small to medium-sized problem settings. 

Subsequent work concentrated on improving or better understanding
the factorizations for specific algorithms, such as plain SGD 
without gradient clipping, momentum, or weight decay~\citep{koloskova2023gradient}
or specific, \eg convex, objective functions~\citep{CNPBINDPL}.
Often, streaming data was asssumed, \ie each data item can contribute 
at most to one model update, which is easier to analyze theoretical, 
but further removed from real-world applications~\citep{dvijotham2024efficient}. 
A concurrent line of works has also focused on scaling matrix factorizations for large-scale training. \citet{hasslefree2024} extended the Buffered Linear Toeplitz (BLT) mechanism to support multi-participation in federated learning, improving privacy-utility tradeoffs, while ensuring memory efficiency.  \citet{scalingmckenna2024} improved DP Banded Matrix Factorization's scalability, enabling it to handle millions of iterations and large models efficiently. These advancements enhance the practicality of differentially private training in real-world applications.

Our work aims at a general-purpose solution
that covers as many realistic scenarios as possible. Our
 ultimate goal is to make general-purpose differentially 
private model training as simple and efficient to use as 
currently dominating non-private technique. 
Our main contributions are:
\begin{enumerate}
    \item We introduce a new factorization, 
    \textbf{the \emph{banded squared root (\acronym)}, 
    which is efficiently computable even for large 
    workload matrices} and agnostic to the underlying
    training objective. 
    For matrices stemming from SGD optimization 
    potentially with momentum and/or weight decay, 
    we even provide \textbf{closed form expressions.} 
    \item We provide \textbf{general lower bounds} on the approximate error for any factorization, along with specific \textbf{upper and lower bounds for the BSR}, Square Root, and baseline factorizations, in the contexts of both single participation (streaming)  and \textbf{repeated participation} (e.g., multi-epoch) training.
    \item We demonstrate experimentally that 
    \textbf{\acronym's approximation error is 
    comparable to the state-of-the-art} method, 
    and that \textbf{both methods also perform 
    comparably in real-world training tasks}. 
\end{enumerate}
Overall, the proposed \textbf{\acronym factorization achieves 
training high-accuracy models with provable privacy 
guarantees while staying computationally efficient even 
for large-scale training tasks.}

\section{Background}
Our work falls into the areas of \emph{differentially private (stochastic) optimization},
of which we remind the reader here, following mostly the description of \citet{Denisov} and \citet{BandedMatrix}.
The goal is to estimate a sequence of parameter vectors, $\Theta = (\theta_1, \dots, \theta_{n})\in\R^d$, where each $\theta_i$ is a linear combination of \emph{update vectors}, $x_1,\dots,x_{i}\in\R^d$, 
that were computed in previous steps, typically as gradients of a model with respect to 
training data that is meant to stay private. We assume that all $x_i$ have a bounded norm, $\|x_i\|\leq \zeta$.
Compactly, we write $\Theta=AX$, where the lower triangular \emph{workload matrix} $A$ contains 
the coefficients and $X\in\R^{n\times d}$ is formed by stacking the update vectors as rows.
With different choices of $A$, the setting then reflects many popular first-order optimization 
algorithms, in particular stochastic gradient descent (SGD), potentially with momentum 
and/or weight decay.
Depending on how exactly $x_1,\dots,x_n$ are obtained, the setting can express
different centralized as well as federated training paradigms.
To formalize this aspect, we adopt the concept of \emph{$b$-min-separated participation}~\citep{BandedMatrix}. 
For some integer $b\geq 1$, it states that if a data item (\eg a single training 
example in central training, or a client batch in federated learning) contributed 
to an update $x_i$, the earliest it can contribute again is the update $x_{i+b}$. 
Additionally, let $1\leq k\leq \frac{n}{b}$ be the maximal number any data point can contribute. 
In particular, this notion also allows us to treat in a unified way 
\emph{streaming data} ($b=n$ or $k=1$), as well as \emph{unrestricted access} 
patterns ($k=n$ with $b=1$), but also intermediate settings, such as 
\emph{multi-epoch training} on a fixed-size dataset. 

The \emph{matrix factorization} approach~\citep{Li_MF} adopts a \emph{factorization} $A=BC$ of the workload matrix and computes $\Theta^{\text{MF}} = B(CX+Z)$, where $Z$ is Gaussian noise that is chosen appropriately to make the 
intermediate result $CX+Z$ private to the desired level. 
Algorithm~\ref{alg:cap} shows the resulting algorithm in pseudocode.
It exploits the fact that instead of explicit multiplication by $C$ 
and $B$, standard optimization toolboxes can be employed with suitably 
modified update vectors, because also $\Theta^{\text{MF}}=A(X+C^{-1}Z)$,
and multiplication by $A$ corresponds to performing the optimization.

\begin{algorithm}[t]
\caption{Differentially Private SGD with Matrix Factorization}\label{alg:cap} 
\begin{algorithmic}
\Require Initial model $\theta_0 \in \mathbb{R}^d$, dataset $D$, batchsize $b$, matrix $C \in \mathbb{R}^{n \times n}$, model loss $\ell(\theta,d)$, clipnorm $\zeta$,
noise matrix $Z \in \mathbb{R}^{n \times d}$ with \iid entries $\sim \mathcal{N}(0, s^2)$, where $s=\sigma\sens_{k,b}(C)$.

\For{$i = 1, 2, \dots ,n$}
\State $S_i \leftarrow \{d_1, \dots, d_m\} \subseteq D$\quad select a data batch, respecting the data participation constraints
\State $g_i \leftarrow \nabla_\theta \ell(\theta_{i-1}, d_j)) \quad\text{for $j=1,\dots,m$}$
\State $x_i \leftarrow \sum\nolimits_{j = 1}^{m} \operatorname{clip}_{\zeta}(g_j)$ \quad
\text{where $\operatorname{clip}_{\zeta}(d) = \min(1, \zeta/||d||)d$}
\State $\hat{x}_i \leftarrow x_i + \zeta [C^{-1}Z]_{[i,\cdot]}$
\State $\theta_{i} \leftarrow \text{update}(\theta_{i-1},\hat{x}_i)$, \qquad// SGD model updates
\EndFor
\Ensure $\Theta=(\theta_1,\dots,\theta_n)$
\end{algorithmic}
\end{algorithm}

Different factorizations recover different algorithms from the literature. For 
example, $B=A$, $C=\Id$ recovers DP-SGD~\citep{Abadi}, where noise is added 
directly to the gradients.
Conversely, $B=\Id$, $C=A$ simply 
adds noise to each iterate of the 
optimization~\citep{dwork2006calibrating}. 
However, better choices than these baselines are possible, in the 
sense that they can guarantee the same levels of privacy with less 
added noise, and therefore potentially with higher retained accuracy. 
The reason lies in the fact that $B$ and $C$ play different roles:
$B$ acts as a \emph{post-processing} operation of already private 
data. Hence, it has no further effect on privacy, but it 
influences to what amount the added noise affects the expected
error in the approximation of $\Theta$. Specifically, for 
$Z\sim\mathcal{N}(0;s\Id)$, 
\begin{align}
\E_Z\|\Theta-\Theta^{\text{MF}}\|^2_{\Fr} = 
\E_Z\|BZ\|^2_{\Fr} = s^2\|B\|^2_{\Fr}.
\label{eq:expected_theta}
\end{align}

In contrast, $CX$ is the quantity that is meant to be made private.
Doing so requires noise of a strength proportional to $C$'s \emph{sensitivity}, 
$\sens(C) := \sup\nolimits_{X \sim X'} \|CX - CX'\|_F,$ where the \emph{neighborhood relation}, $X\sim X'$, indicates that the 
two sequences of update vectors differ only in those entries that 
correspond to a single data item\footnote{As proved in \citet[Theorem 2.1]{Denisov}, 
establishing privacy in this \emph{non-adaptive} setting suffices to guarantee 
also privacy in the \emph{adaptive} setting, where the update vectors depend not 
only on the data items but also the intermediate estimates of the model parameters,
as it is the case for private model training.}
As shown in~\citet{BandedMatrix}, in the setting of $b$-min-separated repeated 
participation, it holds that
\begin{align}
\sens_{k,b}(C) &\le \max_{\pi \in \Pi_{k,b}}\sqrt{{\sum\limits_{i, j \in \pi}}\big|(C^\top\!C)_{[i, j]}\big|}, 
\label{eq:sensPi}
\end{align}
where $\Pi_{k,b} = \{\ \pi\subset \{1,\dots,n\} \ :\  |\pi|\leq k \wedge (\{i, j\} \subset \pi \ \Rightarrow\ i=j \ \vee\ |i-j|\geq b)\ \}$, 
is the set of possible $b$-min-separated index sets with at most $k$ participation.
Furthermore, \eqref{eq:sensPi} holds even with equality if all entries of $C^\top\!C$ are non-negative. 

Combining~\eqref{eq:expected_theta} with $s=\sens_{k,b}(C)$ yields a quantitative measure for the quality of a factorization.
\begin{definition}
For any factorization $A=BC$, its \textbf{expected approximation error} is 
\begin{align}
    \mathcal{E}(B,C) &:= \sqrt{\E_Z\|\Theta-\Theta^{\text{MF}}\|^2_{\Fr}/n} =  \frac{1}{\sqrt{n}}\sens_{k,b}(C)\|B\|_{\Fr},
    \label{eq:approxerror}
\end{align}
where the $1/\sqrt{n}$ factor is meant to make the quantity comparable across different problem sizes. 
\end{definition}

The \emph{optimal factorization} by this reasoning would be the one 
of smallest expected approximation error. 
Unfortunately, minimizing~\eqref{eq:approxerror} across all 
factorizations it is generally computationally intractable.
Instead, \citet{BandedMatrix} propose an \emph{approximately 
optimal factorization}.

\begin{definition}\label{def:AOF} 

For a workload matrix $A$, let $S$ be the solution to the 
optimization problem
\begin{align}
\arg\min_{S \in \mathcal{S}_{+}^{n}} \trace[A^\top\!\!A\,S^{-1}] 
\quad\text{subject to \quad $\diag(S)=1$ \ \  and\ \ $S_{[i,j]}=0$\ \ for\ \ $|i-j|\geq b$,}
\label{eq:AOF}
\end{align}
where $\mathcal{S}_{+}^{n} $ is the cone of positive definite $n\times n$ matrices. 
Then, $A=BC$ is called the \textbf{\AOFfull (\AOF)}, 
if $C$ is lower triangular and fulfills $C^\top\! C=S$. 
\end{definition}

The optimization problem \eqref{eq:AOF} is a semi-definite program (SDP), 
and can therefore be solved numerically using standard packages. 
However, this is computationally costly,
and for large problems (\eg $n>5000$) computing the \AOF solution is impractical.
This poses a problem for real-world training tasks, where the number of 
update steps are commonly thousands or tens of thousands.

Solving~\eqref{eq:AOF} itself only approximately can mitigate this problem
to some extent, but as we will discuss in Section~\ref{sec:experiments}, 
this can lead to robustness problems, especially because the recovery of 
$C$ from $S$ in Definition~\ref{def:AOF}, \eg by a Cholesky decomposition, 
tends to be sensitive to numerical errors.

\section{\Method} 
In the following section, we introduce our main contribution: 
a general-purpose factorization for the task of differentially 
private stochastic optimization that can be computed efficiently 
even for large problem sizes.

\begin{definition}[\Method]\label{def:banded_square_root}
Let $A\in\R^{n\times n}$ be a lower triangular workload 
matrix with strictly positive diagonal entries. Then,
we call $A=C^2$ the \textbf{square root factorization (SR)}, 
when $C$ denotes the unique matrix square root 
that also has strictly positive diagonal entries.
Furthermore, for any \emph{bandwidth} $p\in\{1,\dots,n\}$, 
we define the \textbf{\method of bandwidth $p$ ($p$-\acronym)} 
of $A$, as 
\begin{align}
A = B^{\b{p}}C^{\b{p}}
\end{align}
where $C^{\b{p}}$ is created from $C$ by setting all entries 
below the $p$-th diagonal to $0$, 
\begin{align}
C^{\b{p}}_{[i,j]}=\begin{dcases} C_{[i,j]} &\quad\text{if $i-j<p$,} \\ 
0&\quad\text{otherwise.}\end{dcases}
\qquad\qquad\text{and}\quad B^{\b{p}}=A (C^{\b{p}})^{-1}.
\label{eq:method}
\end{align}
\end{definition}

Note that determining the \SR, and therefore any $p$-\acronym, 
is generally efficient even for large workload matrices, 
because explicit recursive expressions exist for computing 
the square root of a lower triangular matrix~\citep{Bjorck1983ASM,deadman2012blocked}.  %

In the rest of this work, we focus on the case where the workload 
matrix stems from \emph{SGD with momentum and/or weight decay}, 
and we show that then even closed form expressions for the entries 
of $C^{\b{p}}$ exist that renders the computational cost negligible.

\subsection{\Method for SGD with Momentum and Weight Decay}
We recall the update steps of SGD with momentum and weight decay:
\begin{align}
\theta_i &= \alpha\theta_{i - 1} - \eta m_i
\qquad\text{for}\qquad
m_i = \beta m_{i - 1} + x_i
\end{align}
where $x_1,\dots,x_n$ are the update vectors, $\eta>0$ is the \emph{learning rate} and $ 0\le\beta<1$ 
is the \emph{momentum strength} and $0<\alpha\leq 1$ is the weight decay parameter. 
Note that our results also hold for $\beta=0$, \ie without momentum, and for $\alpha=1$, 
\ie, without weight decay. 
In line with real algorithms and to avoid degenerate cases, 
we assume $\beta<\alpha$ throughout this work.
The update vectors are typically gradients of the model with respect to its 
parameters, but additional operations such as normalization or clipping might 
have been applied. 

Unrolling the recursion, we obtain an expression for $\theta_i$ as a linear 
combination of update vectors as
\begin{align}
\theta_i = \eta \sum_{j = 1}^{i}x_j \Big(\sum_{k=j}^{i} \alpha^{i-k}\beta^{k-j}\Big).
\end{align}

Consequently, the workload matrix has the explicit form\footnote{The workload matrix $A_{1, \beta}$ without weight decay has been considered before in the work of \citet{Denisov}, but not in the context of square root factorization.} $A=\eta A_{\alpha,\beta}$ for 
\begin{align}\label{eq:A_alpha_beta}
A_{\alpha,\beta} &= \begin{pmatrix}
    a_0 & 0 & 0 & \dots &0\\
    a_1 & a_0 & 0 & \dots & 0\\
    a_2 & a_1 & a_0 & \dots & 0\\
    \vdots & \vdots & \ddots & \ddots & \vdots\\
    a_{n-1} & a_{n-2} & \dots & a_1 & a_0 \\
\end{pmatrix}\quad \text{with $a_j = \sum_{i=0}^j \alpha^{i}\beta^{j-i}= \frac{\alpha^{j+1}-\beta^{j+1}}{\alpha-\beta}$.}
\end{align}

As one can see, $A_{\alpha,\beta}$ is a lower triangular Toeplitz-matrix, so 
it is completely determined by the entries of its first column. 
In the following, we use the notation $\toep(m_1,\dots,m_n)$ to denote 
a lower triangular Toeplitz matrix with first column $m_1,\dots,m_n$, \ie $A_{\alpha,\beta}=\toep(a_0,\dots,a_{n-1})$.

Our first result is an explicit expression for the positive square root of $A_{\alpha,\beta}$ (and thereby its $p$-\acronym).

\begin{restatable}[Square-Root of SGD Workload Matrix]{theorem}{SGDroot}
\label{thm:SGDM_Root}
Let $A_{\alpha,\beta}$ be the workload matrix~\eqref{eq:A_alpha_beta}. Then $A_{\alpha,\beta}=C_{\alpha,\beta}^2$ 
for $C_{\alpha,\beta}=\toep(c_0,\dots,c_{n-1})$, with $c_0=1$ and $c_j = \sum_{i = 0}^{j} \alpha^{j-i} r_{j-i} r_i \beta^{i}$ 
for $j=1,\dots,n-1$ with coefficients $r_i = \big|\binom{-1/2}{i}\big|$.
For any $p\in\{1,\dots,n\}$, the $p$-banded \acronym matrix $C^{\b{p}}_{\alpha,\beta}$ is 
obtained from this by setting all coefficients $c_j=0$ for $j\geq p$.
\end{restatable}

\begin{proof}
The proof be found in Appendix~\ref{sec:SGDM_Root_proof}.
Its main idea is to factorize $A_{\alpha,\beta}$ into a product of 
two simpler lower triangular matrices, each of which has a closed-form 
square root. We show that the two roots commute and that the matrix 
$C_{\alpha,\beta}$ is their product, which implies the theorem. 
\end{proof}

\subsection{Efficiency}
We first establish that the $p$-\acronym for SGD can be computed efficiently 
even for large problem sizes.

\begin{lemma}[Efficiency of \acronym]
The entries of $C^{\b{p}}_{\alpha,\beta}$ can be determined in runtime $O(p\log p)$, 
\ie, in particular independent of $n$.
\end{lemma}

\begin{proof}
As a lower triangular Toeplitz matrix, $C^{\b{p}}_{\alpha,\beta}$ is fully determined 
by the values of its first column. By construction $c_{p+1},\dots,c_{n}=0$, 
so only the complexity of computing $c_1,\dots,c_{p-1}$ matters.
These can be computed efficiently by writing them as the convolution of 
vectors $(\alpha^i r_i)_{i=0,\dots,p-1}$ and $(\beta^i r_i)_{i=0,\dots,p-1}$ 
and, \eg, employing the fast Fourier transform. 
\end{proof}

Note that for running Algorithm~\ref{alg:cap}, the matrix 
$B$ of the factorization is not actually required.
However, one needs to know the \emph{sensitivity} of 
$C^{\b{p}}_{\alpha,\beta}$, as this determines the 
necessary amount of noise. 
The following theorem establishes that for a large class of 
matrices, including the \acronym in the SGD setting, this 
is possible exactly and in closed form. 

\begin{restatable}[Sensitivity for decreasing non-negative Toeplitz matrices]{theorem}{toeplitzsens}\label{thm:b-sensitivity}
Let $M=\toep(m_0,\dots,m_{n-1})$ be a lower triangular Toeplitz matrix with decreasing non-negative entries, \ie 

$m_0 \geq m_1 \geq m_2 \geq \dots m_{n-1}\geq 0$. 
Then its \emph{sensitivity}~\eqref{eq:sensPi} in the setting of \emph{$b$-min-separation} is
\begin{align}
\sens_{k,b}(M) &= \Big\|\sum_{j = 0}^{k-1}M_{[\cdot, 1 + jb]}\Big\| = \left(\sum\limits_{i=0}^{n-1}\left(\sum\limits_{j=0}^{\min\{k - 1, i/b\}}m_{i - jb}\right)^2\right)^{1/2},
\label{eq:sens_toeplitz}
\end{align}
where $M_{[\cdot, 1 + jb]}$ denotes the $(1 + jb)$-th column of $M$.
\end{restatable}

\begin{proof}
The proof can be found in Appendix~\ref{sec:b-sensitivity-proof}. 
It builds on the identity \eqref{eq:sensPi}, which holds with equality 
because of the non-negative entries of $M$. 
Using the fact that the entries of $M$ are non-increasing one  
establishes that an optimal $b$-separated index set is 
$\{1,1+b,\cdots,1+(k-1)b\}$. From this, the identity~\eqref{eq:sens_toeplitz} follows.
\end{proof}

\begin{restatable}{corollary}{BSRsensitivity}\label{cor:decreasing_entities}
The sensitivity of the $p$-\acronym 
for SGD can be computed using formula \eqref{eq:sens_toeplitz}.
\end{restatable}

\begin{proof}
It suffices to show that the coefficients $c_0,\dots,c_{n-1}$ 
of Theorem~\ref{thm:SGDM_Root} are monotonically decreasing.  
We do so by an explicit computation, see Appendix~\ref{sec:decreasing_entities_proof}.
\end{proof}

\subsection{Approximation Quality -- Single Participation}
Having established the efficiency of \acronym, we now 
demonstrate its suitability for high-quality model training.
To avoid corner cases, we assume that $\frac{n}{b}$ is an 
integer, which does not affect the asymptotic behavior. 
We also discuss only the case in which the update vectors 
have bounded norm $\zeta=1$. 
Results for general $\zeta$ can readily be derived using 
the linearity of the sensitivity with respect to $\zeta$. 

We first discuss the case of model training with single participation ($k=1$), where more precise results are possible than the general case. 
Our main result are bounds on the expected approximation error 
of the square root factorization that, in particular, prove 
its asymptotic optimality.

\begin{restatable}[Expected approximation error with single 
participation]{theorem}{errorstreaming}\label{thm:approximation_streaming}
Let $A_{\alpha,\beta}\in\R^{n\times n}$ be the workload 
matrix~\eqref{eq:A_alpha_beta} of SGD with momentum $0\leq\beta<1$ 
and weight decay parameter $0<\alpha\leq 1$, where $\alpha>\beta$.
Assume that each data item can contribute at most once 
to an update vector (\eg single participation, $k=1$).
Then, the expected approximation error of the \emph{square 
root factorization}, $A_{\alpha,\beta} = C^2_{\alpha,\beta}$, 
fulfills 
\begin{align}
1 \leq \mathcal{E}(C_{\alpha, \beta},C_{\alpha, \beta}) 
 &\leq \frac{1}{(\alpha - \beta)^2}\log\frac{1}{1-\alpha^2}
\intertext{for $\alpha<1$, and }
\max\left\{ 1, \frac{\log(n + 1) - 1}{4}\right\} &\leq\mathcal{E}(C_{1, \beta},C_{1, \beta})
\leq \frac{1+\log(n)}{(1 - \beta)^2}.
\label{eq:approximation_streaming}
\end{align}
\end{restatable}

\begin{proof}
For the proof, we establish a relations between $\sens_{1,n}(C)$ 
and $\|C_{\alpha,\beta}\|_{\Fr}$, and then we bound the resulting 
expressions by an explicit analysis of the norm.
For details, see Appendix~\ref{sec:proof_approximation_streaming}.
\end{proof}

The following two results provide context for the interpretation of Theorem \ref{thm:approximation_streaming}.

\begin{restatable}{theorem}{errorlowerbound}\label{thm:streaming_lower_bound}
Assume the setting of Theorem~\ref{thm:approximation_streaming}.
Then, for any factorization $A_{\alpha,\beta}=BC$ with $C^\top C\geq 0$, 
the expected approximation error fulfills
\begin{align}
\Ecal(B,C) &= 
\begin{dcases}
\Omega(1) \qquad &\text{for $\alpha<1$,}
\\ 
\Omega(\log n)\qquad &\text{for $\alpha=1$.}
\end{dcases}
\label{eq:streaming_lower_bound}
\end{align}
\end{restatable}

\begin{proof}
The theorem is the special case $k=1$ of Theorem~\ref{thm:bmin_lower_bound}, 
which we state in the next section and prove in Section~\ref{sec:bmin_lower_bound_proof}. 
\end{proof}

\begin{restatable}{theorem}{errorbaselines}\label{thm:streaming_baselines}
Assume the setting of Theorem~\ref{thm:approximation_streaming}.
Then, the baseline factorizations $A_{\alpha,\beta}=A_{\alpha,\beta}\cdot\Id$ 
and $A_{\alpha,\beta}=\Id\cdot A_{\alpha,\beta}$ 
fulfill, for $\alpha<1$, 
\begin{align} 
\Ecal(A_{\alpha,\beta},\Id) &=\frac{\sqrt{1 + \alpha\beta}}{\sqrt{(1 - \alpha\beta)(1 - \alpha^2)(1 - \beta^2)}}+o(1)
\quad\text{and}\quad
\Ecal(A_{1,\beta},\Id) \leq \frac{\sqrt{n}}{\sqrt{2}(1 - \beta)}+o(\sqrt{n})\label{eq:streaming_baselines_A_Id}
\\ 
\Ecal(\Id,A_{\alpha,\beta}) &= \frac{\sqrt{1 + \alpha\beta}}{\sqrt{(1 - \alpha\beta)(1 - \alpha^2)(1 - \beta^2)}}\ +o(1)\quad\text{and}\quad
\Ecal(\Id,A_{1,\beta}) \leq \frac{\sqrt{n}}{1 - \beta} \ +o(\sqrt{n}).
\label{eq:streaming_baselines_Id_A}
\end{align}
\end{restatable}

\begin{proof}
The result follows from an explicit analysis of the coefficients, see Appendix~\ref{sec:streaming_baselines_proof}.
\end{proof}

\paragraph{Discussion.}
Theorems~\ref{thm:approximation_streaming} to~\ref{thm:streaming_baselines}
provide a full characterization of the approximation quality of 
the square root factorization as well as its alternatives:
1) the square root factorization has asymptotically optimal
approximation quality, because the upper bounds in 
Equation~\eqref{eq:approximation_streaming} match the lower 
bounds in Equation~\eqref{eq:streaming_lower_bound};
2) the \AOF~from Definition~\ref{def:AOF} also fulfills the 
conditions of Theorem~\ref{thm:streaming_lower_bound}. Therefore, 
it must also adhere to the lower bound~\eqref{eq:streaming_lower_bound} 
and cannot be asymptotically better than the square root factorization;
3) the approximation qualities of the baseline factorizations 
in Equation~\eqref{eq:streaming_baselines_A_Id} and~\eqref{eq:streaming_baselines_Id_A} 
are asymptotically worse than optimal in the $\alpha=1$ setting,
and worse by a constant factor for $\alpha<1$. The BSR factorization can be applied even in more general scenarios, such as with varying learning rates. However, in this case, the workload matrix will no longer be Toeplitz. This makes it difficult to provide analytical guarantees for the matrix, but it can still be applied numerically.

\subsection{Approximation Quality -- Repeated Participation.}
We now provide mostly asymptotic statements about the approximation 
quality of \acronym and baselines in the setting where data items 
can contribute more than once to the update vectors. 

\begin{restatable}[Approximation error of \acronym]{theorem}{approximationbmin}\label{thm:bmin_approximation}
Let $A_{\alpha,\beta}\in\R^{n\times n}$ be the workload matrix \eqref{eq:A_alpha_beta} 
of SGD with momentum $0\leq\beta<1$ and weight decay $0<\alpha\leq 1$, with $\alpha>\beta$.
Let  
$A_{\alpha,\beta}=B^{\b{p}}_{\alpha,\beta} C^{\b{p}}_{\alpha,\beta}$, 
be its \emph{\method} as in Definition~\ref{def:banded_square_root}.
Then, for any $b\in\{1,\dots,n\}$, $p\leq b$, and $k\in\{1,\dots,\frac{n}{b}\}$  it holds:
\begin{align}
    &\Ecal(B^{\b{p}}_{\alpha,\beta}, C^{\b{p}}_{\alpha,\beta}) = 
    \begin{dcases}
    O_{\beta}\left(\sqrt{\frac{nk\log p}{p}}\right) + O_{\beta, p}(\sqrt{k})
    \qquad &\text{for $\alpha=1$,}
    \\
    O_{\beta, p, \alpha}\left(\sqrt{k}\right)
    \qquad &\text{for $\alpha < 1$.}\\
    \end{dcases} 
\end{align}
\end{restatable}
\begin{proof}
For the proof, we separately bound the sensitivity of $C^{\b{p}}_{\alpha,\beta}$ and the Frobenius norm of $B^{\b{p}}_{\alpha,\beta}$. The former is straightforward because of the matrix's band structure. 
The latter requires an in-depth analysis of the inverse matrix' coefficient. Both steps are detailed in Appendix~\ref{sec:proof_bmin_approximation}.
\end{proof}

The following results provide context for the interpretation of Theorem \ref{thm:bmin_approximation}.

\begin{restatable}[Approximation error of Square Root Factorization]{theorem}{approximationbminsquareroot}\label{thm:bmin_approximation_square_root}
Let $A_{\alpha,\beta}\in\R^{n\times n}$ be the workload matrix \eqref{eq:A_alpha_beta} 
of SGD with momentum $0\leq\beta<1$ and weight decay $0<\alpha\leq 1$, with 
$\alpha>\beta$.
Let $A_{\alpha,\beta}=C^2_{\alpha,\beta}$ be its \emph{square root factorization}.
Then, for any $b\in\{1,\dots,n\}$ and $k=\frac{n}{b}$ it holds:
\begin{align}
    &\Ecal(C_{\alpha,\beta}, C_{\alpha,\beta}) = 
    \begin{dcases}
    \Theta_\beta\left(k\sqrt{\log n} + \sqrt{k}\log n\right)\qquad &\text{for $\alpha = 1$,}
    \\
    \Theta_{\alpha, \beta}\big(\sqrt{k}\big)\qquad &\text{for $\alpha < 1$.}
    \end{dcases}
\end{align}
\end{restatable}

\begin{proof}
We bound $\sens_{k,b}(C_{\alpha,\beta})$ and $\|C_{\alpha,\beta}\|_{\Fr}$ using
the explicit entries for $C_{\alpha,\beta}$ from Theorem~\ref{thm:SGDM_Root}. 
Details are provided in Appendix~\ref{sec:proof_bmin_approximation_squareroot}.
\end{proof}

\begin{restatable}{theorem}{lowerboundbmin}\label{thm:bmin_lower_bound}
Assume the setting of Theorem~\ref{thm:bmin_approximation}.
Then, for any factorization $A_{\alpha,\beta}=BC$ with $C^\top C\geq 0$, the
approximation error fulfills
\begin{align}
\Ecal(B,C) &\geq
\begin{dcases}
\frac{\sqrt{k} \log (n + 1)}{\pi} \qquad &\text{for $\alpha=1$,}
\\
 \sqrt{k} \qquad &\text{for $\alpha<1$,}
\end{dcases}\label{eq:bmin_lower_bound_1}
\end{align}
\end{restatable}

\begin{proof}
The proof is based on the observation that $\|X\|_{\Fr}\|Y\|_{\Fr}\geq \|XY\|_*$ 
for any matrices $X,Y$, where $\|\cdot\|_*$ denotes the nuclear norm. 
To derive~\eqref{eq:bmin_lower_bound_1}, we show that $\sens_{k,b}(C)$ is 
lower bounded by $\frac{\sqrt{k}}{n}\|C\|_{\Fr}$, and derive 
explicit bounds on the singular values of $A_{\alpha,\beta}$. 
\end{proof}

\begin{restatable}{theorem}{baselinesbmin}\label{thm:bmin_baselines}
Assume the setting of Theorem~\ref{thm:bmin_approximation}.
Then, the baseline factorizations $A_{\alpha,\beta}=A_{\alpha,\beta}\cdot\Id$ 
and $A_{\alpha,\beta}=\Id\cdot A_{\alpha,\beta}$ 
fulfill
\begin{align}
\Ecal(A_{\alpha,\beta},\Id) &  
\geq \begin{dcases}   
  \sqrt{\frac{nk}{2}}  \qquad&\text{for $\alpha = 1$,}
\\ 
\sqrt{k} \qquad&\text{for $\alpha < 1$.}
\end{dcases}
\qquad\quad
\Ecal(\Id,A_{\alpha,\beta}) 
\geq \begin{dcases} 
\frac{k\sqrt{n}}{\sqrt{3}} &\text{for $\alpha = 1$,}
\\\medskip 
\sqrt{k}  &\text{for $\alpha < 1$.}
\end{dcases}\label{eq:bmin_baselines}
\end{align}
\end{restatable}
\begin{proof}
The proof relies on the fact that the workload matrices can be 
lower bounded componentwise by simpler matrices: $A_{\alpha,\beta}\geq A_{\alpha,0}$ 
and $A_{\alpha,0}\geq\Id$. For the simpler matrices, the 
bounds~\eqref{eq:bmin_baselines} can then be derived analytically,
and the general case follows by monotonicity.
\end{proof}

\paragraph{Discussion.}
Analogously to the case of single participation, Theorems~\ref{thm:bmin_approximation} 
to~\ref{thm:bmin_baselines} again establish that the proposed \acronym 
is asymptotically superior to the baseline factorizations if $\alpha=1$.
A comparison of Theorems~\ref{thm:bmin_approximation} and~\ref{thm:bmin_approximation_square_root} 
suggests that, at least for maximal participation, $k=\frac{n}{b}$ and $p=b$, 
the bandedness of the $p$-\acronym improves the approximation quality, 
specifically in the practically relevant regime where $b\ll n$. 
While none of the methods match the lower bound of Theorem~\ref{thm:bmin_approximation},
we conjecture that this is not because any asymptotically better methods would exist, 
but rather a sign of Equation~\eqref{eq:bmin_lower_bound_1} is not tight. 
Both theoretical consideration and experiments suggest that a term linear in 
$k$ should appear there. 
For $\alpha<1$, all studied methods are asymptotically identical and, 
in fact, optimal. 

\section{Experiments}\label{sec:experiments}
To demonstrate that \acronym can achieve high accuracy not only in theory but also in practice, we compare it to \AOF and baselines in numerical experiments. \textbf{Our results show that \acronym achieves quality comparable to the \AOF, but without the computational overhead, and it clearly outperforms the baseline factorizations.}
The privacy guarantees are identical for all methods, so we do 
not discuss them explicitly. 

\paragraph{Implementation and computational cost.}
We implement \acronym by the closed-form expressions of 
Theorem~\eqref{thm:SGDM_Root}. 
For single data participation, we use the square root 
decomposition directly. For repeated data participation 
we use $p$-\acronym with $p=b$. 
Using standard \emph{python/numpy} code, computing the \acronym 
as dense matrices are memory-bound rather than compute-bound. 
Even sizes of $n=10,000$ or more take at most a few 
seconds.
Computing only the Toeplitz coefficients is even faster, 
of course. 

To compute \AOF, we solve the optimization problem~\eqref{eq:AOF} 
using the \texttt{cvxpy} package with \texttt{SCS} backend, 
see Algorithm~\ref{alg:aof} for the source code\footnote{Additional 
experiments with gradient-based optimizers can be found in Appendix~\ref{sec:other_optimizers}.}.
With the default numerical tolerance, $10^{-4}$, each factorization
took a few minutes ($n\leq 100$) to hours ($n\leq 500$) to several 
days ($n\geq 700$) of CPU time. 
Note that this overhead reappears for any change in the number of 
update steps, $n$, weight decay, $\alpha$, or momentum, $\beta$, 
as these induce different workload matrices. 
In our experiments, when the optimization for \AOF did not terminate 
within 10 days, we reran the optimization problem with the tolerance
increased by a factor of 10.
The runtime depends not only on the matrix size but also on the
entries. 
In particular, we observe matrices with momentum to be harder 
to factorize than without. 
For large matrix sizes we frequently encountered numerical problems: 
the intermediate matrices, $S$, in \eqref{eq:AOF}, often did not 
fulfill the positive definiteness condition required to solve 
the subsequent Cholesky decomposition for $C$. 
Unfortunately, simply projecting the intermediates back to the 
cone of positive semi-definite matrices is not enough, because 
the resulting $C$ matrices also have to be invertible and not 
too badly conditioned. 
Ultimately, we adopted a postprocessing step for $S$ that ensures 
that all its eigenvalues were at least of value $\sqrt{1/n}$, which we find to be a reasonable modification to ensure the stability of the convergence.
Enforcing this empirically found value leads to generally good 
results, as our experiments below show, but it does add an 
undesirable extra level of complexity to the process.
In contrast, due to its analytic expressions, \acronym does 
not suffer from numerical problems. It also does not possess 
additional hyperparameters, such as a numeric tolerance or 
the number of optimization steps. 

Apart from the factorization itself, the computational cost of 
\acronym and \AOF are nearly identical. 
Both methods produce (banded) lower triangular matrices, so 
computing the inverse matrices or solving linear systems can 
be done within milliseconds to seconds using forward substitution. 
Note that, in principle, one could even exploit the Toeplitz 
structure of $p$-\acronym, but we found this 
not to yield any practical benefit in our experiments. 
Computing the sensitivity is trivial for $p$-BSR using
Corollary~\ref{cor:decreasing_entities}, and it is still 
efficient for \AOF by the dynamic program proposed in~\citet{BandedMatrix}. 

\begin{figure}[t]
    \centering
    \includegraphics[width=.48\textwidth]{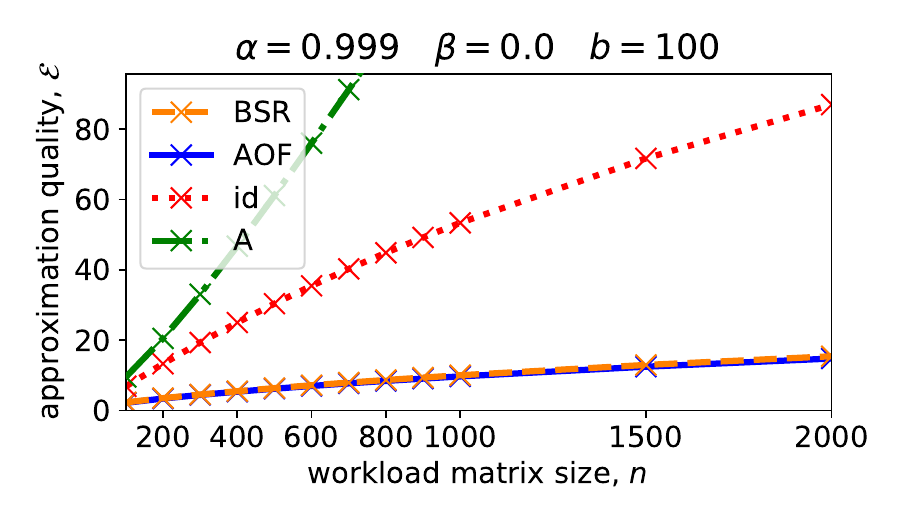}
    \quad
    \includegraphics[width=.48\textwidth]{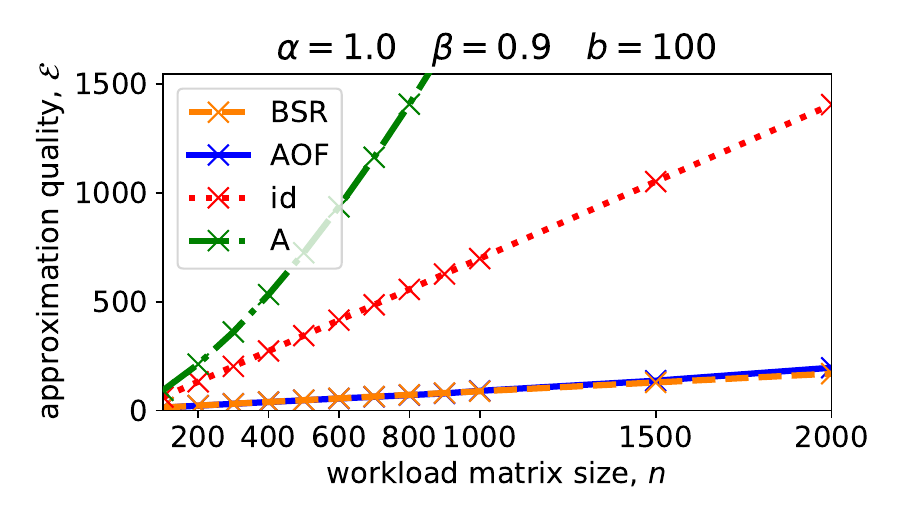}
\caption{Expected approximation error of \acronym, \AOF and baseline factorizations for two different hyperparameter settings (left: $\alpha=0.999, \beta=0$, right: $\alpha=1, \beta=0.9$) with repeated participation $(b=100, k=n/100)$. See Section~\ref{sec:experiments} for details.}\label{fig:accuracy}
\end{figure}

\paragraph{Expected Approximation Error.}
As a first numeric experiment, we evaluate the expected 
approximation error for workload matrices that reflect 
different SGD settings. 
Specifically, we use workload matrices \eqref{eq:A_alpha_beta} 
for $n\in\{100,200,\dots,1000,1500,2000\}$, with $\alpha=\{0.99,0.999,0.9999,1\}$, 
and $\beta\in\{0,0.9\}$, either with single participation, $k=1$, 
or repeated participation, $b=100$, $k=n/100$.
Figure~\ref{fig:accuracy} shows the expected approximate error, 
$\Ecal(B,C)$, of the proposed \acronym, \AOF, as well as the 
baseline factorizations, $A=A\cdot\Id$ and $A=\Id\cdot A$ in 
two exemplary cases.
Additional results for other privacy levels can be 
found in Appendix~\ref{sec:appendix_experiments}.

The results confirm our expectations from the theoretical analysis:
in particular, \acronym's expected approximation error is quite 
close to \AOF's, typically within a few percent (left plot).
Both methods are clearly superior to the naive factorizations. 
For large matrix sizes, \acronym sometimes even yields slightly 
better values than \AOF (right plot). 
However, we believe this to be a numeric artifact of us having 
to solve \AOF with less-than-perfect precision. 

\begin{figure}[t]
    \centering
    \includegraphics[width=.48\textwidth]{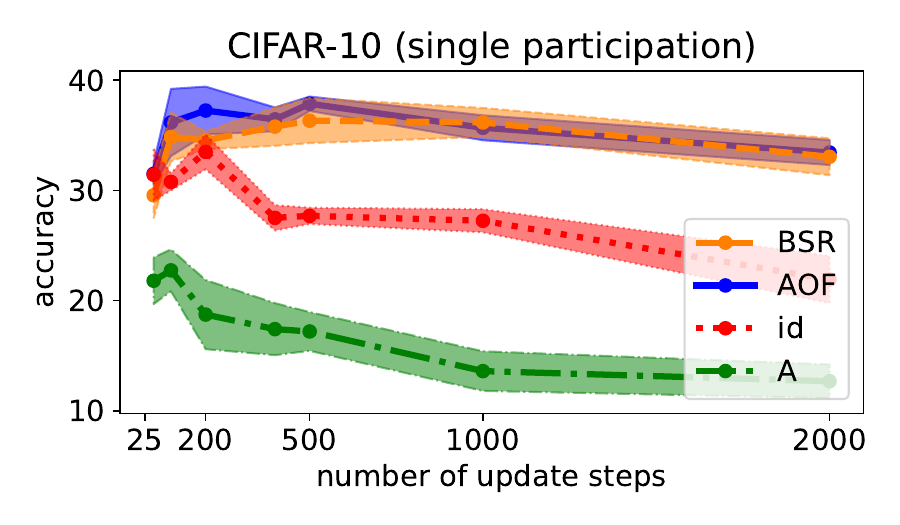}
    \quad
    \includegraphics[width=.48\textwidth]{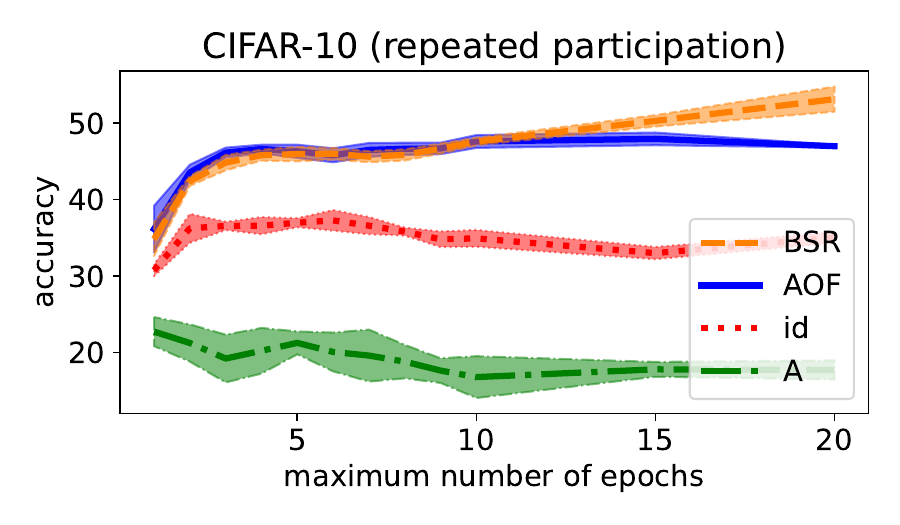}
\caption{Classification accuracy (mean and standard deviation over 5 runs with different random seeds) on CIFAR-10 for \acronym, \AOF, and baselines for $(\epsilon,\delta)=(4, 10^{-5})$ for independent training runs. Left: one epoch, different batch sizes. Right: different number of epochs, constant batch size.}\label{fig:cifar10}
\end{figure}

\paragraph{Private Model Training on CIFAR-10.}
To demonstrate the usefulness of \acronym in practical settings, 
we follow the setup of \citet{Kairouz} and report results for 
training a simple ConvNet on the CIFAR-10 dataset (see Table~\ref{tab:cifar_convnet} 
in Appendix~\ref{sec:appendix_experiments} for the architecture). 
We provide a PyTorch implementation with fast gradient clipping from Opacus \citep{yousefpour2021opacus} in our GitHub repository \footnote{https://github.com/npkalinin/Matrix-Factorization-DP-Training/tree/main}.
To reflect the setting of single-participation training, we 
split the 50,000 training examples into batches of size 
$m\in\{1000,500,250,200,100,50,25\}$, resulting 
in $n\in\{100,200,400,500,1000,2000\}$ update steps. 
For repeated participation, we fix the batch size 
to $500$ and run $k\in\{1,2,\dots,10,15,20\}$ epoch 
of training, \ie $n=100k$ and $b=100$. 
In both cases, 20\% of the training examples are 
used as validation sets to determine the learning rate 
$\eta\in\{0.01,0.05,0.1,0.5,1\}$, weight decay 
parameters $\alpha\in\{0.99, 0.999, 0.9999, 1\}$, 
and momentum $\beta\in\{0,0.9\}$.
Figure~\ref{fig:cifar10} shows the test set accuracy
of the model trained with hyperparameters that achieved
the highest validation accuracy.\footnote{Such a 
setting would not optimal for real-world private training, 
because the many repeated experiments reduce the privacy 
guarantees~\citep{papernot2021hyperparameter,kurakin2022toward,ponomareva2023dp}.  
We nevertheless adopt it here to allow for a simpler and fair 
comparison between methods.}
One can see the expected effect that in DP model training, 
more update steps/epochs do not necessary lead to higher 
accuracy due to the need to add more noise. 
The quality of models trained with \acronym is mostly 
identical to \AOF.
When training for a large number of epochs it achieves even 
better slightly results, but this could also be an artifact 
of us having to solve \AOF with reduced precision in this 
regime. Both methods are clearly superior to the baselines. 

\section{Conclusion and Discussion}
We introduce an efficient and effective approach to the matrix 
factorization mechanism for SGD-based model training with differential 
privacy. The proposed \method (\acronym) factorization achieves 
results on par with the previous state-of-the-art, and clearly 
superior to baseline methods. At the same time, it does not suffer from the previous method's computational overhead, thereby making differentially private model training practical even for large scale problems.

Despite the promising results, some open questions remain. On the theoretical side, the asymptotic optimality of \acronym without weight decay is still unresolved because the current upper bounds on the expected approximation error do not match the provided lower bounds. Based on the experimental results, we believe this discrepancy lies with the lower bounds, which we suspect should be linear in the number of participations.  We observe that BSR achieves results comparable to AOF, although we cannot currently prove this due to the insufficient understanding of AOF's theoretical properties; nonetheless, we consider it a promising research direction. On the practical side, it would be interesting to extend the 
guarantees to even more learning scenarios, such as variable 
learning rates.

\begin{ack}
This research was supported by the Scientific Service Units (SSU) of ISTA through resources provided by Scientific Computing (SciComp).
We thank Monika Henzinger and Jalaj Upadhyay for their valuable comments on the earlier versions of this manuscript.
\end{ack}

\bibliographystyle{abbrvnat}
\bibliography{lit}


\clearpage
\appendix
\renewcommand{\proofname}{Proof}

\section*{Appendix}

\section{ General introduction to differential privacy.}

Differential privacy \citep{Dwork_DP_original} is a robust framework designed to provide strong privacy guarantees for statistical analyses and data sharing. It aims to protect individual data points in a dataset while still allowing meaningful aggregate information to be extracted. Unlike traditional data anonymization techniques, which might involve removing identifiers or aggregating data, differential privacy offers a mathematical definition of privacy that quantifies the amount of privacy loss and ensures that the risk of identifying any individual's data remains low, even when combined with other data sources.To formalize this concept, a randomized mechanism $M$ is said to provide $(\varepsilon, \delta)$-differential privacy if, for all data sets $D$ and $D'$ that differ in one element, and for all subsets of the mechanism's output space $S$:
\begin{equation*}
\Pr[M(D) \in S] \leq e^\varepsilon \cdot \Pr[M(D') \in S] + \delta
\end{equation*}

For 
a detailed introduction to 
differential privacy, we recommend the books "The Algorithmic Foundations
of Differential Privacy" by \citet{DP_book_Dwork} and  "The Complexity of Differential Privacy" by  \citet{DP_book_Salil}.

\section{Source code for computing \AOF}
\begin{algorithm}\label{alg:aof}
\begin{lstlisting}[language=Octave]
import cvxpy as cp
import numpy as np

def banded_factorization(A, b):
    n = len(A)
    X = cp.Variable((n, n), PSD=True)

    # cp.matrix_frac(A, X) = tr(A.T @ X^-1 @ A)
    objective = cp.Minimize(cp.matrix_frac(A.T, X) * np.ceil(n / b))
    constraints = [cp.diag(X) == 1]

    for i in range(b, n):
        constraints += [cp.diag(X, i) == 0, cp.diag(X, -i) == 0]

    prob = cp.Problem(objective, constraints)
    return prob.solve(solver='SCS'), X.value 
\end{lstlisting}
\caption{Source code for computing \AOF using \texttt{cvxpy}.}
\end{algorithm}

\section{Network architecture for CIFAR-10 experiments}\label{sec:appendix_experiments}

\begin{table}[H]
\caption{ConvNet architecture for CIFAR-10 experiments}\label{tab:cifar_convnet}
{\small\sffamily\begin{tabular}{|l|}
\hline
Conv2D(channels=32, kernel=(3, 3), strides=(1, 1), padding='SAME', activation='relu')
\\\hline
Conv2D(channels=32, kernel=(3, 3), strides=(1, 1), padding='SAME', activation='relu')
\\\hline
MaxPool(kernel=(2, 2), strides=(2, 2))
\\\hline
Conv2D(channels=64, kernel=(3, 3), strides=(1, 1), padding='SAME'), activation='relu')
\\\hline
Conv2D(channels=64, kernel=(3, 3), strides=(1, 1), padding='SAME'), activation='relu')
\\\hline
MaxPool(kernel=(2, 2), strides=(2, 2))
\\\hline
Conv2D(channels=128, kernel=(3, 3), strides=(1, 1), padding='SAME'), activation='relu')
\\\hline
Conv2D(channels=128, kernel=(3, 3), strides=(1, 1), padding='SAME'), activation='relu')
\\\hline
MaxPool(kernel=(2, 2), strides=(2, 2))
\\\hline
Flatten()
\\\hline
Dense(outputs=10)
\\\hline
\end{tabular}}
\end{table}

\clearpage
\section{Additional Experimental Results}\label{sec:additional_experiments}

In this section we provide additional experiments comparing 
\acronym, \AOF and baselines: Figures~\ref{fig:accuracy_extra} 
and \ref{fig:accuracy_extra_single} and following tables 
show their \emph{expected approximation error} 
(lower is better) for workload matrices stemming from SGD with 
different hyperparameter settings.
Figure~\ref{fig:cifar10appendix}) and following tables show 
the accuracy of resulting 
classifiers on CIFAR-10 (higher is better) for different privacy 
levels. 

The results show the same trends as the one in Section~\ref{sec:experiments}.
\acronym achieves almost identical expected approximation error
as \AOF, and results in equally good classifiers.
In some cases, results for \acronym even improve over \AOF's. 
Presumably this is because of numerical issues in solving 
the optimization problem for \AOF.

\begin{figure}[H]
\includegraphics[width=.48\textwidth]{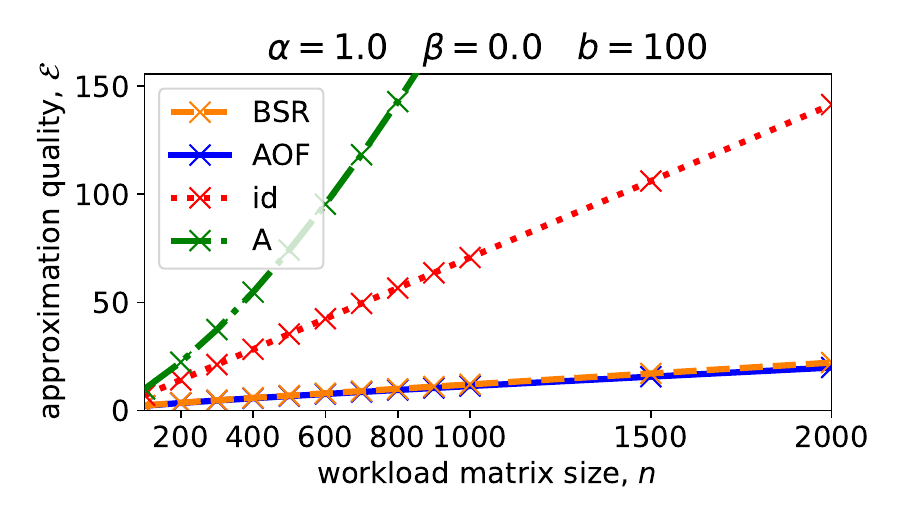}
\quad
\includegraphics[width=.48\textwidth]{experiments/chl/quality-alpha1-beta0.9-b100-p100.pdf}
\\
\includegraphics[width=.48\textwidth]{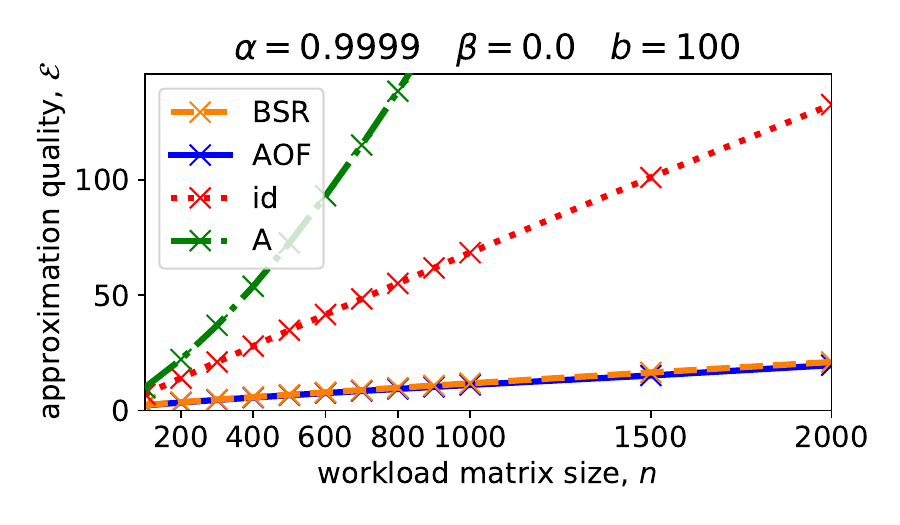}
\quad
\includegraphics[width=.48\textwidth]{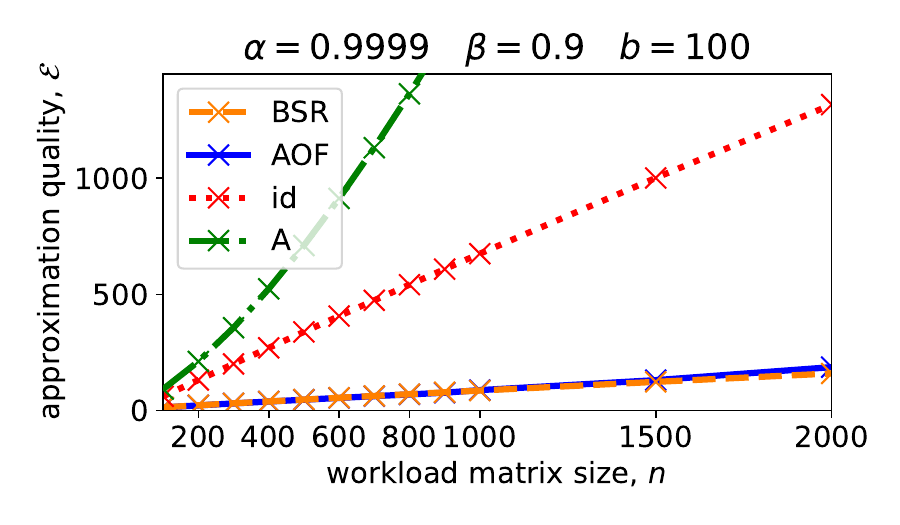}
\\
\includegraphics[width=.48\textwidth]{experiments/chl/quality-alpha0.999-beta0-b100-p100.pdf}
\quad
\includegraphics[width=.48\textwidth]{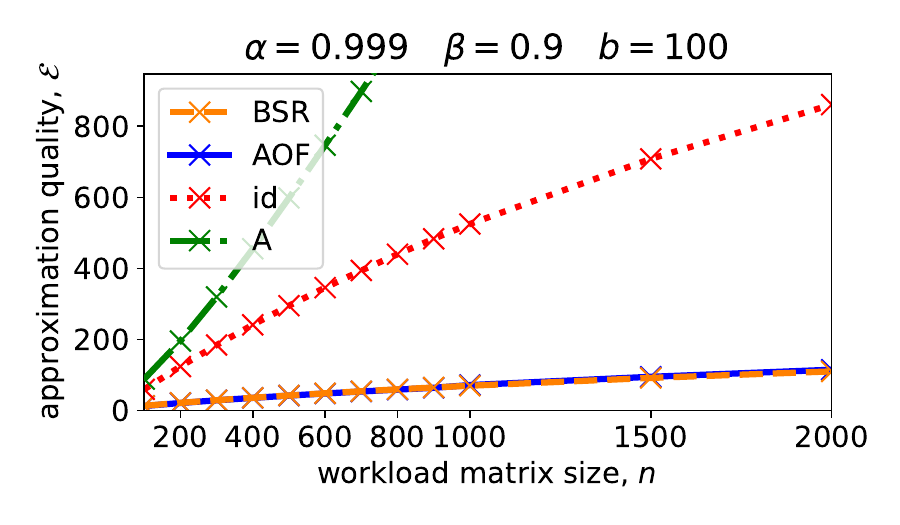}
\\
\includegraphics[width=.48\textwidth]{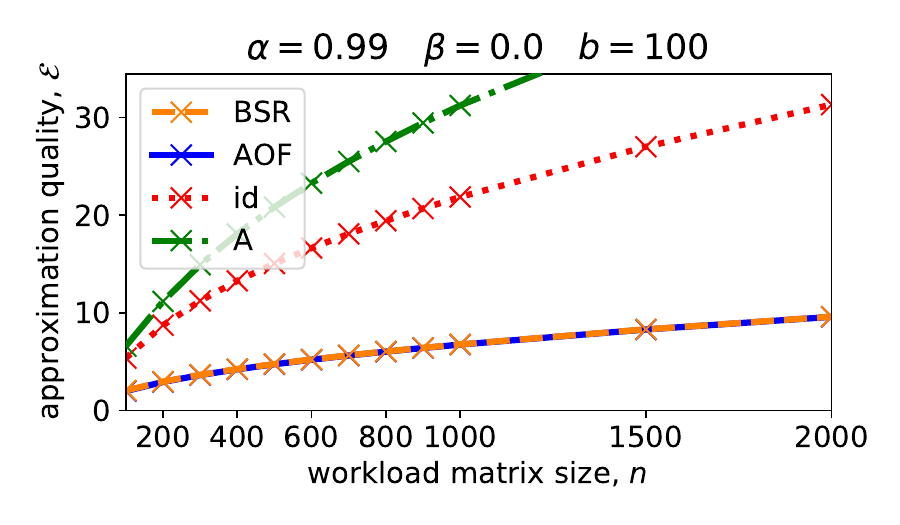}
\quad
\includegraphics[width=.48\textwidth]{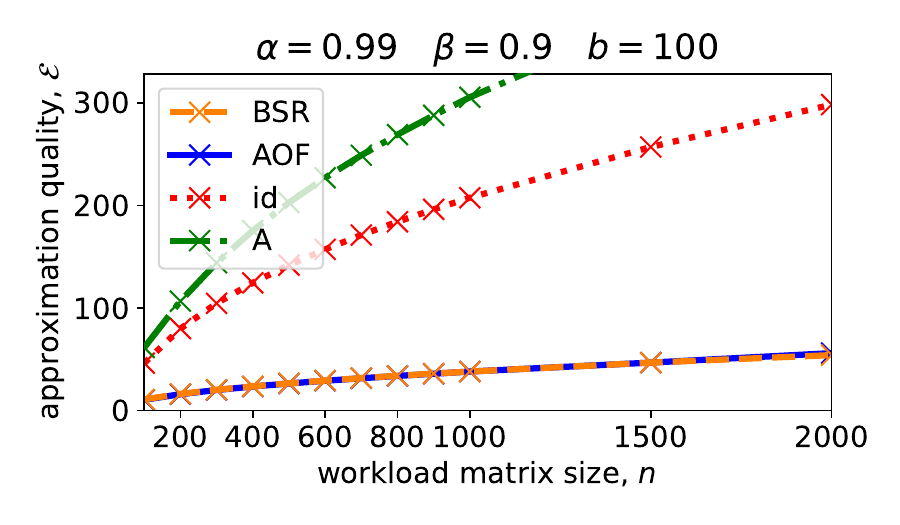}
\caption{Expected approximation error of $p$-\acronym, \AOF and baseline factorizations with 
lmultiple participations and $p=b=100$. }\label{fig:accuracy_extra}
\end{figure}

\begin{figure}[H]
\includegraphics[width=.48\textwidth]{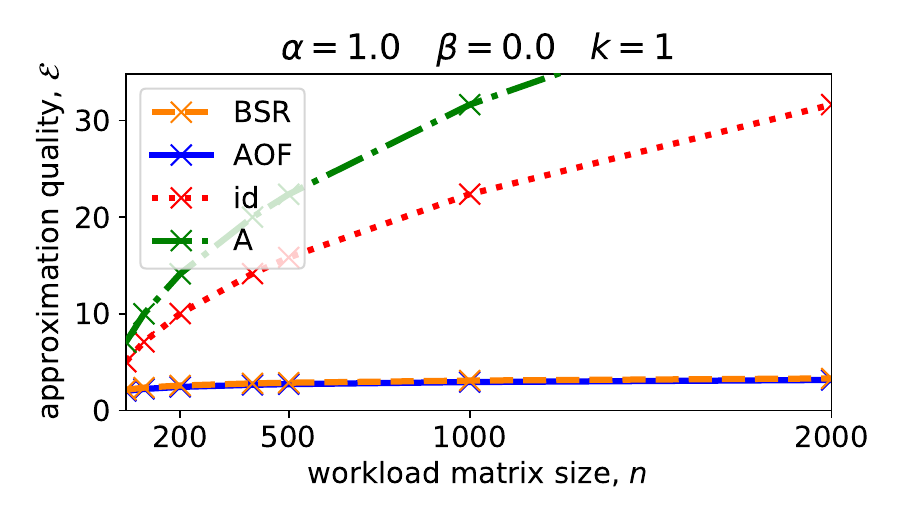}
\quad
\includegraphics[width=.48\textwidth]{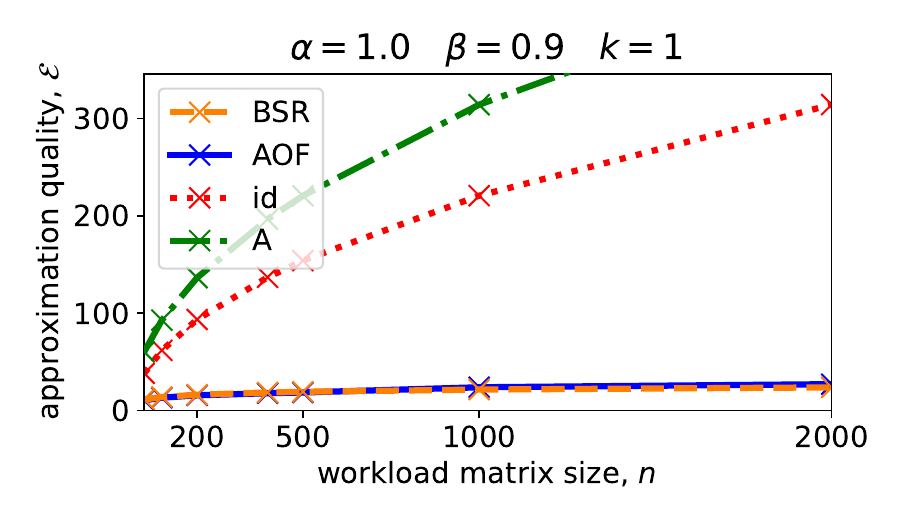}
\\
\includegraphics[width=.48\textwidth]{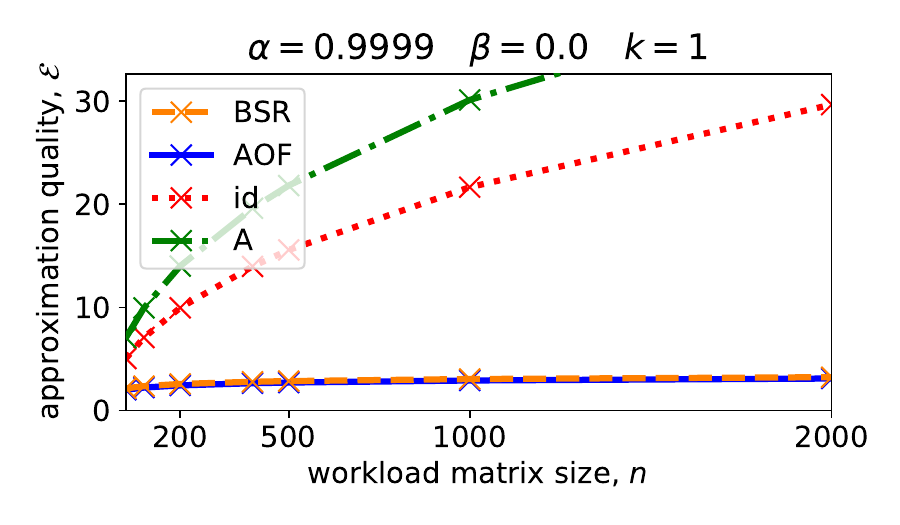}
\quad
\includegraphics[width=.48\textwidth]{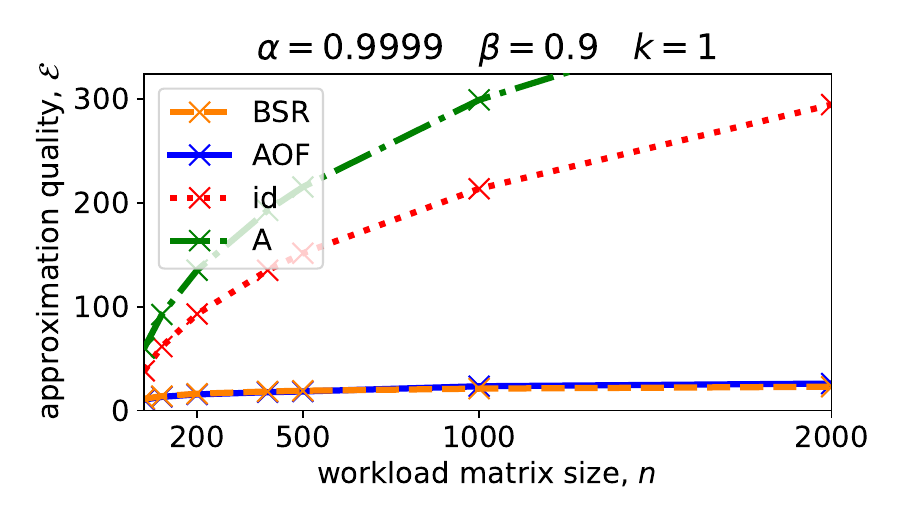}
\\
\includegraphics[width=.48\textwidth]{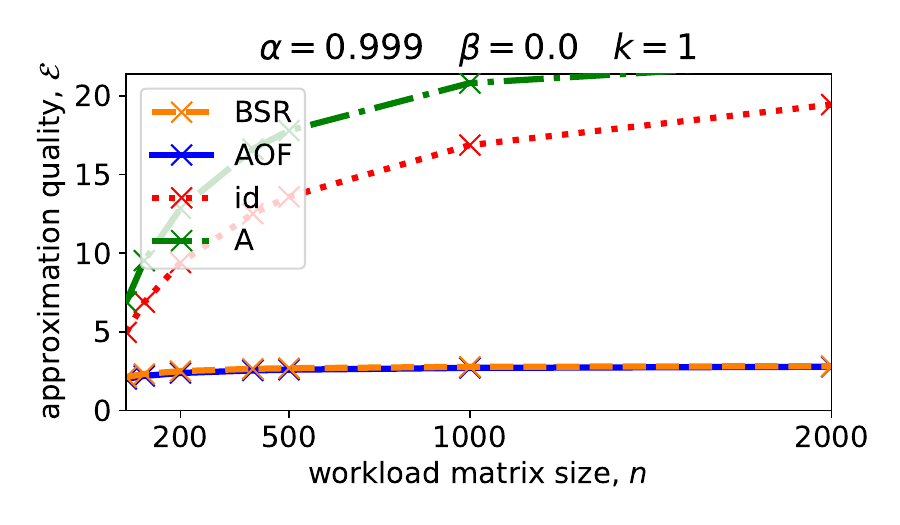}
\quad
\includegraphics[width=.48\textwidth]{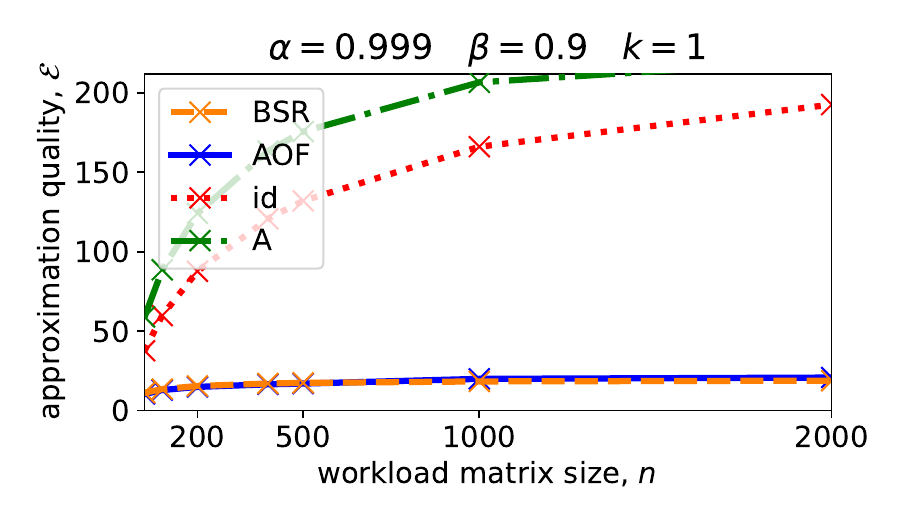}
\\
\includegraphics[width=.48\textwidth]{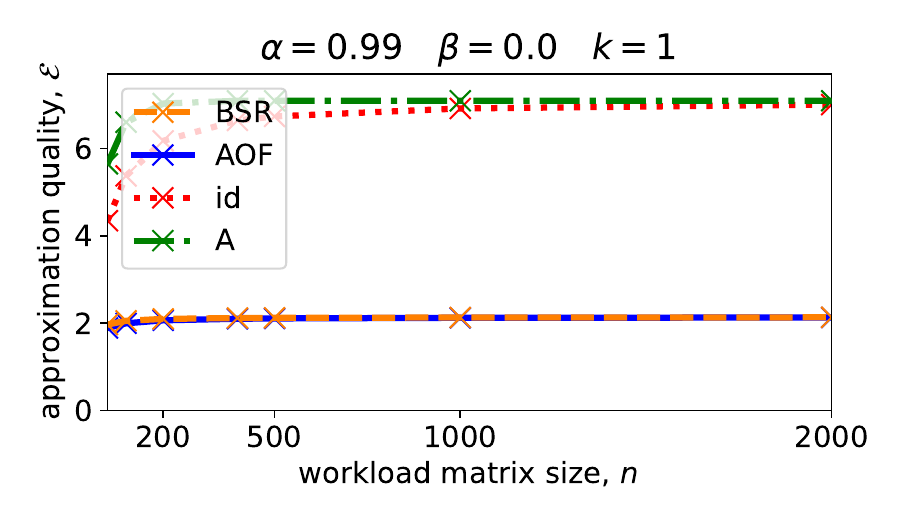}
\quad
\includegraphics[width=.48\textwidth]{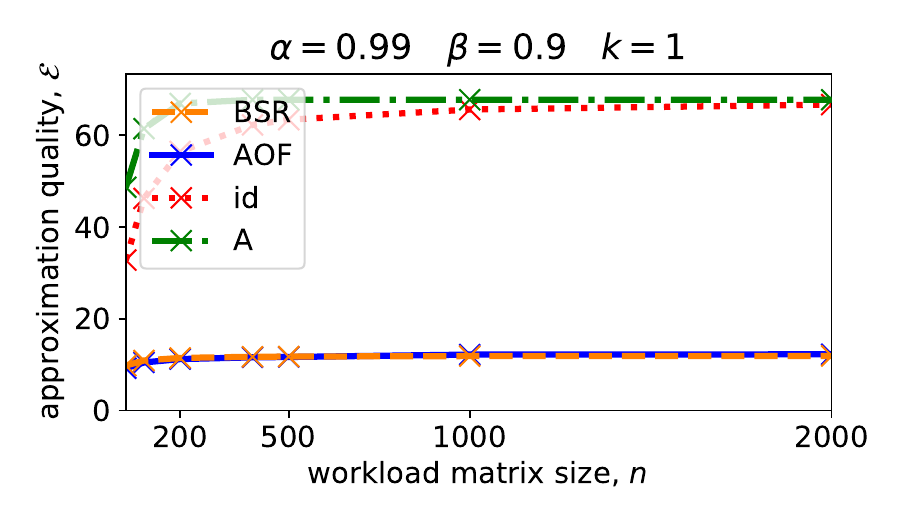}
\caption{Expected approximation error of \acronym, \AOF and baseline factorizations with 
single participation ($k=1$, $p=b=n$). }\label{fig:accuracy_extra_single}
\end{figure}

\begin{table}[H] \centering
\begin{tabular}{rccccc}
\toprule
    & \multicolumn{5}{c}{expected factorization error}\\
$n$ & \acronym & sqrt & \AOF & $\Id$ & $A$\\
\midrule
100& $2.4$& $2.4$& $2.2$& $7.1$& $10.0$\\
200& $3.6$& $4.0$& $3.5$& $14.2$& $22.4$\\
300& $4.8$& $5.5$& $4.6$& $21.2$& $37.4$\\
400& $5.9$& $6.9$& $5.7$& $28.3$& $54.8$\\
500& $6.9$& $8.4$& $6.7$& $35.4$& $74.2$\\
600& $8.0$& $9.9$& $7.7$& $42.5$& $95.4$\\
700& $9.0$& $11.3$& $8.6$& $49.5$& $118.3$\\
800& $10.1$& $12.8$& $9.5$& $56.6$& $142.8$\\
900& $11.1$& $14.2$& $10.4$& $63.7$& $168.8$\\
1000& $12.1$& $15.7$& $11.3$& $70.7$& $196.2$\\
1500& $17.2$& $23.1$& $15.7$& $106.1$& $352.1$\\
2000& $22.2$& $30.6$& $19.9$& $141.5$& $535.7$\\
\bottomrule
\end{tabular}
\caption{Numeric results for Figure~\ref{fig:accuracy_extra} as well as a plain square root decomposition: $\alpha=1$, $\beta=0$, $k=1$, $b=k/n$}
\end{table}

\begin{table}[H] \centering
\begin{tabular}{rccccc}
\toprule
    & \multicolumn{5}{c}{expected factorization error}\\
$n$ & \acronym & sqrt & \AOF & $\Id$ & $A$\\
\midrule
100& $13.9$& $13.9$& $13.3$& $61.8$& $92.9$\\
200& $22.9$& $25.7$& $22.4$& $132.3$& $213.2$\\
300& $31.4$& $37.4$& $31.0$& $202.9$& $361.2$\\
400& $39.7$& $49.1$& $39.3$& $273.6$& $532.6$\\
500& $48.0$& $61.0$& $47.4$& $344.3$& $724.7$\\
600& $56.2$& $73.0$& $55.4$& $415.0$& $935.3$\\
700& $64.3$& $85.1$& $63.2$& $485.7$& $1163.0$\\
800& $72.5$& $97.2$& $71.1$& $556.4$& $1406.6$\\
900& $80.6$& $109.5$& $78.8$& $627.1$& $1665.2$\\
1000& $88.7$& $121.8$& $90.6$& $697.8$& $1937.9$\\
1500& $129.2$& $184.2$& $137.2$& $1051.3$& $3491.5$\\
2000& $169.7$& $247.8$& $196.9$& $1404.9$& $5322.6$\\
\bottomrule
\end{tabular}
\caption{Numeric results for Figure~\ref{fig:accuracy_extra} as well as a plain square root decomposition: $\alpha=1$, $\beta=0.9$, $k=1$, $b=k/n$}
\end{table}

\begin{table}[H] \centering
\begin{tabular}{rccccc}
\toprule
    & \multicolumn{5}{c}{expected factorization error}\\
$n$ & \acronym & sqrt & \AOF & $\Id$ & $A$\\
\midrule
100& $2.4$& $2.4$& $2.2$& $7.1$& $10.0$\\
200& $3.6$& $3.9$& $3.5$& $14.1$& $22.2$\\
300& $4.8$& $5.4$& $4.6$& $21.0$& $36.9$\\
400& $5.8$& $6.9$& $5.7$& $27.9$& $53.9$\\
500& $6.9$& $8.3$& $6.6$& $34.8$& $72.7$\\
600& $7.9$& $9.7$& $7.6$& $41.6$& $93.1$\\
700& $8.9$& $11.1$& $8.5$& $48.4$& $115.1$\\
800& $9.9$& $12.5$& $9.4$& $55.1$& $138.4$\\
900& $10.9$& $13.9$& $10.3$& $61.8$& $163.0$\\
1000& $11.8$& $15.3$& $11.1$& $68.5$& $188.7$\\
1500& $16.5$& $22.3$& $15.2$& $101.1$& $332.6$\\
2000& $21.0$& $29.1$& $19.7$& $132.6$& $496.8$\\
\bottomrule
\end{tabular}
\caption{Numeric results for Figure~\ref{fig:accuracy_extra} as well as a plain square root decomposition: $\alpha=0.9999$, $\beta=0$, $k=1$, $b=k/n$}
\end{table}

\begin{table}[H] \centering
\begin{tabular}{rccccc}
\toprule
    & \multicolumn{5}{c}{expected factorization error}\\
$n$ & \acronym & sqrt & \AOF & $\Id$ & $A$\\
\midrule
100& $13.9$& $13.9$& $13.2$& $61.6$& $92.4$\\
200& $22.8$& $25.5$& $22.3$& $131.4$& $211.4$\\
300& $31.2$& $37.0$& $30.8$& $200.9$& $356.7$\\
400& $39.3$& $48.6$& $38.9$& $270.0$& $524.0$\\
500& $47.3$& $60.1$& $46.8$& $338.6$& $710.3$\\
600& $55.3$& $71.7$& $54.5$& $406.8$& $913.3$\\
700& $63.1$& $83.3$& $62.1$& $474.6$& $1131.5$\\
800& $70.9$& $95.0$& $69.6$& $541.9$& $1363.5$\\
900& $78.6$& $106.6$& $77.0$& $608.8$& $1608.1$\\
1000& $86.2$& $118.3$& $88.1$& $675.3$& $1864.6$\\
1500& $123.7$& $176.5$& $131.0$& $1001.3$& $3298.4$\\
2000& $159.9$& $234.4$& $187.0$& $1317.1$& $4937.9$\\
\bottomrule
\end{tabular}
\caption{Numeric results for Figure~\ref{fig:accuracy_extra} as well as a plain square root decomposition: $\alpha=0.9999$, $\beta=0.9$, $k=1$, $b=k/n$}
\end{table}

\begin{table}[H] \centering
\begin{tabular}{rccccc}
\toprule
    & \multicolumn{5}{c}{expected factorization error}\\
$n$ & \acronym & sqrt & \AOF & $\Id$ & $A$\\
\midrule
100& $2.3$& $2.3$& $2.2$& $6.9$& $9.5$\\
200& $3.5$& $3.8$& $3.4$& $13.3$& $20.5$\\
300& $4.6$& $5.1$& $4.4$& $19.3$& $33.1$\\
400& $5.5$& $6.4$& $5.4$& $25.0$& $46.7$\\
500& $6.4$& $7.6$& $6.2$& $30.4$& $61.1$\\
600& $7.2$& $8.7$& $7.0$& $35.4$& $76.0$\\
700& $8.0$& $9.8$& $7.7$& $40.2$& $91.2$\\
800& $8.7$& $10.8$& $8.4$& $44.8$& $106.5$\\
900& $9.4$& $11.8$& $9.1$& $49.2$& $122.0$\\
1000& $10.0$& $12.8$& $9.7$& $53.3$& $137.4$\\
1500& $13.0$& $17.2$& $12.5$& $71.6$& $212.7$\\
2000& $15.4$& $21.1$& $14.8$& $86.9$& $282.9$\\
\bottomrule
\end{tabular}
\caption{Numeric results for Figure~\ref{fig:accuracy_extra} as well as a plain square root decomposition: $\alpha=0.999$, $\beta=0$, $k=1$, $b=k/n$}
\end{table}

\begin{table}[H] \centering
\begin{tabular}{rccccc}
\toprule
    & \multicolumn{5}{c}{expected factorization error}\\
$n$ & \acronym & sqrt & \AOF & $\Id$ & $A$\\
\midrule
100& $13.5$& $13.5$& $12.9$& $59.8$& $88.6$\\
200& $21.8$& $24.2$& $21.4$& $124.0$& $195.8$\\
300& $29.2$& $34.3$& $28.9$& $184.5$& $319.9$\\
400& $36.1$& $44.0$& $35.8$& $241.4$& $455.3$\\
500& $42.5$& $53.3$& $42.2$& $295.2$& $598.5$\\
600& $48.6$& $62.3$& $48.1$& $346.0$& $746.7$\\
700& $54.3$& $70.9$& $53.7$& $394.2$& $898.2$\\
800& $59.8$& $79.2$& $59.0$& $440.0$& $1051.7$\\
900& $65.0$& $87.2$& $64.1$& $483.6$& $1205.9$\\
1000& $70.0$& $95.0$& $71.8$& $525.1$& $1360.2$\\
1500& $91.9$& $130.3$& $95.1$& $708.4$& $2113.9$\\
2000& $110.3$& $161.0$& $115.0$& $861.2$& $2816.2$\\
\bottomrule
\end{tabular}
\caption{Numeric results for Figure~\ref{fig:accuracy_extra} as well as a plain square root decomposition: $\alpha=0.999$, $\beta=0.9$, $k=1$, $b=k/n$}
\end{table}

\begin{table}[H] \centering
\begin{tabular}{rccccc}
\toprule
    & \multicolumn{5}{c}{expected factorization error}\\
$n$ & \acronym & sqrt & \AOF & $\Id$ & $A$\\
\midrule
100& $2.1$& $2.1$& $2.0$& $5.4$& $6.6$\\
200& $3.0$& $3.1$& $2.9$& $8.7$& $11.2$\\
300& $3.7$& $3.8$& $3.6$& $11.2$& $14.9$\\
400& $4.3$& $4.5$& $4.2$& $13.3$& $18.1$\\
500& $4.8$& $5.0$& $4.7$& $15.1$& $20.8$\\
600& $5.2$& $5.5$& $5.2$& $16.6$& $23.3$\\
700& $5.7$& $6.0$& $5.6$& $18.1$& $25.5$\\
800& $6.1$& $6.4$& $6.0$& $19.4$& $27.5$\\
900& $6.4$& $6.8$& $6.4$& $20.7$& $29.4$\\
1000& $6.8$& $7.2$& $6.8$& $21.9$& $31.2$\\
1500& $8.3$& $8.9$& $8.3$& $27.0$& $38.9$\\
2000& $9.6$& $10.3$& $9.6$& $31.3$& $45.4$\\
\bottomrule
\end{tabular}
\caption{Numeric results for Figure~\ref{fig:accuracy_extra} as well as a plain square root decomposition: $\alpha=0.99$, $\beta=0$, $k=1$, $b=k/n$}
\end{table}

\begin{table}[H] \centering
\begin{tabular}{rccccc}
\toprule
    & \multicolumn{5}{c}{expected factorization error}\\
$n$ & \acronym & sqrt & \AOF & $\Id$ & $A$\\
\midrule
100& $10.9$& $10.9$& $10.5$& $46.2$& $61.4$\\
200& $16.2$& $17.0$& $15.9$& $79.9$& $106.7$\\
300& $20.3$& $21.7$& $20.0$& $104.5$& $144.0$\\
400& $23.6$& $25.7$& $23.4$& $124.5$& $175.4$\\
500& $26.6$& $29.1$& $26.4$& $141.7$& $202.7$\\
600& $29.2$& $32.1$& $29.1$& $157.1$& $226.9$\\
700& $31.7$& $34.9$& $31.5$& $171.1$& $248.8$\\
800& $33.9$& $37.5$& $33.8$& $184.0$& $268.9$\\
900& $36.0$& $40.0$& $35.9$& $196.1$& $287.7$\\
1000& $38.0$& $42.3$& $38.1$& $207.4$& $305.3$\\
1500& $46.7$& $52.2$& $46.8$& $256.9$& $381.4$\\
2000& $54.1$& $60.6$& $56.0$& $298.2$& $444.6$\\
\bottomrule
\end{tabular}
\caption{Numeric results for Figure~\ref{fig:accuracy_extra} as well as a plain square root decomposition: $\alpha=0.99$, $\beta=0.9$, $k=1$, $b=k/n$}
\end{table}

\begin{table}[H] \centering
\begin{tabular}{rccccc}
\toprule
    & \multicolumn{5}{c}{expected factorization error}\\
$n$ & \acronym & sqrt & \AOF & $\Id$ & $A$\\
\midrule
50& $2.2$& $2.2$& $2.0$& $5.0$& $7.1$\\
100& $2.4$& $2.4$& $2.2$& $7.1$& $10.0$\\
200& $2.6$& $2.6$& $2.4$& $10.0$& $14.1$\\
400& $2.8$& $2.8$& $2.7$& $14.2$& $20.0$\\
500& $2.9$& $2.9$& $2.7$& $15.8$& $22.4$\\
1000& $3.1$& $3.1$& $2.9$& $22.4$& $31.6$\\
2000& $3.3$& $3.3$& $3.2$& $31.6$& $44.7$\\
\bottomrule
\end{tabular}
\caption{Numeric results for Figure~\ref{fig:accuracy_extra_single} as well as a plain square root decomposition: $\alpha=1$, $\beta=0$, $b=0$}
\end{table}

\begin{table}[H] \centering
\begin{tabular}{rccccc}
\toprule
    & \multicolumn{5}{c}{expected factorization error}\\
$n$ & \acronym & sqrt & \AOF & $\Id$ & $A$\\
\midrule
50& $11.4$& $11.4$& $10.6$& $38.2$& $60.3$\\
100& $13.9$& $13.9$& $13.3$& $61.8$& $92.9$\\
200& $16.3$& $16.3$& $15.7$& $93.5$& $136.5$\\
400& $18.6$& $18.6$& $17.8$& $136.8$& $196.5$\\
500& $19.3$& $19.3$& $18.6$& $154.0$& $220.5$\\
1000& $21.6$& $21.6$& $23.9$& $220.7$& $314.0$\\
2000& $23.8$& $23.8$& $27.1$& $314.1$& $445.7$\\
\bottomrule
\end{tabular}
\caption{Numeric results for Figure~\ref{fig:accuracy_extra_single} as well as a plain square root decomposition: $\alpha=1$, $\beta=0.9$, $b=0$}
\end{table}

\begin{table}[H] \centering
\begin{tabular}{rccccc}
\toprule
    & \multicolumn{5}{c}{expected factorization error}\\
$n$ & \acronym & sqrt & \AOF & $\Id$ & $A$\\
\midrule
50& $2.1$& $2.1$& $2.0$& $5.0$& $7.1$\\
100& $2.4$& $2.4$& $2.2$& $7.1$& $10.0$\\
200& $2.6$& $2.6$& $2.4$& $10.0$& $14.0$\\
400& $2.8$& $2.8$& $2.7$& $14.0$& $19.6$\\
500& $2.9$& $2.9$& $2.7$& $15.6$& $21.8$\\
1000& $3.1$& $3.1$& $2.9$& $21.7$& $30.1$\\
2000& $3.2$& $3.2$& $3.1$& $29.7$& $40.6$\\
\bottomrule
\end{tabular}
\caption{Numeric results for Figure~\ref{fig:accuracy_extra_single} as well as a plain square root decomposition: $\alpha=0.9999$, $\beta=0$, $b=0$}
\end{table}

\begin{table}[H] \centering
\begin{tabular}{rccccc}
\toprule
    & \multicolumn{5}{c}{expected factorization error}\\
$n$ & \acronym & sqrt & \AOF & $\Id$ & $A$\\
\midrule
50& $11.4$& $11.4$& $10.6$& $38.2$& $60.2$\\
100& $13.9$& $13.9$& $13.2$& $61.6$& $92.4$\\
200& $16.2$& $16.2$& $15.6$& $92.9$& $135.2$\\
400& $18.4$& $18.4$& $17.7$& $135.0$& $192.7$\\
500& $19.1$& $19.1$& $18.4$& $151.4$& $215.2$\\
1000& $21.1$& $21.1$& $23.4$& $213.5$& $299.1$\\
2000& $23.0$& $23.0$& $26.0$& $294.5$& $404.7$\\
\bottomrule
\end{tabular}
\caption{Numeric results for Figure~\ref{fig:accuracy_extra_single} as well as a plain square root decomposition: $\alpha=0.9999$, $\beta=0.9$, $b=0$}
\end{table}

\begin{table}[H] \centering
\begin{tabular}{rccccc}
\toprule
    & \multicolumn{5}{c}{expected factorization error}\\
$n$ & \acronym & sqrt & \AOF & $\Id$ & $A$\\
\midrule
50& $2.1$& $2.1$& $2.0$& $5.0$& $6.9$\\
100& $2.3$& $2.3$& $2.2$& $6.9$& $9.5$\\
200& $2.5$& $2.5$& $2.4$& $9.4$& $12.8$\\
400& $2.7$& $2.7$& $2.6$& $12.5$& $16.6$\\
500& $2.7$& $2.7$& $2.6$& $13.6$& $17.8$\\
1000& $2.8$& $2.8$& $2.7$& $16.9$& $20.8$\\
2000& $2.8$& $2.8$& $2.8$& $19.4$& $22.2$\\
\bottomrule
\end{tabular}
\caption{Numeric results for Figure~\ref{fig:accuracy_extra_single} as well as a plain square root decomposition: $\alpha=0.999$, $\beta=0$, $b=0$}
\end{table}

\begin{table}[H] \centering
\begin{tabular}{rccccc}
\toprule
    & \multicolumn{5}{c}{expected factorization error}\\
$n$ & \acronym & sqrt & \AOF & $\Id$ & $A$\\
\midrule
50& $11.2$& $11.2$& $10.4$& $37.6$& $59.0$\\
100& $13.5$& $13.5$& $12.9$& $59.8$& $88.6$\\
200& $15.5$& $15.5$& $14.9$& $87.7$& $124.2$\\
400& $17.0$& $17.0$& $16.5$& $120.7$& $163.3$\\
500& $17.4$& $17.4$& $17.0$& $132.0$& $175.6$\\
1000& $18.4$& $18.4$& $20.0$& $166.1$& $206.6$\\
2000& $18.8$& $18.8$& $20.8$& $192.6$& $220.5$\\
\bottomrule
\end{tabular}
\caption{Numeric results for Figure~\ref{fig:accuracy_extra_single} as well as a plain square root decomposition: $\alpha=0.999$, $\beta=0.9$, $b=0$}
\end{table}

\begin{table}[H] \centering
\begin{tabular}{rccccc}
\toprule
    & \multicolumn{5}{c}{expected factorization error}\\
$n$ & \acronym & sqrt & \AOF & $\Id$ & $A$\\
\midrule
50& $2.0$& $2.0$& $1.9$& $4.3$& $5.6$\\
100& $2.1$& $2.1$& $2.0$& $5.4$& $6.6$\\
200& $2.1$& $2.1$& $2.1$& $6.2$& $7.0$\\
400& $2.1$& $2.1$& $2.1$& $6.6$& $7.1$\\
500& $2.1$& $2.1$& $2.1$& $6.7$& $7.1$\\
1000& $2.1$& $2.1$& $2.1$& $6.9$& $7.1$\\
2000& $2.1$& $2.1$& $2.1$& $7.0$& $7.1$\\
\bottomrule
\end{tabular}
\caption{Numeric results for Figure~\ref{fig:accuracy_extra_single} as well as a plain square root decomposition: $\alpha=0.99$, $\beta=0$, $b=0$}
\end{table}

\begin{table}[H] \centering
\begin{tabular}{rccccc}
\toprule
    & \multicolumn{5}{c}{expected factorization error}\\
$n$ & \acronym & sqrt & \AOF & $\Id$ & $A$\\
\midrule
50& $9.8$& $9.8$& $9.2$& $32.8$& $48.7$\\
100& $10.9$& $10.9$& $10.5$& $46.2$& $61.4$\\
200& $11.5$& $11.5$& $11.2$& $56.5$& $66.9$\\
400& $11.7$& $11.7$& $11.6$& $62.3$& $67.7$\\
500& $11.8$& $11.8$& $11.7$& $63.4$& $67.7$\\
1000& $11.9$& $11.9$& $12.2$& $65.6$& $67.7$\\
2000& $11.9$& $11.9$& $12.3$& $66.7$& $67.7$\\
\bottomrule
\end{tabular}
\caption{Numeric results for Figure~\ref{fig:accuracy_extra_single} as well as a plain square root decomposition: $\alpha=0.99$, $\beta=0.9$, $b=0$}
\end{table}

\begin{figure}[H]
\includegraphics[width=.48\textwidth]{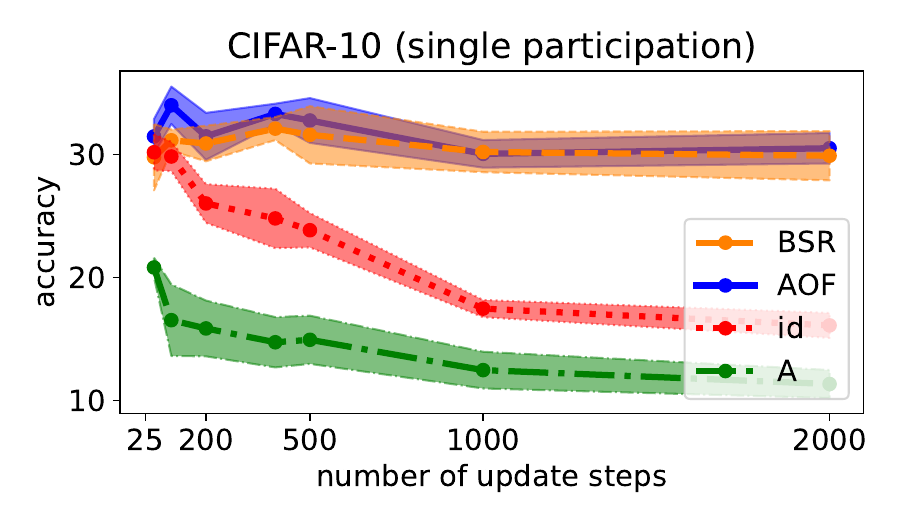}
\quad
\includegraphics[width=.48\textwidth]{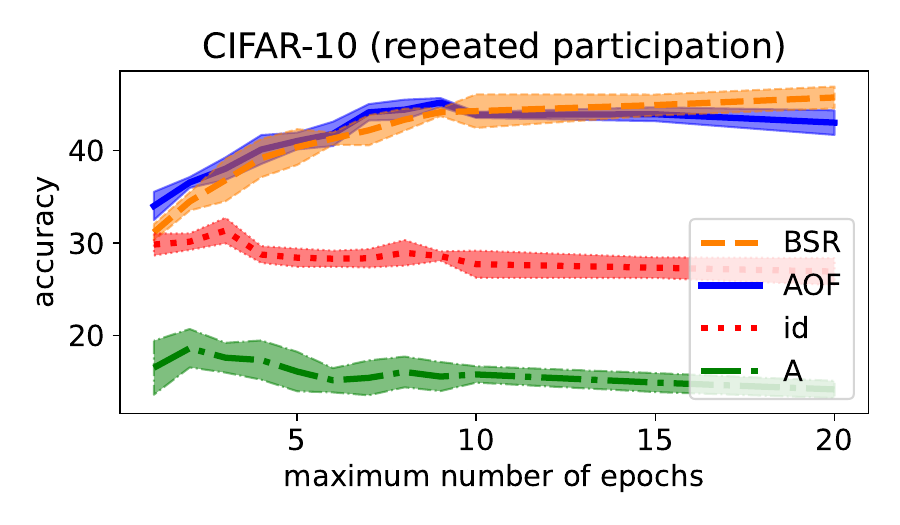}
\\
\includegraphics[width=.48\textwidth]{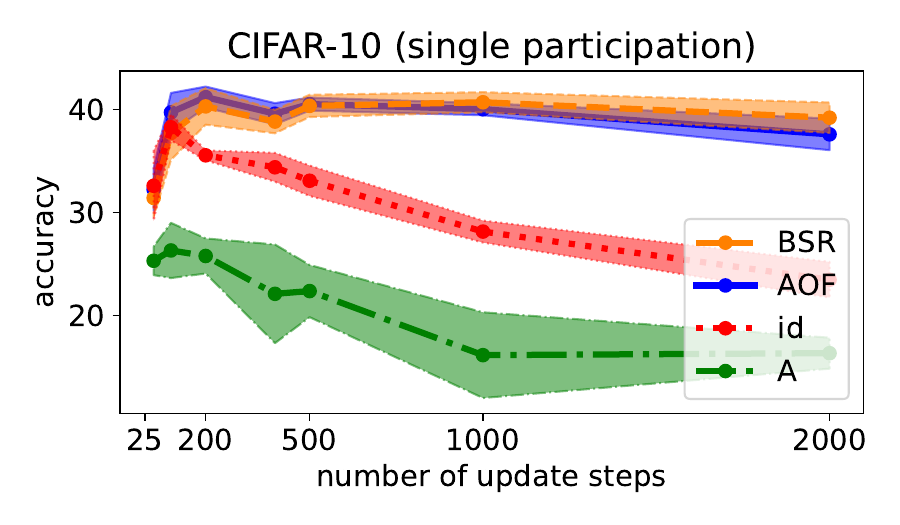}
\quad
\includegraphics[width=.48\textwidth]{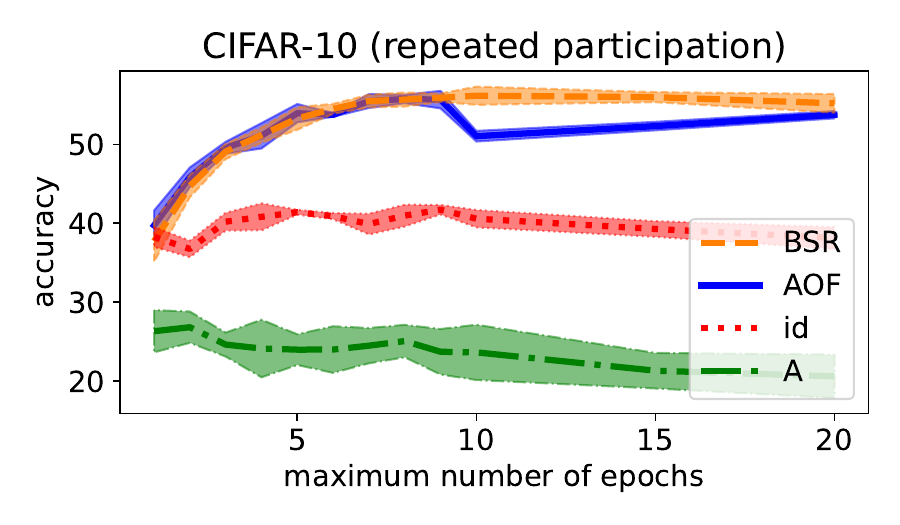}
\caption{Classification accuracy (mean and standard deviation over 5 runs with different random seeds) on CIFAR-10 for \acronym, \AOF, and baselines for independent training runs. Top row: classification accuracy on CIFAR-10 with $(\epsilon,\delta)=(2,10^{-5})$. Bottom row: classification accuracy on CIFAR-10 with $(\epsilon,\delta)=(8,10^{-5})$. Left plots: one epoch, different batch sizes. Right plots: different number of epochs, constant batch size. }\label{fig:cifar10appendix}
\end{figure}

\begin{table}[H]\centering
\begin{tabular}{rcccc}
\toprule
    & \multicolumn{4}{c}{accuracy}\\
number of updates & \acronym & \AOF & $\Id$ & $A$\\
\midrule
50 & $\res{29.548000000000002}{2.0599320377138644}$& $\res{31.5}{0.9751153777886997}$& $\res{31.416000000000004}{2.3430386253751787}$& $\res{21.798000000000002}{2.118707152959087}$\\
100 & $\res{34.838}{2.2033542611209858}$& $\res{36.156}{3.0465111192969587}$& $\res{30.764}{0.7720945537950649}$& $\res{22.740000000000002}{1.9033785750606742}$\\
200 & $\res{34.534}{0.808721212779782}$& $\res{37.238}{2.175102296444927}$& $\res{33.488}{1.5780272494478698}$& $\res{18.728}{3.138832585532398}$\\
400 & $\res{35.782}{1.7480474821926313}$& $\res{36.44}{1.079212675981894}$& $\res{27.504}{1.1469219677031233}$& $\res{17.407999999999998}{2.3597287979765817}$\\
500 & $\res{36.322}{2.0398946051205677}$& $\res{37.852}{0.657282283345597}$& $\res{27.684000000000005}{0.741673782737398}$& $\res{17.198}{1.7579448227973482}$\\
1000 & $\res{36.158}{1.317789816321251}$& $\res{35.674}{1.1322234761742058}$& $\res{27.24}{1.0490948479522715}$& $\res{13.6}{1.799652744281518}$\\
2000 & $\res{33.046}{1.675493957016855}$& $\res{33.43}{1.133953261823436}$& $\res{21.906}{2.1442317971711935}$& $\res{12.684000000000001}{1.5240341203529533}$\\
\bottomrule
\end{tabular}
\caption{Numeric values for results in Table~\ref{fig:cifar10} left plot (CIFAR-10, single participation, $(\epsilon,\delta)=(4,10^{-5})$.}
\end{table}

\begin{table}[H]\centering
\begin{tabular}{rcccc}
\toprule
    & \multicolumn{4}{c}{accuracy}\\
number of epochs & \acronym & \AOF & $\Id$ & $A$\\
\midrule
1 & $\res{34.838}{2.2033542611209858}$& $\res{36.156}{3.0465111192969587}$& $\res{30.764}{0.7720945537950649}$& $\res{22.740000000000002}{1.9033785750606742}$\\
2 & $\res{42.438}{0.6371185133081598}$& $\res{43.554}{1.0171676361347708}$& $\res{36.254}{1.9013232234420292}$& $\res{21.252000000000002}{2.4390510449763045}$\\
3 & $\res{44.842}{1.0103316287239559}$& $\res{46.148}{0.6575484773003463}$& $\res{36.56}{0.5424020648928245}$& $\res{19.22}{3.128114448034151}$\\
4 & $\res{45.812}{0.712509649057469}$& $\res{46.624}{0.5632317462643572}$& $\res{36.592}{1.129101412628645}$& $\res{20.236}{2.9883908713553513}$\\
5 & $\res{45.958}{0.9468474005878682}$& $\res{46.32000000000001}{0.8681301745706104}$& $\res{36.980000000000004}{0.5969924622639721}$& $\res{21.268}{1.5006398635248908}$\\
6 & $\res{46.004}{0.761859567111945}$& $\res{45.8}{0.9149863386958345}$& $\res{37.284}{1.350455478718198}$& $\res{20.078}{2.5377490025611276}$\\
7 & $\res{45.62}{0.728319984622142}$& $\res{46.504000000000005}{0.9050027624267247}$& $\res{36.598}{1.1371763275763358}$& $\res{19.607999999999997}{3.403435029495937}$\\
8 & $\res{45.854}{0.7364984725035066}$& $\res{46.648}{0.7956255400626621}$& $\res{35.82}{0.5239751902523622}$& $\res{18.82}{2.194846236072131}$\\
9 & $\res{46.676}{0.6660555532386166}$& $\res{46.702}{0.7533724178651612}$& $\res{34.815999999999995}{1.0238310407484235}$& $\res{17.644}{1.615775355672936}$\\
10 & $\res{47.598}{0.48684699855293356}$& $\res{47.606}{0.8415640201434451}$& $\res{34.938}{1.1037979887642486}$& $\res{16.794}{2.735320822134033}$\\
15 & $\res{50.302}{0.7488457785151768}$& $\res{47.948}{0.8296505288372922}$& $\res{33.010000000000005}{0.8083625424275908}$& $\res{17.806}{0.9713032482186001}$\\
20 & $\res{53.141999999999996}{1.6447553009490499}$& $\res{46.988}{0.19070920271449723}$& $\res{35.04}{0.7007139216541932}$& $\res{17.728}{1.2466635472331742}$\\
\bottomrule
\end{tabular}
\caption{Numeric values for results in Table~\ref{fig:cifar10} right plot (CIFAR-10, repeated participation, $(\epsilon,\delta)=(4,10^{-5})$.}
\end{table}

\begin{table}[H]\centering
\begin{tabular}{rcccc}
\toprule
    & \multicolumn{4}{c}{accuracy}\\
number of updates & \acronym & \AOF & $\Id$ & $A$\\
\midrule
50 & $\res{29.782}{2.698827523203364}$& $\res{31.474}{1.4139589810174826}$& $\res{30.188}{1.474337139191711}$& $\res{20.816}{0.8005185819204956}$\\
100 & $\res{31.157999999999998}{0.8728802896159356}$& $\res{34.006}{1.5220808125720524}$& $\res{29.839999999999996}{1.1917424218345172}$& $\res{16.532}{2.9139269036816957}$\\
200 & $\res{30.9}{1.4493446795017408}$& $\res{31.479999999999997}{1.908821626030047}$& $\res{26.022000000000002}{1.574585659784821}$& $\res{15.857999999999999}{2.2682416097056324}$\\
400 & $\res{32.108000000000004}{0.9357456919484044}$& $\res{33.302}{0.8339484396531958}$& $\res{24.806}{2.4160256621153695}$& $\res{14.736}{2.0467730699811346}$\\
500 & $\res{31.604000000000003}{2.327923967830564}$& $\res{32.769999999999996}{1.8203708413397535}$& $\res{23.846}{1.3990282341682752}$& $\res{14.936000000000002}{1.9603902672682294}$\\
1000 & $\res{30.21}{1.6466936569987765}$& $\res{30.056}{1.1214410372373564}$& $\res{17.472}{0.7007995433788459}$& $\res{12.47}{1.4935695497699462}$\\
2000 & $\res{29.904000000000003}{2.00160685450465}$& $\res{30.514}{1.2386403836465227}$& $\res{16.106}{1.0151009801985225}$& $\res{11.34}{1.1385736691141246}$\\
\bottomrule
\end{tabular}
\caption{Numeric values for results in Table~\ref{fig:cifar10appendix} top left plot (CIFAR-10, single participation, $(\epsilon,\delta)=(2,10^{-5})$.}
\end{table}

\begin{table}[H]\centering
\begin{tabular}{rcccc}
\toprule
    & \multicolumn{4}{c}{accuracy}\\
number of epochs & \acronym & \AOF & $\Id$ & $A$\\
\midrule
1 & $\res{31.157999999999998}{0.8728802896159356}$& $\res{34.006}{1.5220808125720524}$& $\res{29.839999999999996}{1.1917424218345172}$& $\res{16.532}{2.9139269036816957}$\\
2 & $\res{34.486000000000004}{0.9779468288204648}$& $\res{36.544}{0.5837636508039851}$& $\res{30.144}{0.8959520076432673}$& $\res{18.637999999999998}{2.076275511583183}$\\
3 & $\res{36.8}{2.2863398697481543}$& $\res{38.036}{1.213766863940519}$& $\res{31.348000000000003}{1.386062769141427}$& $\res{17.616}{1.6083314335049217}$\\
4 & $\res{39.178}{2.0624427264775136}$& $\res{40.104}{1.5627955720438935}$& $\res{28.758}{0.9088839309834889}$& $\res{17.358}{2.122927224376757}$\\
5 & $\res{40.368}{1.929292616478899}$& $\res{41.025999999999996}{0.911334186783311}$& $\res{28.424}{0.9913021739106596}$& $\res{16.124000000000002}{2.1316378679316044}$\\
6 & $\res{41.288}{0.66096142096192}$& $\res{41.794}{1.3128328149463653}$& $\res{28.312}{0.8729662078225023}$& $\res{15.175999999999998}{1.3109080822086652}$\\
7 & $\res{42.182}{1.6229510158966605}$& $\res{44.124}{0.8846920368128086}$& $\res{28.351999999999997}{0.9948969795913556}$& $\res{15.440000000000001}{1.885590093312966}$\\
8 & $\res{43.33}{1.1867181636766144}$& $\res{44.410000000000004}{1.078378412246834}$& $\res{28.942}{1.385846311825377}$& $\res{16.078}{1.661466219939485}$\\
9 & $\res{44.146}{0.4281121348431966}$& $\res{45.176}{0.48562331080787446}$& $\res{28.592000000000002}{0.50948012718849}$& $\res{15.578}{1.5708819179047158}$\\
10 & $\res{44.251999999999995}{1.8163066921640776}$& $\res{43.846}{0.32928710876680334}$& $\res{27.718}{1.467913485189097}$& $\res{15.807999999999998}{0.885618427992552}$\\
15 & $\res{44.904}{1.1374225248341112}$& $\res{43.916000000000004}{0.7507529553721377}$& $\res{27.338}{1.11886996563497}$& $\res{14.930000000000001}{1.0253048327204934}$\\
20 & $\res{45.726}{1.1776799225596062}$& $\res{43.008}{1.3292178151078182}$& $\res{26.922000000000004}{1.4554964788689806}$& $\res{14.190000000000001}{0.9275235846058046}$\\
\bottomrule
\end{tabular}
\caption{Numeric values for results in Table~\ref{fig:cifar10appendix} top right plot (CIFAR-10, repeated participation, $(\epsilon,\delta)=(2,10^{-5})$.}
\end{table}

\begin{table}[H]\centering
\begin{tabular}{rcccc}
\toprule
    & \multicolumn{4}{c}{accuracy}\\
number of updates & \acronym & \AOF & $\Id$ & $A$\\
\midrule
50 & $\res{31.419999999999998}{1.8227314667827514}$& $\res{32.236000000000004}{2.0489460705445626}$& $\res{32.6}{3.4077338511098545}$& $\res{25.31}{1.3935027807650777}$\\
100 & $\res{37.736000000000004}{2.6228953467494667}$& $\res{39.738}{1.8849456225578491}$& $\res{38.282}{1.2795975929955485}$& $\res{26.314}{2.682858550128948}$\\
200 & $\res{40.314}{1.7756350976481619}$& $\res{41.218}{1.0243144048581954}$& $\res{35.56}{0.48213068767710315}$& $\res{25.778}{1.7065081306574532}$\\
400 & $\res{38.839999999999996}{1.140876855756134}$& $\res{39.604}{1.0289217657334293}$& $\res{34.398}{1.406029871659918}$& $\res{22.1}{4.796092159247985}$\\
500 & $\res{40.364}{1.083457428789888}$& $\res{40.54}{0.6328901958475868}$& $\res{33.089999999999996}{1.480388462532723}$& $\res{22.368000000000002}{2.536270884586265}$\\
1000 & $\res{40.717999999999996}{0.9779928425095968}$& $\res{40.04}{0.5858754133772824}$& $\res{28.15}{1.0713776178360268}$& $\res{16.14}{4.16523108602632}$\\
2000 & $\res{39.215999999999994}{1.4804323692759498}$& $\res{37.598}{1.5373093377716824}$& $\res{23.468}{1.7349841497835075}$& $\res{16.332}{1.5060776872392745}$\\
\bottomrule
\end{tabular}
\caption{Numeric values for results in Table~\ref{fig:cifar10appendix} bottom left plot (CIFAR-10, repeated participation, $(\epsilon,\delta)=(8,10^{-5})$.}
\end{table}

\begin{table}[H]\centering
\begin{tabular}{rcccc}
\toprule
    & \multicolumn{4}{c}{accuracy}\\
number of epochs & \acronym & \AOF & $\Id$ & $A$\\
\midrule
1 & $\res{37.736000000000004}{2.6228953467494667}$& $\res{39.738}{1.8849456225578491}$& $\res{38.282}{1.2795975929955485}$& $\res{26.314}{2.682858550128948}$\\
2 & $\res{44.736000000000004}{1.4331364205824937}$& $\res{45.846}{1.2442186303057854}$& $\res{36.75}{1.0476879306358367}$& $\res{26.824}{1.987266967470652}$\\
3 & $\res{49.104}{0.9628239714506479}$& $\res{49.58}{0.7579907651152493}$& $\res{40.208000000000006}{1.129898225505288}$& $\res{24.618000000000002}{1.5227508003609778}$\\
4 & $\res{51.19}{0.976012295004524}$& $\res{51.1}{1.611955334368791}$& $\res{40.822}{1.7191480448175513}$& $\res{24.122}{3.6566131870899334}$\\
5 & $\res{53.306}{1.4363251720971815}$& $\res{53.970000000000006}{1.1671118198356132}$& $\res{41.408}{0.3127618902615861}$& $\res{23.98}{1.9525368114327575}$\\
6 & $\res{54.522000000000006}{0.6304125633265903}$& $\res{53.758}{0.24427443582986463}$& $\res{40.910000000000004}{0.39547439866570255}$& $\res{23.988}{2.957172636150956}$\\
7 & $\res{55.504}{0.8072050544935917}$& $\res{55.508}{0.8813739274564436}$& $\res{39.906}{1.323359361624802}$& $\res{24.484}{2.241613258347657}$\\
8 & $\res{55.702}{0.8904605549938741}$& $\res{55.76800000000001}{0.5293108727392631}$& $\res{40.976}{1.3888592441280714}$& $\res{25.066}{2.062796645333708}$\\
9 & $\res{55.92}{0.537075413699044}$& $\res{55.688}{1.1165661646315461}$& $\res{41.730000000000004}{0.6083995397762884}$& $\res{23.71}{2.8877586464245932}$\\
10 & $\res{56.186}{1.1581580203063815}$& $\res{51.05}{0.6760547315121754}$& $\res{40.588}{1.1088372288122372}$& $\res{23.633999999999997}{3.50786687318661}$\\
15 & $\res{56.0}{0.6433894621455949}$& $\res{52.33200000000001}{0.5415902510200887}$& $\res{39.28}{1.0009245725827693}$& $\res{21.308}{2.2549434582711827}$\\
20 & $\res{55.222}{1.1845969778789747}$& $\res{53.77}{0.4645965992126902}$& $\res{38.134}{1.3979377668551625}$& $\res{20.598}{2.7526568983438535}$\\
\bottomrule
\end{tabular}
\caption{Numeric values for results in Table~\ref{fig:cifar10appendix} bottom right plot (CIFAR-10, repeated participation, $(\epsilon,\delta)=(8,10^{-5})$.}
\end{table}

\section{Experimental Results for Different Optimizers}\label{sec:other_optimizers}

In this section we report on experimental results when different optimizers are used to (approximately) 
solve the AOF optimization problem~\eqref{eq:AOF}. 
Besides \texttt{cvxpy} \emph{(CVX)} these are standard \emph{gradient descent (GD)} and 
the \emph{Limited-Memory Broyden-Fletcher-Goldfarb-Shanno algorithm (LBFGS)}. The latter
two we implement in \emph{jax} using the \emph{optax} toolbox.
Similar to~\citep[\texttt{ftrl\_mechanism.py}]{pfl_research}, we use an adaptive line-search for the step size of the gradient-based methods, which at the same time ensures the positive definiteness constraints of the optimization problem.  Our implementation differs from theirs, however, in that our learning rate is not restricted to shrink monotonically, thereby avoiding premature termination.

\subsection{Runtime}
We report the runtimes for the different methods in Tables~\ref{tab:optimizers_first} to~\ref{tab:optimizers_last}. 
For comparison, we also include results for BSR and the CVX optimizers with three 
tolerance levels in the same settings, where practically feasible. 
Note that while the experiments for BSR and CVX used a single-core CPU-only environment, 
the experiments for GD and LBFGS were run on an NVIDIA H100 GPU with 16 available 
CPU cores.
As a consequence, the absolute runtimes are not directly comparable between the methods, 
but they should rather be seen as illustrations of the scaling behavior of the method 
for different workload types and problem sizes. 

Indeed, the results show a clear trend: BSR is the fastest, with almost no overhead. 
Even for the largest problem sizes of $n=10\,000$, BSR never took more than 2.5s to
despite running in the single-core CPU-only setup. GD and LBFGS benefit strongly 
from the GPU hardware. 
In the multiple participation setting ($p=100, k=n/p$), they solve most workload 
sizes within a few minutes, except the largest ones, which for GD can take a few hours.
In the single participation setting ($k=1$), LBFGS also occasionally need several hours 
to converge. In general, stronger weight decay (smaller $\alpha$) tends to lead to 
lower runtimes, while the use of momentum ($\beta=0.9$) to higher times until convergence.
CVX (on weak hardware) is orders of magnitude slower than the other methods. 
Furthermore, its runtime grow approximately cubic with the problem size, whereas 
for GD and LBFGS the relation is not too far from linear. 
Note that despite the stable patterns described above, all runtime results should 
be taken with caution, because internal parameters of the optimization, such as 
the convergence criterion and the specific implementation of the line search can 
substantially influence the overall runtime as well. 

\begin{table}\centering\small
\caption{$\alpha=1.0$, $\beta=0.9$, $p=100$, $k=n/p$}\label{tab:optimizers_first}
\begin{tabular}{r|r|rr|rrr}
\toprule
n & BSR & GD & LBFGS & CVX(tol=0.01) & CVX(tol=0.001) & CVX(tol=0.0001) \\
\midrule
100 & $<1$s & 28.5s &  1m39s & 4.5s &  7m18s &  1h27m30s \\
200 & $<1$s &  1m10s &  2m31s & 37.1s & 21m00s & 10h35m00s \\
300 & $<1$s &  1m44s &  3m14s &  2m16s &  1h12m40s & 22h36m40s \\
400 & $<1$s &  2m35s &  3m46s &  6m06s &  2h14m50s & 53h03m20s \\
500 & $<1$s &  3m47s &  4m47s & 11m45s &  5h31m40s & 90h50m00s \\
600 & $<1$s &  4m27s &  5m11s & 26m50s & 17h36m40s & 40h16m40s \\
700 & $<1$s &  5m13s &  6m12s & 47m50s & 22h10m00s & 66h23m20s \\
800 & $<1$s &  5m52s &  7m30s &  1h29m40s & 38h03m20s & 164h26m40s \\
900 & $<1$s &  6m20s &  7m29s &  1h45m30s & 62h30m00s & 253h53m20s \\
1000 & $<1$s &  6m55s &  8m01s &  1h59m40s & 83h36m40s & 245h00m00s \\
1500 & $<1$s & 10m11s & 11m49s &  6h08m20s & 121h23m20s & timeout \\
2000 & $<1$s & 13m39s & 13m21s & 15h10m00s & 297h13m20s & timeout \\
5000 & 1.1s &  1h09m55s & 33m00s & --- & --- & --- \\
10000 & 1.6s &  6h07m10s &  1h47m39s & --- & --- & --- \\
\bottomrule
\end{tabular}
\end{table}

\begin{table}\centering\small
\caption{$\alpha=1.0$, $\beta=0.0$, $p=100$, $k=n/p$}
\begin{tabular}{r|r|rr|rrr}
\toprule
n & BSR & GD & LBFGS & CVX(tol=0.01) & CVX(tol=0.001) & CVX(tol=0.0001) \\
\midrule
100 & $<1$s & 1.4s & 12.3s & 5.6s & 48.6s & 33.4s \\
200 & $<1$s & 2.4s & 20.0s &  1m09s &  2m38s &  3m39s \\
300 & $<1$s & 3.4s & 25.0s &  2m25s & 11m23s & 19m20s \\
400 & $<1$s & 6.4s & 30.2s & 12m15s & 26m10s &  1h14m40s \\
500 & $<1$s & 6.8s & 36.6s & 33m10s &  1h07m00s &  2h53m20s \\
600 & $<1$s & 9.8s & 44.1s & 52m00s &  7h45m00s &  4h53m20s \\
700 & $<1$s & 12.8s & 51.5s &  1h34m50s & 14h30m00s & 17h48m20s \\
800 & $<1$s & 13.2s & 54.3s &  3h33m20s & 17h26m40s & 26h55m00s \\
900 & $<1$s & 17.9s & 59.1s &  4h43m20s & 27h18m20s & 41h06m40s \\
1000 & $<1$s & 22.9s &  1m09s &  7h31m40s & 45h33m20s & 57h30m00s \\
1500 & $<1$s & 47.3s &  1m36s & 29h43m20s & 84h43m20s & 200h50m00s \\
2000 & $<1$s &  1m24s &  1m45s & 68h53m20s & 258h36m40s & timeout \\
5000 & 2.4s & 16m35s &  4m18s & --- & --- & --- \\
10000 & 1.6s &  2h29m46s & 14m23s & --- & --- & --- \\
\bottomrule
\end{tabular}
\end{table}

\begin{table}\centering\small
\caption{$\alpha=0.9999$, $\beta=0.9$, $p=100$, $k=n/p$}
\begin{tabular}{r|r|rr|rrr}
\toprule
n & BSR & GD & LBFGS & CVX(tol=0.01) & CVX(tol=0.001) & CVX(tol=0.0001) \\
\midrule
100 & $<1$s & 29.9s &  1m35s & 6.1s &  3m42s & 47m00s \\
200 & $<1$s &  1m10s &  2m33s & 36.3s & 18m00s &  6h56m40s \\
300 & $<1$s &  1m45s &  3m04s &  1m48s & 40m30s & 25h46m40s \\
400 & $<1$s &  2m32s &  3m46s &  6m21s &  2h35m20s & 57h13m20s \\
500 & $<1$s &  3m32s &  4m29s & 11m05s &  5h41m40s & 24h08m20s \\
600 & $<1$s &  4m21s &  4m56s & 10m41s & 13h36m40s & 41h56m40s \\
700 & $<1$s &  5m10s &  5m37s & 54m20s & 20h33m20s & 76h56m40s \\
800 & $<1$s &  5m39s &  6m20s &  1h21m40s & 38h20m00s & 158h53m20s \\
900 & $<1$s &  6m25s &  7m05s &  2h16m00s & 51h06m40s & 268h20m00s \\
1000 & $<1$s &  6m56s &  7m46s &  2h11m00s & 68h03m20s & 223h20m00s \\
1500 & $<1$s & 10m09s & 10m15s &  8h08m20s & 115h50m00s & timeout \\
2000 & $<1$s & 13m39s & 12m09s & 12h46m40s & 247h13m20s & timeout \\
5000 & 2.4s &  1h10m18s & 26m09s & --- & --- & --- \\
10000 & 2.6s &  6h06m46s &  1h10m10s & --- & --- & --- \\
\bottomrule
\end{tabular}
\end{table}

\begin{table}\centering\small
\caption{$\alpha=0.9999$, $\beta=0.0$, $p=100$, $k=n/p$}
\begin{tabular}{r|r|rr|rrr}
\toprule
n & BSR & GD & LBFGS & CVX(tol=0.01) & CVX(tol=0.001) & CVX(tol=0.0001) \\
\midrule
100 & $<1$s & 1.2s & 13.7s & 6.0s & 14.1s & 20.3s \\
200 & $<1$s & 2.3s & 18.5s &  1m04s &  4m15s &  4m00s \\
300 & $<1$s & 3.5s & 23.8s &  4m03s & 15m10s & 10m51s \\
400 & $<1$s & 5.3s & 29.7s &  7m45s & 14m32s & 19m50s \\
500 & $<1$s & 6.1s & 38.5s & 22m10s &  1h23m50s &  2h44m20s \\
600 & $<1$s & 8.9s & 43.3s & 29m40s &  3h46m40s &  8h23m20s \\
700 & $<1$s & 12.0s & 44.5s &  1h56m50s & 12h03m20s & 12h11m40s \\
800 & $<1$s & 12.7s & 51.6s &  3h43m20s & 16h40m00s & 23h56m40s \\
900 & $<1$s & 16.1s &  1m02s &  5h11m40s & 26h30m00s & 30h33m20s \\
1000 & $<1$s & 20.3s &  1m05s &  8h23m20s & 40h33m20s & 41h23m20s \\
1500 & $<1$s & 39.8s &  1m19s & 35h50m00s & 78h53m20s & 170h00m00s \\
2000 & $<1$s &  1m10s &  1m27s & 72h30m00s & 239h26m40s & timeout \\
5000 & 1.1s &  9m41s &  3m15s & --- & --- & --- \\
10000 & 2.5s &  1h00m37s & 10m51s & --- & --- & --- \\
\bottomrule
\end{tabular}
\end{table}

\begin{table}\centering\small
\caption{$\alpha=0.999$, $\beta=0.9$, $p=100$, $k=n/p$}
\begin{tabular}{r|r|rr|rrr}
\toprule
n & BSR & GD & LBFGS & CVX(tol=0.01) & CVX(tol=0.001) & CVX(tol=0.0001) \\
\midrule
100 & $<1$s & 32.6s &  1m52s & 3.7s &  3m36s & 58m10s \\
200 & $<1$s & 59.0s &  2m08s & 36.1s & 20m50s &  8h10m00s \\
300 & $<1$s &  1m25s &  2m36s &  3m38s & 37m30s & 25h48m20s \\
400 & $<1$s &  1m54s &  3m02s &  4m48s &  1h30m00s & 56h56m40s \\
500 & $<1$s &  2m13s &  3m33s & 13m45s &  2h58m20s & 84h43m20s \\
600 & $<1$s &  3m07s &  3m57s & 26m20s &  9h51m40s & 41h56m40s \\
700 & $<1$s &  3m04s &  4m09s & 52m30s & 13h10m00s & 85h33m20s \\
800 & $<1$s &  3m28s &  4m18s & 59m00s & 28h03m20s & 164h43m20s \\
900 & $<1$s &  3m59s &  4m33s & 57m00s & 39h10m00s & 280h33m20s \\
1000 & $<1$s &  4m34s &  5m21s &  1h27m00s & 59h43m20s & 258h03m20s \\
1500 & $<1$s &  5m48s &  5m21s &  4h41m40s & 81h40m00s & timeout \\
2000 & $<1$s &  7m26s &  5m28s & 14h16m40s & 219h10m00s & timeout \\
5000 & $<1$s & 34m57s &  9m16s & --- & --- & --- \\
10000 & 2.5s &  2h42m14s & 28m36s & --- & --- & --- \\
\bottomrule
\end{tabular}
\end{table}

\begin{table}\centering\small
\caption{$\alpha=0.999$, $\beta=0.0$, $p=100$, $k=n/p$}
\begin{tabular}{r|r|rr|rrr}
\toprule
n & BSR & GD & LBFGS & CVX(tol=0.01) & CVX(tol=0.001) & CVX(tol=0.0001) \\
\midrule
100 & $<1$s & 1.1s & 11.5s & 7.3s & 28.0s & 27.6s \\
200 & $<1$s & 2.0s & 15.6s &  1m10s & 46.8s &  3m06s \\
300 & $<1$s & 3.1s & 21.3s &  3m39s &  6m32s &  8m10s \\
400 & $<1$s & 3.8s & 26.0s &  8m19s & 30m30s & 29m50s \\
500 & $<1$s & 4.6s & 30.3s &  9m47s & 44m50s &  2h17m50s \\
600 & $<1$s & 5.3s & 29.3s & 48m30s &  2h01m40s &  1h59m50s \\
700 & $<1$s & 6.1s & 39.5s &  1h10m20s &  6h50m00s &  7h01m40s \\
800 & $<1$s & 6.7s & 38.6s &  3h16m40s & 13h10m00s & 13h50m00s \\
900 & $<1$s & 7.9s & 37.4s &  5h26m40s & 21h51m40s & 28h03m20s \\
1000 & $<1$s & 9.8s & 45.1s &  6h26m40s & 30h50m00s & 32h46m40s \\
1500 & $<1$s & 12.8s & 52.1s & 29h26m40s & 77h13m20s & 49h43m20s \\
2000 & $<1$s & 16.6s & 50.3s & 70h00m00s & 91h23m20s & 174h26m40s \\
5000 & 1.3s &  1m19s &  1m29s & --- & --- & --- \\
10000 & 2.4s &  6m40s &  3m23s & --- & --- & --- \\
\bottomrule
\end{tabular}
\end{table}

\begin{table}\centering\small
\caption{$\alpha=0.99$, $\beta=0.9$, $p=100$, $k=n/p$}
\begin{tabular}{r|r|rr|rrr}
\toprule
n & BSR & GD & LBFGS & CVX(tol=0.01) & CVX(tol=0.001) & CVX(tol=0.0001) \\
\midrule
100 & $<1$s & 15.8s &  1m12s & 3.5s &  1m28s &  1h19m10s \\
200 & $<1$s & 19.0s &  1m09s & 14.4s & 22m00s & 11h23m20s \\
300 & $<1$s & 19.6s &  1m15s &  1m10s &  2h29m50s & 30h16m40s \\
400 & $<1$s & 24.5s &  1m23s &  3m07s &  4h55m00s & 59h43m20s \\
500 & $<1$s & 23.8s &  1m21s &  6m55s & 13h06m40s & 67h30m00s \\
600 & $<1$s & 28.7s &  1m20s & 24m10s & 31h40m00s & 44h43m20s \\
700 & $<1$s & 33.4s &  1m24s & 36m10s & 61h23m20s & 78h53m20s \\
800 & $<1$s & 33.7s &  1m29s & 48m20s & 73h36m40s & 146h06m40s \\
900 & $<1$s & 39.2s &  1m37s &  1h06m50s & 50h00m00s & 280h33m20s \\
1000 & $<1$s & 37.3s &  1m32s &  1h53m30s & 66h06m40s & 274h43m20s \\
1500 & $<1$s & 53.9s &  1m44s &  7h36m40s & 302h46m40s & timeout \\
2000 & $<1$s &  1m07s &  1m50s & 22h45m00s & timeout & timeout \\
5000 & $<1$s &  5m56s &  2m51s & --- & --- & --- \\
10000 & 2.4s & 29m34s &  7m45s & --- & --- & --- \\
\bottomrule
\end{tabular}
\end{table}

\begin{table}\centering\small
\caption{$\alpha=0.99$, $\beta=0.0$, $p=100$, $k=n/p$}
\begin{tabular}{r|r|rr|rrr}
\toprule
n & BSR & GD & LBFGS & CVX(tol=0.01) & CVX(tol=0.001) & CVX(tol=0.0001) \\
\midrule
100 & $<1$s & 1.0s & 9.9s & 3.9s & 9.0s & 15.9s \\
200 & $<1$s & $<1$s & 9.8s & 30.3s &  1m40s &  1m38s \\
300 & $<1$s & $<1$s & 10.2s &  2m11s &  2m11s &  1m50s \\
400 & $<1$s & 1.2s & 11.4s & 11m16s & 10m36s &  4m16s \\
500 & $<1$s & 1.3s & 14.6s & 24m10s & 17m40s & 22m20s \\
600 & $<1$s & 1.3s & 10.6s & 32m00s & 39m00s &  1h44m00s \\
700 & $<1$s & 1.4s & 10.7s &  1h06m10s &  1h33m20s &  1h46m30s \\
800 & $<1$s & 1.5s & 11.2s &  1h43m30s &  3h36m40s &  2h41m20s \\
900 & $<1$s & 1.6s & 11.3s &  3h15m00s &  7h23m20s &  8h21m40s \\
1000 & $<1$s & 1.8s & 13.0s &  5h11m40s &  8h31m40s &  9h26m40s \\
1500 & $<1$s & 2.7s & 12.7s & 20h10m00s & 29h10m00s & 30h00m00s \\
2000 & $<1$s & 3.2s & 12.5s & 46h40m00s & 31h56m40s & 33h03m20s \\
5000 & $<1$s & 17.0s & 18.3s & --- & --- & --- \\
10000 & 1.4s &  1m26s & 42.7s & --- & --- & --- \\
\bottomrule
\end{tabular}
\end{table}

\begin{table}\centering\small
\caption{$\alpha=1.0$, $\beta=0.9$, $k=1$}
\begin{tabular}{r|r|rr|rrr}
\toprule
n & BSR & GD & LBFGS & CVX(tol=0.01) & CVX(tol=0.001) & CVX(tol=0.0001) \\
\midrule
100 & $<1$s & 28.5s &  1m39s & 2.3s &  1m46s & 50m20s \\
200 & $<1$s &  1m02s &  2m08s & 25.5s & 31m00s &  7h15m00s \\
300 & $<1$s &  1m37s &  3m07s &  1m45s &  1h53m40s & 21h03m20s \\
400 & $<1$s &  2m25s &  3m34s &  5m48s &  5h00m00s & 50h00m00s \\
500 & $<1$s &  2m39s &  4m05s &  7m38s & 12h28m20s & 71h40m00s \\
600 & $<1$s &  5m10s &  5m32s & 26m30s & 25h48m20s & 43h20m00s \\
700 & $<1$s &  4m22s &  5m30s & 37m10s & 49h26m40s & 94h26m40s \\
800 & $<1$s &  4m22s &  5m42s &  1h09m40s & 57h30m00s & 175h16m40s \\
900 & $<1$s &  5m37s &  6m25s &  2h04m50s & 36h23m20s & 305h33m20s \\
1000 & $<1$s &  6m53s &  6m32s &  2h19m20s & 53h03m20s & 260h00m00s \\
1500 & $<1$s & 10m24s & 10m08s & 13h13m20s & 242h30m00s & timeout \\
2000 & $<1$s & 13m39s & 12m45s & 33h03m20s & 64h10m00s & timeout \\
5000 & 1.1s &  1h10m17s & 44m58s & --- & --- & --- \\
10000 & $<1$s &  6h06m27s &  2h30m09s & --- & --- & --- \\
\bottomrule
\end{tabular}
\end{table}

\begin{table}\centering\small
\caption{$\alpha=1.0$, $\beta=0.0$, $k=1$}
\begin{tabular}{r|r|rr|rrr}
\toprule
n & BSR & GD & LBFGS & CVX(tol=0.01) & CVX(tol=0.001) & CVX(tol=0.0001) \\
\midrule
100 & $<1$s & 1.4s & 12.3s & 6.8s & 6.4s & 19.6s \\
200 & $<1$s & 3.7s & 21.8s & 45.3s &  2m53s &  4m32s \\
300 & $<1$s & 3.3s & 25.1s &  3m00s & 11m09s &  7m37s \\
400 & $<1$s & 9.5s & 33.8s &  6m34s & 19m30s & 52m00s \\
500 & $<1$s & 19.3s & 50.4s & 16m10s & 41m40s & 22m50s \\
600 & $<1$s & 9.6s & 42.9s & 51m20s &  2h26m40s &  2h16m20s \\
700 & $<1$s & 20.3s & 54.3s &  1h47m10s &  4h40m00s &  5h45m00s \\
800 & $<1$s & 33.0s &  1m00s &  3h06m40s &  9h40m00s & 12h18m20s \\
900 & $<1$s & 57.2s &  1m22s &  5h06m40s & 14h05m00s & 16h28m20s \\
1000 & $<1$s &  1m12s &  1m39s &  7h03m20s & 21h13m20s & 28h20m00s \\
1500 & $<1$s &  1m54s &  1m39s & 24h51m40s & 80h16m40s & 53h20m00s \\
2000 & $<1$s &  5m39s &  2m56s & 61h56m40s & 110h33m20s & 146h23m20s \\
5000 & 2.5s & 35m20s &  4m19s & --- & --- & --- \\
10000 & 1.6s &  5h11m15s & 13m51s & --- & --- & --- \\
\bottomrule
\end{tabular}
\end{table}

\begin{table}\centering\small
\caption{$\alpha=0.9999$, $\beta=0.9$, $k=1$}
\begin{tabular}{r|r|rr|rrr}
\toprule
n & BSR & GD & LBFGS & CVX(tol=0.01) & CVX(tol=0.001) & CVX(tol=0.0001) \\
\midrule
100 & $<1$s & 29.9s &  1m35s & 6.5s &  5m12s &  1h14m00s \\
200 & $<1$s &  1m05s &  2m13s & 34.4s & 33m20s & 10h03m20s \\
300 & $<1$s &  1m38s &  2m49s &  2m52s &  1h41m10s & 29h43m20s \\
400 & $<1$s &  2m30s &  3m29s &  2m35s &  6h30m00s & 30h00m00s \\
500 & $<1$s &  2m29s &  4m03s &  7m17s & 11h55m00s & 75h50m00s \\
600 & $<1$s &  4m57s &  4m58s & 19m00s & 26h40m00s & 41h40m00s \\
700 & $<1$s &  5m13s &  7m24s & 42m20s & 43h36m40s & 99h10m00s \\
800 & $<1$s &  3m59s &  5m13s & 53m40s & 56h23m20s & 167h30m00s \\
900 & $<1$s &  5m04s &  5m47s &  1h27m30s & 34h10m00s & 302h46m40s \\
1000 & $<1$s &  6m53s &  6m12s &  2h15m10s & 54h10m00s & 274h10m00s \\
1500 & $<1$s & 10m05s & 13m04s & 12h23m20s & 308h20m00s & timeout \\
2000 & $<1$s & 13m37s &  9m39s & 35h00m00s & timeout & timeout \\
5000 & 2.5s &  1h10m23s & 28m03s & --- & --- & --- \\
10000 & 2.5s &  6h06m15s &  1h54m24s & --- & --- & --- \\
\bottomrule
\end{tabular}
\end{table}

\begin{table}\centering\small
\caption{$\alpha=0.9999$, $\beta=0.0$, $k=1$}
\begin{tabular}{r|r|rr|rrr}
\toprule
n & BSR & GD & LBFGS & CVX(tol=0.01) & CVX(tol=0.001) & CVX(tol=0.0001) \\
\midrule
100 & $<1$s & 1.2s & 13.7s & 5.1s & 21.5s & 17.3s \\
200 & $<1$s & 3.1s & 20.5s & 55.7s &  3m45s &  7m08s \\
300 & $<1$s & 2.9s & 22.7s &  2m58s &  4m10s &  8m32s \\
400 & $<1$s & 9.3s & 34.4s &  6m59s & 31m30s & 26m20s \\
500 & $<1$s & 19.3s & 48.9s & 30m40s &  1h03m00s & 41m40s \\
600 & $<1$s & 8.2s & 41.8s &  1h24m40s &  2h06m40s &  1h22m10s \\
700 & $<1$s & 17.7s & 47.9s &  1h54m40s &  5h25m00s &  4h38m20s \\
800 & $<1$s & 30.3s & 57.8s &  2h26m50s &  8h36m40s &  7h43m20s \\
900 & $<1$s & 45.0s &  1m15s &  4h36m40s & 13h58m20s & 15h35m00s \\
1000 & $<1$s &  1m01s &  1m23s &  8h23m20s & 18h46m40s & 22h43m20s \\
1500 & $<1$s &  1m28s &  1m33s & 28h03m20s & 80h33m20s & 91h23m20s \\
2000 & $<1$s &  3m56s &  2m11s & 64h26m40s & 98h36m40s & 126h06m40s \\
5000 & 2.5s & 58m03s &  4m57s & --- & --- & --- \\
10000 & 2.5s &  6h00m13s & 11m30s & --- & --- & --- \\
\bottomrule
\end{tabular}
\end{table}

\begin{table}\centering\small
\caption{$\alpha=0.999$, $\beta=0.9$, $k=1$}
\begin{tabular}{r|r|rr|rrr}
\toprule
n & BSR & GD & LBFGS & CVX(tol=0.01) & CVX(tol=0.001) & CVX(tol=0.0001) \\
\midrule
100 & $<1$s & 32.6s &  1m52s & 3.5s &  6m20s &  1h28m10s \\
200 & $<1$s &  1m00s &  2m13s & 32.2s & 26m30s &  8h36m40s \\
300 & $<1$s & 59.6s &  2m20s &  2m21s &  2h50m00s & 28h20m00s \\
400 & $<1$s &  2m08s &  2m56s &  5m53s &  3h36m40s & 58h03m20s \\
500 & $<1$s &  4m05s &  4m19s &  9m40s & 13h16m40s & 72h30m00s \\
600 & $<1$s &  2m03s &  3m35s & 12m33s & 22h50m00s & 44h10m00s \\
700 & $<1$s &  2m30s &  3m51s & 47m50s & 46h06m40s & 101h40m00s \\
800 & $<1$s &  3m33s &  4m01s &  1h06m20s & 78h03m20s & 145h00m00s \\
900 & $<1$s &  4m08s &  4m39s &  1h40m20s & 88h36m40s & 305h33m20s \\
1000 & $<1$s &  5m35s &  5m31s &  2h42m00s & 50h00m00s & timeout \\
1500 & $<1$s & 10m10s &  7m29s & 10h10m00s & 280h33m20s & timeout \\
2000 & $<1$s & 13m47s &  9m12s & 26h08m20s & timeout & timeout \\
5000 & $<1$s &  1h10m19s & 21m05s & --- & --- & --- \\
10000 & 2.5s &  6h08m30s & 45m33s & --- & --- & --- \\
\bottomrule
\end{tabular}
\end{table}

\begin{table}\centering\small
\caption{$\alpha=0.999$, $\beta=0.0$, $k=1$}
\begin{tabular}{r|r|rr|rrr}
\toprule
n & BSR & GD & LBFGS & CVX(tol=0.01) & CVX(tol=0.001) & CVX(tol=0.0001) \\
\midrule
100 & $<1$s & 1.1s & 11.5s & 5.9s & 6.3s & 32.4s \\
200 & $<1$s & 2.4s & 16.8s & 47.4s &  1m30s &  1m54s \\
300 & $<1$s & 6.3s & 31.3s &  2m48s &  7m17s & 17m40s \\
400 & $<1$s & 5.0s & 27.8s &  8m10s & 17m30s & 42m50s \\
500 & $<1$s & 9.4s & 33.0s & 13m49s & 45m50s &  1h06m00s \\
600 & $<1$s & 13.6s & 40.9s & 50m20s &  1h04m20s &  3h21m40s \\
700 & $<1$s & 19.3s & 48.0s &  1h53m00s &  3h06m40s &  3h21m40s \\
800 & $<1$s & 25.5s & 52.3s &  2h45m50s &  7h03m20s &  8h45m00s \\
900 & $<1$s & 34.5s &  1m07s &  3h55m00s &  9h23m20s & 13h15m00s \\
1000 & $<1$s & 42.4s &  1m10s &  5h46m40s & 17h08m20s & 18h21m40s \\
1500 & $<1$s &  1m35s &  1m40s & 26h55m00s & 54h10m00s & 56h40m00s \\
2000 & $<1$s &  2m45s &  2m14s & 60h33m20s & 52h13m20s & 82h30m00s \\
5000 & $<1$s & 21m16s &  4m32s & --- & --- & --- \\
10000 & 2.4s &  1h59m54s & 12m44s & --- & --- & --- \\
\bottomrule
\end{tabular}
\end{table}

\begin{table}\centering\small
\caption{$\alpha=0.99$, $\beta=0.9$, $k=1$}
\begin{tabular}{r|r|rr|rrr}
\toprule
n & BSR & GD & LBFGS & CVX(tol=0.01) & CVX(tol=0.001) & CVX(tol=0.0001) \\
\midrule
100 & $<1$s & 15.8s &  1m12s & 3.6s &  1m53s & 35m00s \\
200 & $<1$s & 15.6s &  1m14s & 31.8s & 14m57s &  9h00m00s \\
300 & $<1$s & 22.5s &  1m15s &  1m45s & 51m30s & 22h15m00s \\
400 & $<1$s & 33.6s &  1m33s &  3m56s &  2h19m00s & 61h40m00s \\
500 & $<1$s & 41.1s &  1m33s & 17m20s &  3h31m40s & 73h53m20s \\
600 & $<1$s & 47.8s &  1m40s & 31m50s &  8h40m00s & 41h40m00s \\
700 & $<1$s & 54.5s &  1m50s & 51m30s & 20h26m40s & 89h26m40s \\
800 & $<1$s &  1m03s &  1m53s &  1h12m40s & 39h43m20s & 140h00m00s \\
900 & $<1$s &  1m12s &  1m55s &  2h09m30s & 56h06m40s & 270h00m00s \\
1000 & $<1$s &  1m05s &  2m02s &  3h20m00s & 61h06m40s & 265h16m40s \\
1500 & $<1$s &  1m53s &  2m05s & 13h46m40s & 124h26m40s & timeout \\
2000 & $<1$s &  2m16s &  2m14s & 32h30m00s & 319h26m40s & timeout \\
5000 & $<1$s & 11m48s &  3m20s & --- & --- & --- \\
10000 & 2.5s &  1h02m59s &  7m50s & --- & --- & --- \\
\bottomrule
\end{tabular}
\end{table}

\begin{table}\centering\small
\caption{$\alpha=0.99$, $\beta=0.0$, $k=1$}\label{tab:optimizers_last}
\begin{tabular}{r|r|rr|rrr}
\toprule
n & BSR & GD & LBFGS & CVX(tol=0.01) & CVX(tol=0.001) & CVX(tol=0.0001) \\
\midrule
100 & $<1$s & 1.0s & 9.9s & 3.8s & 8.1s & 22.1s \\
200 & $<1$s & 1.7s & 14.0s & 30.8s &  1m15s &  2m02s \\
300 & $<1$s & 2.5s & 18.8s &  1m18s &  5m23s &  8m00s \\
400 & $<1$s & 3.0s & 19.6s &  3m45s &  7m41s &  7m43s \\
500 & $<1$s & 3.5s & 23.9s &  9m51s & 14m17s & 38m30s \\
600 & $<1$s & 4.4s & 26.4s & 33m20s & 35m00s & 45m20s \\
700 & $<1$s & 4.3s & 24.7s &  1h02m20s &  1h22m40s &  1h42m10s \\
800 & $<1$s & 4.8s & 27.8s &  1h31m40s &  2h27m50s &  2h56m40s \\
900 & $<1$s & 5.4s & 28.7s &  2h11m10s &  4h01m40s &  4h26m40s \\
1000 & $<1$s & 5.6s & 30.3s &  2h51m40s &  5h45m00s &  6h45m00s \\
1500 & $<1$s & 8.5s & 39.3s & 11h21m40s & 17h55m00s & 17h26m40s \\
2000 & $<1$s & 11.5s & 42.9s & 26h46m40s & 45h00m00s & 42h46m40s \\
5000 & $<1$s & 57.1s &  1m07s & --- & --- & --- \\
10000 & 1.3s &  4m58s &  2m42s & --- & --- & --- \\
\bottomrule
\end{tabular}
\end{table}

\subsection{Expected Approximation Error}
Figures~\ref{fig:optimizers_first} and~\ref{fig:optimizers_last} report the 
expected approximation errors achieved by the different optimizers of 
AOF~\eqref{eq:AOF} and by BSR. 
For CVX, we report the smallest error value across all tolerance levels for 
which the optimization converged. 

The curves show several clear trends. GD and LBFGS generally perform similarly,
and achieve expected approximation errors slightly (at most a few percent) 
lower than BSR. 
An exception are the problems with momentum ($\beta=0.9$) in the single 
participation setting, where it appears that SGD occasionally fails to 
find the optimum for large problem sizes ($n\geq 1500$).
CVX performs comparably to the other methods for problems without 
momentum ($\beta=0$). 
With momentum, however, the solutions it find are often worse than 
the other methods, especially in the single-participation 
setting and for medium to large problem sizes ($n\geq 500$). 
Presumably, even smaller \emph{tolerance} values would be require 
here, which, however, would result in even longer runtimes.

\begin{figure}[t]
    \centering
    \includegraphics[width=0.48\linewidth]{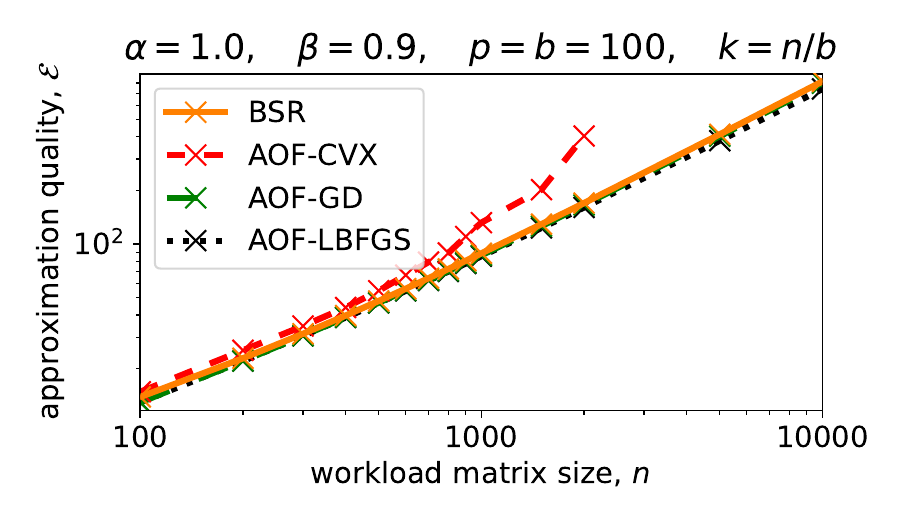}
    \quad   
    \includegraphics[width=0.48\linewidth]{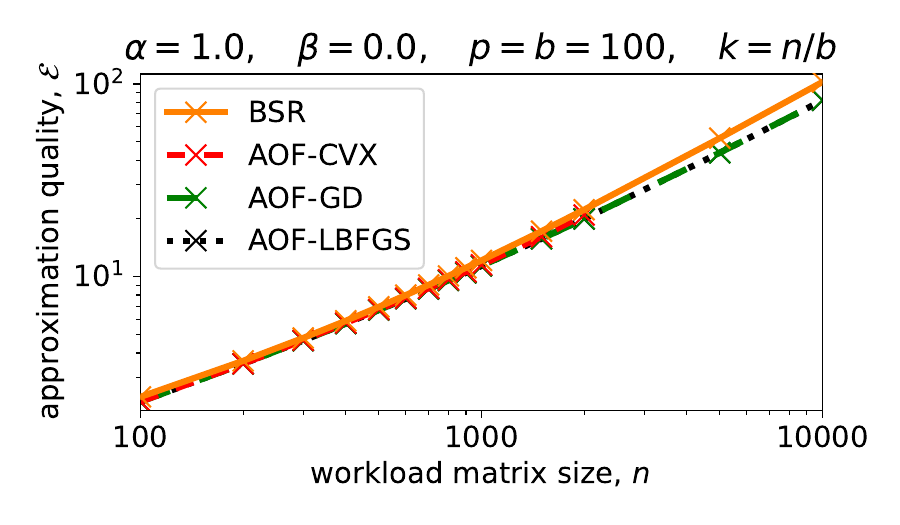}
\\
    \includegraphics[width=0.48\linewidth]{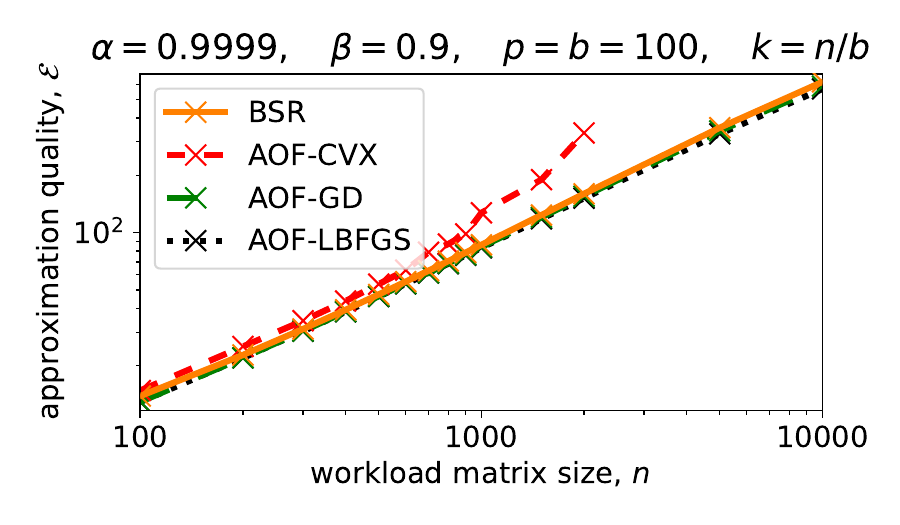}
    \quad   
    \includegraphics[width=0.48\linewidth]{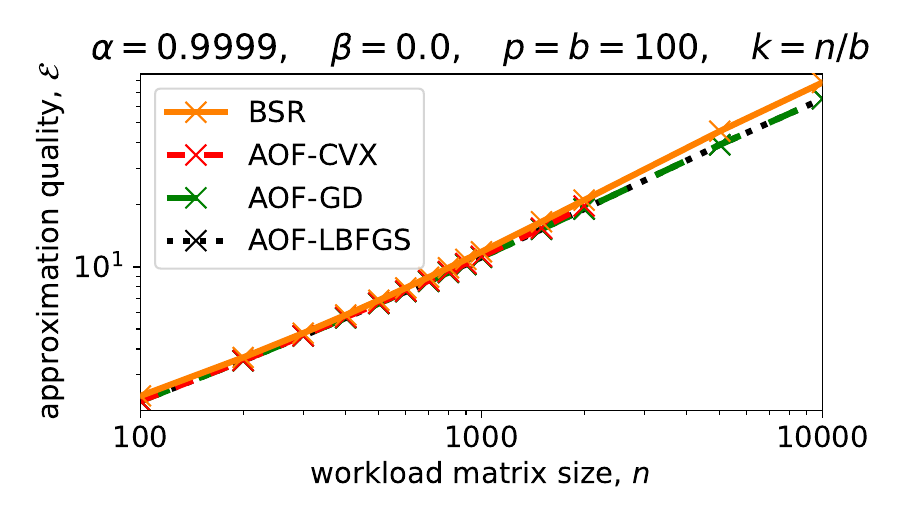}
\\
    \includegraphics[width=0.48\linewidth]{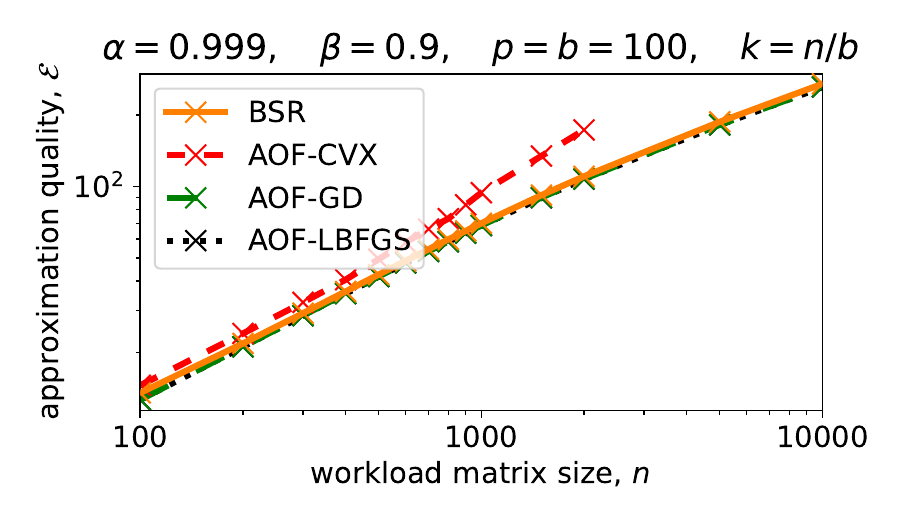}
    \quad   
    \includegraphics[width=0.48\linewidth]{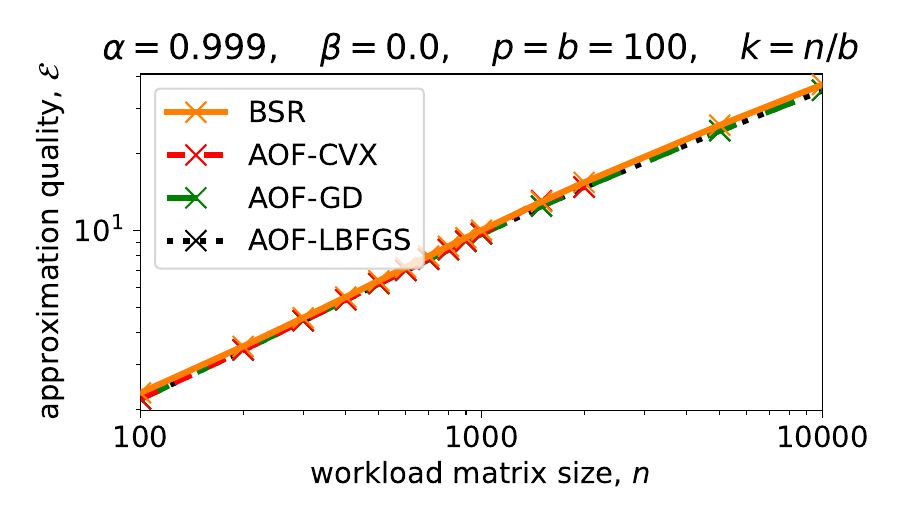}
\\
    \includegraphics[width=0.48\linewidth]{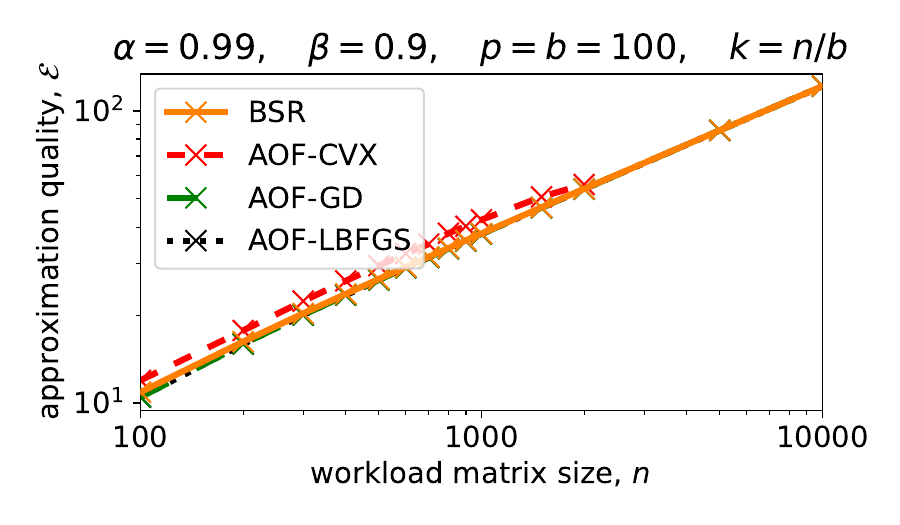}
    \quad   
    \includegraphics[width=0.48\linewidth]{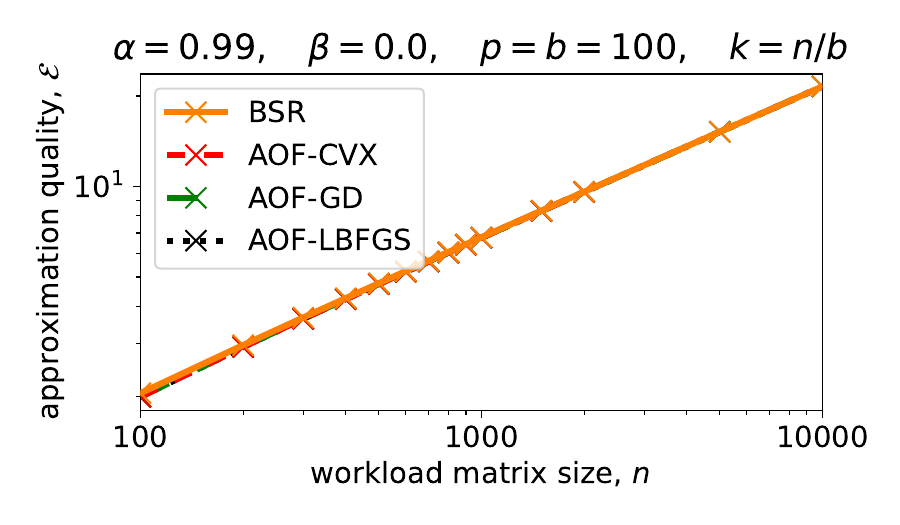}
\caption{Expected approximation error for AOF with GD, LBFGS and CVX optimizers as well as for BSR in the multiple participations setting.}\label{fig:optimizers_first}
\end{figure}

\begin{figure}[t]
    \centering
    \includegraphics[width=0.48\linewidth]{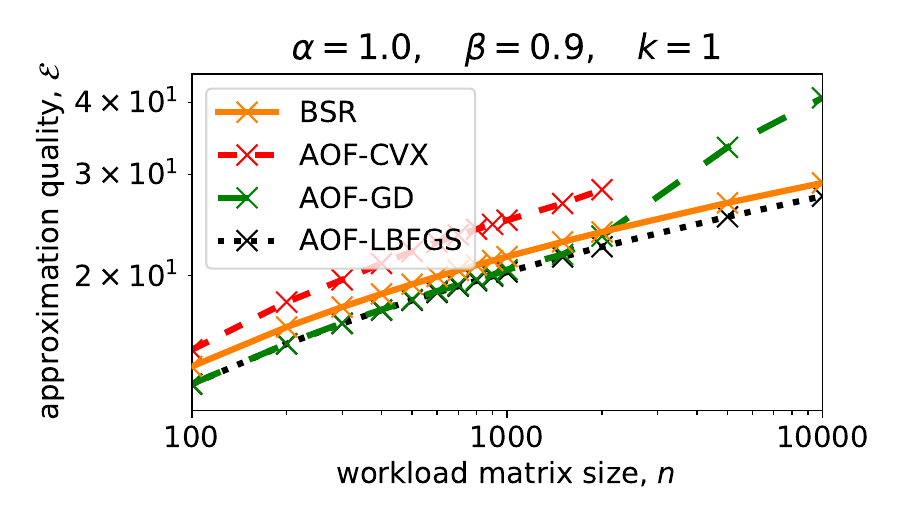}
    \quad   
    \includegraphics[width=0.48\linewidth]{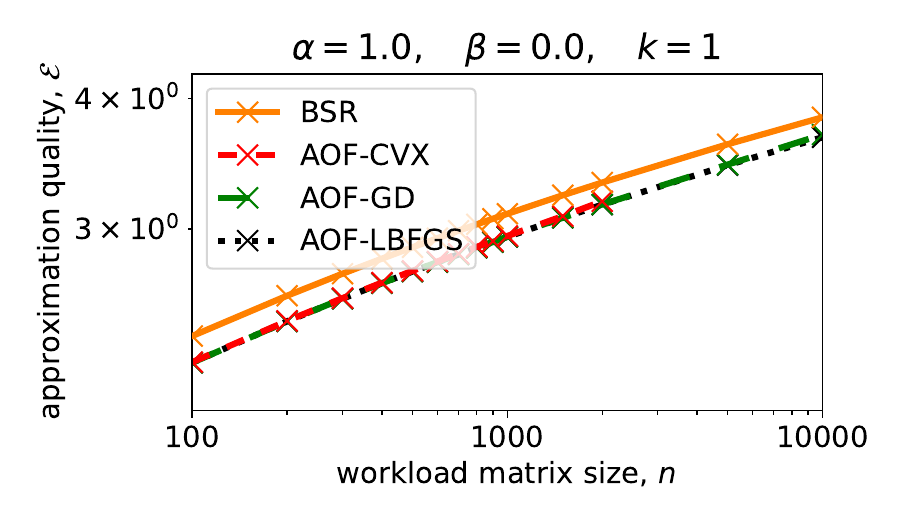}
\\
    \includegraphics[width=0.48\linewidth]{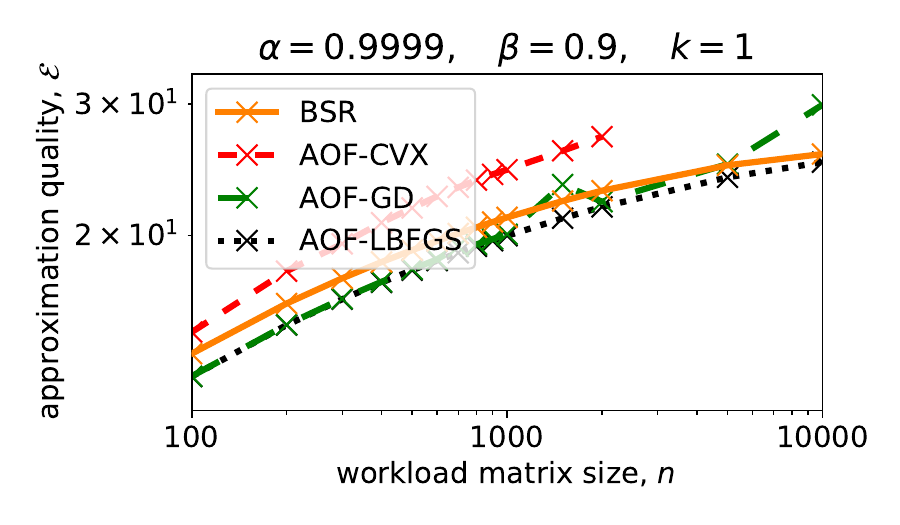}
    \quad   
    \includegraphics[width=0.48\linewidth]{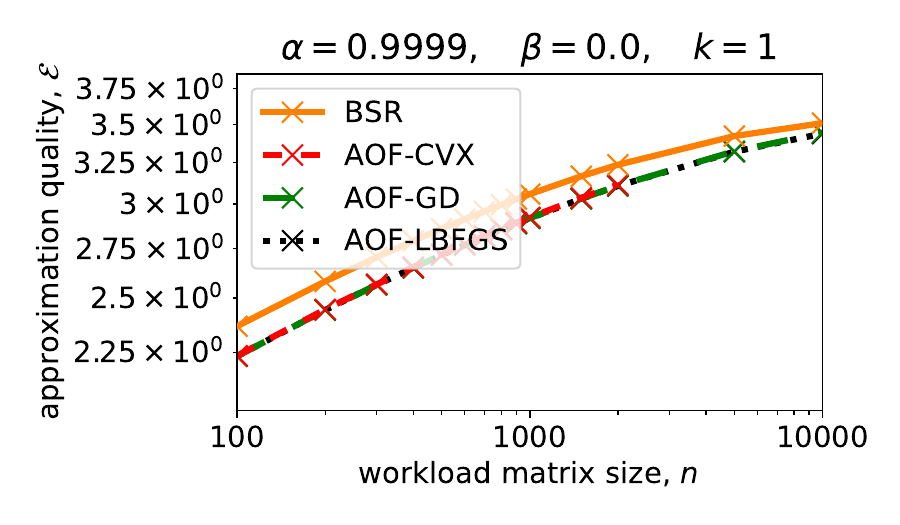}
\\
    \includegraphics[width=0.48\linewidth]{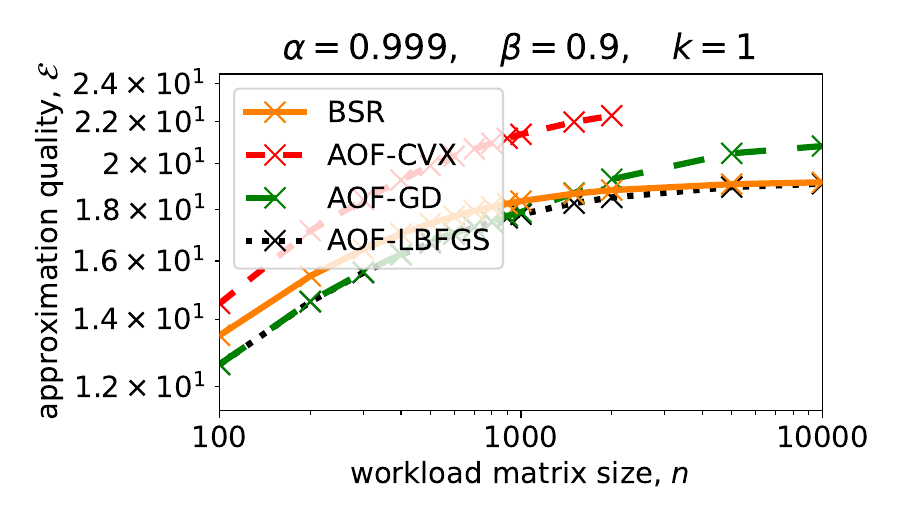}
    \quad   
    \includegraphics[width=0.48\linewidth]{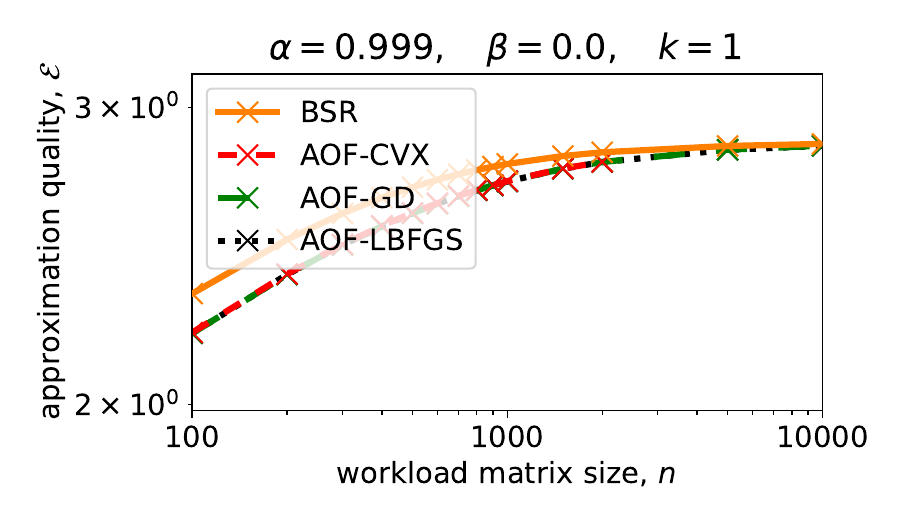}
\\
    \includegraphics[width=0.48\linewidth]{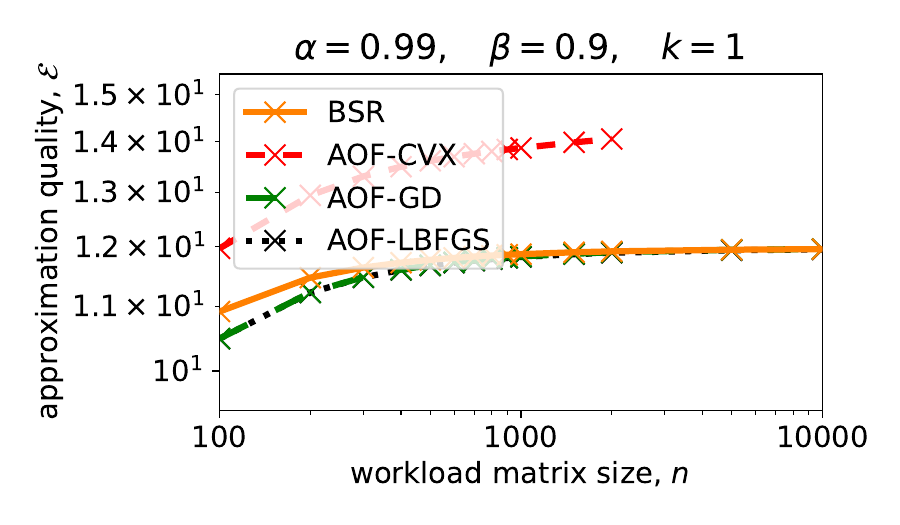}
    \quad   
    \includegraphics[width=0.48\linewidth]{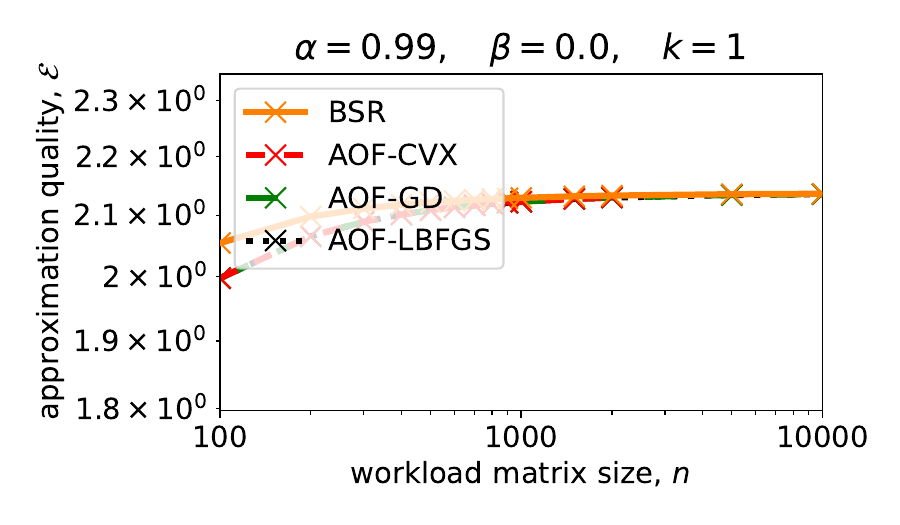}
\caption{Expected approximation error for AOF with GD, LBFGS and CVX optimizers as well as for BSR in the single participation setting.}\label{fig:optimizers_last}
\end{figure}

\clearpage
\section{Complete Proofs}\label{sec:proofs}


In this section, we provide the complete proofs for our 
results from the main manuscript. For the convenience of the
reader, we also restate the statements themselves. 

\subsection{Proof of Theorem~\ref{thm:SGDM_Root}}\label{sec:SGDM_Root_proof}
\SGDroot*
\begin{proof}
We observe that $A_{\alpha, \beta}$ can be written as  
\begin{align}A_{\alpha, \beta} =\begin{pmatrix}
    1 & 0 & \dots &0 \\
    \alpha & 1 & \dots & 0\\
    \vdots & \vdots & \ddots & \vdots \\
    \alpha^{n - 1} & \alpha^{n - 2} & \dots & 1\\
\end{pmatrix}\times \begin{pmatrix}
    1 & 0 & \dots &0 \\
    \beta & 1 & \dots & 0\\
    \vdots & \vdots & \ddots & \vdots \\
    \beta^{n - 1} & \beta^{n - 2} & \dots & 1\\
\end{pmatrix} =: E_\alpha\times E_\beta
\end{align}

Relying on the result from \citet{Henzinger}, that 
%
$\displaystyle 
    E_1^{1/2} = \begin{pmatrix}
    1 & 0 & \dots & 0 \\
    r_1 & 1 & \dots & 0 \\
    \vdots & \vdots & \ddots & \vdots\\
    r_{n - 1} & r_{n - 2} & \dots & 1
\end{pmatrix}$
%
with $r_k = \left|\binom{-1/2}{k}\right|$, 
one can check that the square roots of the matrices $E_\alpha, E_\beta$ are:

\begin{align}
E_\alpha^{1/2} &= \begin{pmatrix}
    1 & 0 & \dots & 0 \\
    \alpha r_1 & 1 & \dots & 0 \\
    \vdots & \vdots & \ddots & \vdots\\
    \alpha^{n - 1} r_{n - 1} & \alpha^{n - 2}r_{n - 2} & \dots & 1
\end{pmatrix} \;\;\; E_\beta^{1/2} = \begin{pmatrix}
    1 & 0 & \dots & 0 \\
    \beta r_1 & 1 & \dots & 0 \\
    \vdots & \vdots & \ddots & \vdots\\
    \beta^{n - 1}r_{n - 1} & \beta^{n - 2}r_{n - 2} & \dots & 1
\end{pmatrix}.
\end{align}

An explicit check yields that these matrices commute, \ie $E_\alpha^{1/2} E_\beta^{1/2} = E_\beta^{1/2} E_\alpha^{1/2}$. Therefore

\begin{align}
C_{\alpha, \beta} &= A_{\alpha, \beta}^{1/2} = E_\alpha^{1/2} \times E_\beta^{1/2} = \begin{pmatrix}
    1 & 0 & \dots & 0\\
    c_1 & 1 & \dots & 0 \\
    \vdots & \vdots & \ddots & \vdots\\
    c_{n - 1} & c_{n - 2} & \dots & 1\\
\end{pmatrix}, \ \ \text{with}\ \ c_k = \sum_{i = 0}^{k} \alpha^i r_i r_{k - i} \beta^{k-i}.
\end{align}
\end{proof}

\subsection{Proof of Theorem \ref{thm:b-sensitivity}}
\label{sec:b-sensitivity-proof}
\toeplitzsens*

\begin{proof}
Because all entries of $M$ are positive, so are the entries of $M^\top M$. 
Therefore, the condition is fulfilled such that \eqref{eq:sensPi} holds 
with equality, and  
\begin{align}
\sens^2_{k,b}(M) &= \max_{\pi \in \Pi_{k,b}}
\sum_{i, j \in \pi}(M^\top\!M)_{[i, j]}
= 
\max_{\pi \in \Pi_{k,b}} f(\pi,\pi)
\quad\text{for}\quad f(\pi,\pi')=\sum_{i \in \pi}\sum_{j \in \pi'} \langle M_{[\cdot, i]}, M_{[\cdot, j]}\rangle,
\end{align}
where $\Pi_{k,b} = \{\ \pi\subset \{1,\dots,n\} \ :\  |\pi|\leq k \wedge (\{i, j\} \subset \pi \ \Rightarrow\ i=j \ \vee\ |i-j|\geq b)\ \}$.
We now establish that $\{1,1+b,\dots,1+(k-1)b\}$ is an optimal index set, which
implies the statement of the theorem.

To see this, let $\pi^*$ be any optimal solution and let $i^*\in\pi^*$ be a column index that is not as far left as possible, \ie, $\pi^*\setminus\{i^*\}\cup\{i^*-1\}$ would be a valid index set in $\Pi_{k,b}$. 
If such an index $i^*$ exists, we split $\pi^*$ into the 
\emph{left} indices, which are smaller than $i^*$, and 
the remaining, \emph{right}, ones:
$\pi^* = \pi^*_L\dot\cup \pi^*_R$ with 
$\pi^*_L=\{i \mid i\in\pi \wedge i<i^*\}$, 
$\pi^*_R=\{i \mid i\in\pi \wedge i\geq i^*\}$.
Then, we construct a new index set in which the 
left indices are kept but all right ones are
shifted by one position to the left:
$\pi' = \pi^*_L \cup \overleftarrow{\pi}^*_R$
with $\overleftarrow{\pi}^*_R=\{i-1 \mid i\in\pi^*_R\}$.
By the condition on $i^*$, we know $\pi'\in\Pi_{k,b}$.

We now prove that $f(\pi',\pi')\geq f(\pi,\pi)$, so $\pi'$
must also be optimal. 
First, we observe two inequalities: for any $i,j>1$:
\begin{align}
\langle M_{[\cdot, i-1]}, M_{[\cdot, j-1]} \rangle 
&=\langle M_{[\cdot, i]}, M_{[\cdot, j]} \rangle + m_{n - i + 1}m_{n - j + 1} 
\geq 
\langle M_{[\cdot, i]}, M_{[\cdot, j]} \rangle,\label{eq:Mshiftboth}
\end{align}
and for $i\geq 1, j>1$:
\begin{align}
    \langle M_{[\cdot, i]}, M_{[\cdot, j-1]} \rangle 
    &= \sum_{l=1}^n M_{[l, i]}, M_{[l, j-1]} = \sum_{l=j-1}^n m_{l-i}m_{l-j+1}
    \\
    &= \sum_{l=j-1}^{n-1} m_{l-i}m_{l-j+1} \quad+\quad  m_{n-i}m_{n-j} 
    \\
    &\geq  \sum_{l=j}^{n} m_{l-i-1}m_{l-j} 
    \\
    &\geq \sum_{l=j}^n m_{l-i}m_{l-j} = \sum_{l=1}^n M_{[l, i]}, M_{[l, j]}  = \langle M_{[\cdot, i]}, M_{[\cdot, j]} \rangle,    \label{eq:Mshiftone}
\end{align}
where the last inequality holds because by assumption $m_{l-i-1}\geq m_{l-i}$ for $l\geq i+1$.

Now, we split $f(\pi',\pi')$ and $f(\pi',\pi')$ into three terms: 
the inner products of indices below $i^*$, the ones of terms
above $i^*$ and the ones between both,
\begin{align}
    f(\pi^*,\pi^*) &= f(\pi^*_L, \pi^*_L) 
    + f(\pi^*_R,\pi^*_R) + 2f(\pi^*_L,\pi^*_R).
\\
    f(\pi',\pi') &= f(\pi^*_L, \pi^*_L) + f(\overleftarrow{\pi}^*_R,\overleftarrow{\pi}^*_R)
    + 2f(\pi^*_L,\overleftarrow{\pi}^*_R).
\end{align}
The first term appears identically in both expressions.
The second term fulfills 
\begin{align}
f(\overleftarrow{\pi}^*_R,\overleftarrow{\pi}^*_R)
&=\sum_{i,j\in\overleftarrow{\pi}^*_R} \langle M_{[\cdot, i]}, M_{[\cdot, j]}\rangle
=\sum_{i,j\in\pi^*_R} \langle M_{[\cdot, i-1]}, M_{[\cdot, j-1]}\rangle
\\
&\geq \sum_{i,j\in\pi^*_R} \langle M_{[\cdot, i]}, M_{[\cdot, j]} \rangle
= f(\pi^*_R,\pi^*_R)
\end{align}
by Equation~\eqref{eq:Mshiftboth}.
The third term fulfills 
\begin{align}
    f(\pi^*_L,\overleftarrow{\pi}^*_R)
    &= \sum_{i\in\pi^*_L}\sum_{j\in\overleftarrow{\pi}^*_R}
    \langle M_{[\cdot, i]}, M_{[\cdot, j]} \rangle 
    = \sum_{i\in\pi^*_L}\sum_{j\in\pi^*_R}
    \langle M_{[\cdot, i]}, M_{[\cdot, j-1]} \rangle 
    \\
    &\geq \sum_{i\in\pi^*_L}\sum_{j\in\pi^*_R}
    \langle M_{[\cdot, i]}, M_{[\cdot, j]} \rangle = f(\pi^*_L,\pi^*_R)
\end{align}
by Equation~\eqref{eq:Mshiftone}.
In combination, this establishes $f(\pi',\pi')\geq f(\pi^*,\pi^*)$,
and since $\pi^*$ was already optimal, the same must hold for $\pi'$.

Using the above construction, we can create a new optimal index 
sets, $\pi^*$, until reaching one that does not contain any 
index $i^*$ as described anymore. 
Then $\pi^*=\{1,1+b,\dots,1+(l-1)b\}$ for some $l\in\mathbb{N}$ 
must hold. 
If $l=k$, the statement of Theorem~\ref{thm:b-sensitivity} is confirmed.
Otherwise, $\pi'=\{1,\dots,1+(k-1)b\}$ is superset of $\pi^*$, so 
because of the positivity of entries, $f(\pi',\pi')\geq f(\pi^*,\pi^*)$
must hold. 
Once again, because $\pi^*$ was optimal, the same must hold for $\pi'$, 
which concludes the proof.
\end{proof}

\subsection{Proof of Corollary  \ref{cor:decreasing_entities}}\label{sec:decreasing_entities_proof}

\BSRsensitivity*

\begin{proof}
From \eqref{thm:SGDM_Root} we know that $C_{\alpha, \beta}$ 
is a Toeplitz matrix with coefficients $(1,c_1,\dots,c_{n-1})$,
where 
$c_j=\sum_{i=0}^j \alpha^i r_i r_{j-i} \beta^{j-i}$
for 
$0\leq \beta < \alpha \leq 1$, with $r_i=|\binom{-1/2}{i}|
= \frac{B_i}{4^i}$,
where $B_i=\binom{2i}{i}$ is the $i$-\emph{central binomial coefficient}.
It suffices to show that $c_j\geq c_{j+1}$ for any $j\in\{1,\dots,n-1\}$.

First, we show for the $r_i$ coefficients: 
\begin{align}
r_i - r_{i+1} &= \frac{1}{4^i}\binom{2i}{i} - \frac{1}{4^{i+1}}\binom{2i+2}{i+1} 
= \frac{1}{4^i}\frac{(2i)!}{i!\,i!}  -\frac{1}{4^{i+1}}\frac{(2i+2)!}{(i+1)!\,(i+1)!}
\\
&= \frac{1}{4^i}\frac{(2i)!}{i!\,i!}\Big(1 -\frac{1}{4}\frac{(2i+2)(2i+1)}{(i+1)\,(i+1)}\Big)
= r_i\Big(1 -\frac{2i+1}{2(i+1)}\Big)
\\
&=\frac{r_i}{2(i+1)} = \frac{1}{4^i\cdot 2}C_{i+1}
\end{align}
where $C_j=\frac{1}{j+1}B_j = \frac{1}{j+1}\binom{2j}{j}$ is the $j$-th \emph{Catalan number}.

Now, we study the case $\alpha=1$. 
If $\beta=0$, then $c_1=c_2=\cdots=c_n=1$, so monotonicity is fulfilled.
Otherwise, \ie $0<\beta<1$, we write

\begin{align}
c_{k}-c_{k+1} &= \sum_{i=0}^{k}r_i (r_{k - i} - r_{k + 1 - i})\beta^{i} - r_{k + 1}r_0 \beta^{k + 1}
\\
&\geq \frac{1}{4^k}\sum_{i = 0}^{k}\frac{1}{2}B_i C_{k - i}\beta^{k-i} - \frac{1}{4^{k +1}}B_{k + 1}\beta^{k + 1}
\\
& = \frac{\beta^{k}}{4^{k + 1}} \left[ 2\sum_{i = 0}^{k}B_i C_{k - i}\beta^{-i} - B_{k + 1}\beta\right]
\intertext{Using the classic identity between Catalan numbers, $2\sum_{i = 0}^{k}B_{k - i} C_{i} = B_{k + 1}$, \eg \cite[Identity 4.2]{Catalan} we obtain}
& = \frac{\beta^{k}}{4^{k + 1}} \left[ 2\sum_{i = 0}^{k}B_i C_{k - i}(\beta^{-i} - \beta)\right] > 0,
\end{align}
where the last inequality follow from the fact that $\beta^{-i}-\beta>0$ for each $i=0,\dots,k$ and any $\beta<1$. This proves the monotonicity of $c_k$. 

For $\alpha<1$, we observe that $c_j=\alpha^j\sum_{i=1}^j r_i r_{k-i}\
\gamma^{j-i}$ for $\gamma=\frac{\alpha}{\beta}<1$. Clearly, the sequence 
$\alpha^j$ is decreasing, and by the above argument, the sum is 
decreasing, too. Consequently, $c_j$ is the product of two decreasing sequences, so it is also decreasing, which concludes the proof.
\end{proof}

\subsection{Useful Lemmas}
Before the remaining proofs, we establish a number of useful lemmas.

\begin{lemma}\label{lem:sens_n_identity}
For any $C\in\R^{n\times n}$ with $C^\top C\geq 0$ it holds for any $b\in\{1,\dots,n\}$ that 
\begin{align}
\sens_{1,b}^2(C) = \|C\|_{2,\infty},
\label{eq:sens_n_identity}
\end{align}
where $\|C\|^2_{2,\infty}=\max_{i=1,\dots,n} \|C_{[\cdot,i]}\|^2$.
\end{lemma}
\begin{proof}
This follows directly from Theorem~\ref{thm:b-sensitivity}: 
\begin{align}
\sens_{1,b}^2(C) = \max_{\pi\in\Pi_{1,b}} \sum_{i,j\in\pi} [C^\top C]_{i,j}
= \max_{i=1,\dots,n}[C^\top C]_{i,i}  
= \max_{i=1,\dots,n} \|C_{[\cdot,i]}\|^2.
\end{align}
\end{proof}

\begin{lemma}\label{lem:sens_inequalities}
For any $C\in\R^{n\times n}$ with $C^\top C\geq 0$ it holds 
for any $b\in\{1,\dots,n\}$ and $k\in\{1,\dots,\frac{b}{n}\}$ that 
\begin{align}
\frac{k}{n}\|C\|^2_{\Fr}  &\leq \sens_{k,b}^2(C) \leq k \|C\|_{\Fr}^2,
\label{eq:sens_upper_lower_bound}
\end{align}
\end{lemma}

\begin{proof}
We first show the upper bound.
Observe that for any $\pi\subset[n]$:
\begin{align}
\sum_{i,j\in\pi} [C^\top C]_{i,j}
&= \sum_{i,j\in\pi} \langle C_{[\cdot, i]}, C_{[\cdot, j]} \rangle
\leq \sum_{i,j\in\pi}
\|C_{[\cdot, i]}\|\|C_{[\cdot, j]}\|
\\
&=(\sum_{i\in\pi} \|C_{[\cdot, i]}\|)^2
\leq 
|\pi|\sum_{i\in\pi} \|C_{[\cdot, i]}\|^2
\leq 
|\pi|\|C\|^2_{\Fr}
\intertext{Therefore, using Theorem~\ref{thm:b-sensitivity}:}
\sens^2(C) &= \max_{\pi\in\Pi_{k,b}}\sum_{i,j\in\pi} [C^\top C]_{i,j} \leq k\|C\|^2_{\Fr}
\end{align}

For the lower bound, we introduce some additional notation.
Let $\tilde\Pi_k$ be the set of $b$-separated index sets with \emph{exactly} $k$ elements.
Then, from 
Theorem~\ref{thm:b-sensitivity}, we obtain
\begin{align}
\sens^2_{k,b}(C)&= \max_{\pi\in\Pi_{k,b}} \sum_{i,j\in\pi}[C^\top\!C]_{ij}
\geq \max_{\pi\in\Pi_{k,b}} \sum_{i\in\pi}[C^\top\!C]_{ii}
=\max_{\pi\in\Pi_{k,b}} S(\pi) \geq \max_{\pi\in\tilde\Pi_k} S(\pi), 
\label{eq:sens_sum_norms}
\end{align}
with the notation $S(I) = \sum_{i\in I}\|C_{[\cdot,i]}\|^2$ for any index set $I\subset \{1,\dots,n\}$.

Now, we prove by backwards induction over $k=1,\dots,\frac{n}{b}$:
\begin{align}
\max_{\pi\in\tilde\Pi_k} S(\pi) &\geq \frac{k}{n}\|C\|^2_{\Fr}
\label{eq:sens_Pi_geq_Fr}
\end{align}

As base case, let $k=\frac{n}{b}$. Denote by $\pi_{i} := \{i, i+b, i+2b, \dots, i+(n-b)\}$ for $i=1,\dots,b$ the uniformly spaced index sets. 
By construction they all fulfill $\pi_i\in\tilde\Pi_{n/b}$ and $\bigcup_{i=1}^n \pi_i = [n]$, where the union is disjoint. 
Therefore
\begin{align}
\max_{\pi\in\tilde\Pi_{n/b}} S(\pi)&\geq \max_{i=1,\dots,b}S(\pi_i) \geq 
\frac{1}{b}\sum_{i=1}^b S(\pi_i) = \frac{1}{b}S([n]) = \frac{1}{b}\|C\|^2_{\Fr}.
\end{align}
This proves the statement~\eqref{eq:sens_Pi_geq_Fr}, because $\frac{1}{b}=\frac{k}{n}$ in this case. 

As an induction step, we prove that if \eqref{eq:sens_Pi_geq_Fr} holds for some value $k\leq\frac{n}{b}$, then it also holds for $k-1\geq 1$. 

Let $\pi^*\in\operatorname{argmax}_{\pi\in\tilde\Pi_{k}}S(\pi)$
and $j^*=\operatorname{argmin}_{j\in\pi^*}S(\{j\})$, such that 
we know that $S(\{j\})\leq \frac{1}{k}S(\pi^*)$.
Now, set $\pi'=\pi^*\setminus\{j^*\}$. Because $\pi'\in\tilde\Pi_{k-1}$, 
it follows that
\begin{align}
    \max_{\pi\in\tilde\Pi_{k-1}} S(\pi) &\geq S(\pi') 
    = S(\pi^*) - S(\{j\}) \geq \frac{k-1}{k}S(\pi^*) 
    \geq \frac{k-1}{n}\|C\|^2_{\Fr},
\end{align}
where in the last step we used the induction hypothesis. This concludes the proof.
\end{proof}

\begin{lemma}\label{lem:Ecal_C_alpha_beta}
For $C_{\alpha, \beta}$ as in \eqref{thm:SGDM_Root}, and $k=1$, it holds that 
\begin{align}
\frac{1}{n} \sum_{j=1}^{n}\sum_{i=0}^{n-j}c_i^2 \leq 
\Ecal(C_{\alpha, \beta},C_{\alpha, \beta}) &\leq \sum_{i = 0}^{n - 1} c_i^2
\end{align}
\end{lemma}

\begin{proof}
From Lemmas~\ref{lem:sens_inequalities} and~\ref{lem:sens_n_identity} 
we obtain
\begin{align}
\Ecal(C_{\alpha, \beta},C_{\alpha, \beta}) &\leq \frac{1}{\sqrt{n}}\|C_{\alpha, \beta}\|_{\Fr}\sens_{1,b}(C_{\alpha, \beta})
\leq \Big(\sens_{1,b}(C_{\alpha, \beta})\Big)^2
\leq \|C\|^2_{2,\infty} = \sum_{i=0}^{n-1} c_i^2,
\end{align}
where the last identify follows from the explicit form of $C_{\alpha, \beta}$.
The lower bound follows from
\begin{align}
\Ecal(C_{\alpha, \beta},C_{\alpha, \beta}) &= \frac{1}{\sqrt{n}}\|C_{\alpha, \beta}\|_{\Fr}\sens_{1,b}(C_{\alpha, \beta})
\geq \frac{1}{n}\|C\|^2_{\Fr} 
\end{align}
and again the explicit form of $\|C\|^2_{\Fr}$.
\end{proof}

\begin{lemma}\label{lem:r_bounds}
For $r_j = |\binom{-1/2}{j}| = \frac{1}{4^j}\binom{2j}{j}$ it holds that:
\begin{align}
r_0 &= 1 \qquad\text{and}\qquad r_1=\frac12\qquad\text{and in general}\qquad 
\frac{1}{2\sqrt{j}}\le r_j \le  \frac{1}{\sqrt{\pi j}}
\label{eq:r_bounds}
\qquad\text{for $j\geq 1$.}
\end{align}
\end{lemma}

\begin{proof}
The double inequality is a particular case of a more general pair of binomial inequalities when $k = j$ and $m = 2j$:
    \begin{equation}
        \label{eq:binomial_inequality}
        \sqrt{\frac{m}{8k(m -k)}}2^{mH(k/m)} \le \binom{m}{k} \le \sqrt{\frac{m}{2\pi k(m -k)}}2^{mH(k/m)},
    \end{equation}
where $H(k/m)$ is the binary entropy function, with $H(1/2) = 1$.
The proof of the general result \eqref{eq:binomial_inequality}, can be found in \citet[Chapter 10, Lemma 7, p309]{Binomials}.
\end{proof}

\begin{lemma}\label{lem:c_bounds}
Let $c_k=\sum_{j=0}^k \alpha^j r_j r_{k-j}\beta^{k-j}$ as in \eqref{thm:SGDM_Root}. Then $c_0=1$, and for $j\geq 1$: 
\begin{align}
    \frac{\alpha^j}{2 \sqrt{j+1}} &\leq c_j \leq \frac{\alpha^j}{(1 - \frac{\beta}{\alpha}) \sqrt{j+1}}.
\end{align}
\end{lemma}

\begin{proof}
We exploit the upper and lower bounds from Lemma~\ref{lem:r_bounds}.
First, we write  $c_k= \alpha^k\sum_{j=0}^k r_j r_{k-j}\gamma^{j}$ 
with $\gamma:=\frac{\beta}{\alpha}$. 
Then we check immediately that $c_0=1$ and $c_1 = \frac{1}{2}(\alpha+\beta) = \frac{\alpha}{2}(1+\gamma)\leq \frac{\alpha}{2}\frac{1}{1-\gamma}$.

For $j\geq 2$ we derive the upper bound by
\begin{align}
\frac{c_j}{\alpha^j} &= r_j(1+\gamma^j) + \sum_{i=1}^{j-1} r_i r_{j - i}\gamma^{i}
\leq \frac{1+\gamma^j}{\sqrt{\pi j}} + \sum_{i=1}^{j-1} \frac{\gamma^{i}}{\pi \sqrt{i(j-i)}}
\\
&\leq \frac{1+\gamma^j}{\sqrt{\pi j}} + \sum_{i=1}^{j-1} \frac{\gamma^{i}}{\pi \sqrt{j-1}}
\leq \frac{\sqrt{\pi}-1}{\pi\sqrt{j}} +\frac{\sqrt{\pi}-1}{\pi\sqrt{j}}\gamma^j + \frac{1}{\pi \sqrt{j-1}}\sum_{i=0}^{j}\gamma^i
\\
&= \frac{\sqrt{\pi}-1}{\pi \sqrt{j}}(1+\gamma^j) + \frac{1}{\pi \sqrt{j-1}}\frac{1}{(1-\beta)}
\\
&= \underbrace{\frac{(\sqrt{\pi}-1)\sqrt{j+1}}{\sqrt{j}}}_{\leq 1}\frac{(1+\gamma^kj)}{\pi \sqrt{j+1}} + \underbrace{\frac{\sqrt{j+1}}{\sqrt{j-1}}}_{\leq \sqrt{3}\leq 2}\frac{1}{\pi \sqrt{j+1}}\frac{1}{(1-\gamma)}
\\
&\leq \frac{3}{\pi \sqrt{j+1}}\frac{1}{(1-\gamma)}
\leq \frac{1}{\sqrt{j+1}}\frac{1}{(1-\gamma)},
\end{align}
which proves the upper bound on $a_j$. The lower bound for $j\geq 1$ follows trivially from

\begin{equation}
c_j  \geq \alpha^j r_j \geq \frac{\alpha^j}{2\sqrt{j}} \geq \frac{\alpha^j}{2\sqrt{j+1}} 
\end{equation}
\end{proof}

\begin{lemma}\label{lem:sum_ci_bounds}
For $j\in \{1,\dots,n\}$ it holds
\begin{align}
\frac{\log(j+1)}{4} &\leq \sum_{i=0}^{j-1} c^2_i \leq \frac{1+\log j}{(1 - \beta)^2}
\intertext{for $\alpha=1$, and otherwise}
1 \leq  \sum_{i=0}^{j-1} c^2_i &\leq\frac{1}{(\alpha - \beta)^2}  \log\left(\frac{1}{1 - \alpha^2}\right) 
\end{align}   
\end{lemma}

\begin{proof}
We first prove the result for $\alpha=1$. 
Combining Lemmas~\ref{lem:Ecal_C_alpha_beta} and \ref{lem:c_bounds} we obtain
\begin{align}
   \sum_{i=0}^{j-1}c_i^2 &\leq 
   \frac{1}{(1 - \beta)^2}\sum_{i=0}^{j-1} \frac{1}{i+1}
   =
   \frac{1}{(1 - \beta)^2}\sum_{i=1}^{j} \frac{1}{i}
   \leq \frac{1+\log j}{(1 - \beta)^2}.
   \\
   \sum_{i=0}^{j-1}c_i^2 &\geq 
   \frac{1}{4}\sum_{i=0}^{j-1} \frac{1}{i+1}
   =
   \frac{1}{4}\sum_{i=1}^{j} \frac{1}{i} \geq \frac{\log(j+1)}{4}.
\end{align}
For $\alpha<1$, if follows analogously:
\begin{align}
   \sum_{i=0}^{j-1}c_i^2 &\leq 
   \frac{1}{(1 - \frac{\alpha}{\beta})^2}\sum_{i=0}^{j-1} 
   \frac{\alpha^{2i}}{i+1} 
    \leq \frac{1}{(\alpha - \beta)^2}  \sum_{i=1}^{\infty} \frac{\alpha^{2i}}{i} = \frac{1}{(\alpha-\beta)^2}\log\left(\frac{1}{1 - \alpha^2}\right).
\\
   \sum_{i=0}^{j-1}c_i^2 &\geq 
   \frac{1}{4} \sum_{i=0}^{j-1} \frac{\alpha^{2i}}{i+1}
   = \frac{1}{4\alpha^2} \sum_{i=1}^{j} \frac{\alpha^{2i}}{i}
   = \frac{1}{4\alpha^2} \Big[\sum_{i=1}^{\infty} \frac{\alpha^{2i}}{i}
   - \sum_{i=j+1}^{\infty} \frac{\alpha^{2i}}{i}\Big]
   \\
   &\geq \frac{1}{4\alpha^2} \Big[\log\left(\frac{1}{1 - \alpha^2}\right) - \frac{\alpha^{2(j+1)}}{(j+1)(1-\alpha^2)}\Big],
\end{align}
where the last term emerges from 
$\sum_{i=j+1}^{\infty} \frac{\alpha^{2i}}{i}
\geq 
\frac{\alpha^{2(j+1)}}{j+1}\sum_{i=0}^{\infty} \alpha^{2i}
= \frac{\alpha^{2(j+1)}}{j+1}\frac{1}{1-\alpha^2}$.
\end{proof}

\begin{lemma}\label{thm:Aalphabeta_norms}
Let $0\leq\beta<\alpha\leq 1$. Let $\sigma_1 \geq \dots \geq \sigma_n$ be the sorted 
list of singular values of $A_{\alpha,\beta}$. If $\alpha<1$, then for $j=1,\dots,n$:
\begin{align}
    \frac{1}{(1+\alpha)(1+\beta)} \leq \sigma_j  &\leq \frac{1}{(1-\alpha)(1-\beta)}
\intertext{and}
    n \leq \|A_{\alpha,\beta}\|_* &\leq \frac{n}{(1-\alpha)(1-\beta)}.
\end{align}
If $\alpha=1$, then for $j=1,\dots,n$,\footnote{The following two inequalities contained typos in the conference version; please refer to the current manuscript for the corrected versions.} 
\begin{align}
    \frac{2}{\pi}\frac{1}{1+\beta}\frac{n + \frac{1}{2}}{j - \frac{1}{2}} \leq \sigma_j &\leq \frac{1}{1-\beta}\frac{n + \frac{1}{2}}{j - \frac{1}{2}}
\intertext{and consequently}
    \frac{2}{\pi}\frac{n\log(n+1)}{1+\beta} \leq \|A_{1,\beta}\|_* &\leq \frac{(n+\frac{1}{2})(3+\log (n - 1))}{1-\beta}
\end{align}
\end{lemma}

\begin{proof}
The statements on the singular values follow from the following Lemma~\ref{lem:singular_values_Et}, because $A_{\alpha,\beta}=E_\alpha E_\beta$. 
Because $E_\alpha$ and $E_\beta$ are diagonalizable and they commute, 
we have $\sigma_n(E_\beta)\sigma_j(E_\alpha) \leq \sigma_j(E_\alpha E_\beta)\leq \sigma_1(E_\beta)\sigma_j(E_\alpha)$. For $\alpha<1$ the lower bound follows from 
$\|A_{\alpha,\beta}\|_*\geq \trace A_{\alpha,\beta}$, and the upper bound follows
from the identity $\|A_{\alpha,\beta}\|_*=\sum_{j=1}^n \sigma_j$.

For $\alpha=1$, the bounds follow from the same identity together with the 
fact that 
\begin{align}
\log(n+1) \leq \sum_{j=1}^n \frac{1}{j} & \le \sum\limits_{j = 1}^{n} \frac{1}{j - \frac{1}{2}} \le 2 + \sum\limits_{j =1}^{n - 1}\frac{1}{j}\leq 3 + \log(n - 1).
\end{align}
\end{proof}

\begin{lemma}[Singular values of $E_t$] \label{lem:singular_values_Et}
For $0\leq t\leq 1$, let $E_{t}= \toep(1,t,\dots,t^{n-1})\in\R^{n\times n}$. 
Then the singular values $\sigma_1(E_t)\geq \dots\geq \sigma_n(E_t)$ fulfill
for $i=1,\dots,n$:
\begin{align}
    \frac{1}{1+t} &\leq \sigma_i(E_t) \leq \frac{1}{1-t} \quad \text{for $0\leq t<1$,}
\qquad\text{and}\qquad
    \sigma_i(E_1) = \frac{1}{\sin\big(\frac{i-\frac12}{n+\frac12}\frac{\pi}{2}\big)}. \label{eq:sigmaE}
\end{align}
\end{lemma}
    
\begin{proof}
We follow the steps of \citet{SebastienB}, and use that the singular values of $E_t$ are 
the reciprocals of the singular values of $E_t^{-1}$, which themselves are the eigenvalues 
of $(E_t)^{-1}((E_t)^{-1})^{\top}=:T$,
\ie, for $i=1,\dots,n$:
\begin{align}
\sigma_j(E_t)=\frac{1}{\sqrt{\lambda_{n+1-j}(T)}} \label{eq:sigma_from_lambda}
\end{align}

The $E^{-1}_t$ and $T$ can be computed explicitly as 
\begin{align}
E^{-1}_{t} &= \begin{pmatrix}
    1 & 0 & 0 & 0 & \dots &0 \\
    -t & 1 & 0 & 0 & \dots &  0\\
    0 & -t & 1  & 0 & \dots & 0\\
    \vdots & \ddots & \ddots & \ddots & \vdots & \vdots \\
    0 & \dots & 0 &  -t & 1 & 0\\
    0 & \dots & 0 & 0 &  -t & 1\\
\end{pmatrix},
\qquad
T=\begin{pmatrix}
    1 & -t & 0 & \dots &0 \\
    -t & 1+t^2 & -t & \dots & 0\\
    0 & \ddots & \ddots & \vdots & \vdots \\
    0 & \dots & -t & 1+t^2 & -t\\
    0 & \dots & 0 & -t & 1+t^2\\
\end{pmatrix} \label{eq:Einv_and_T}
\end{align}

\begin{lemma}\label{lem:mu_upper_lower}
All eigenvalues, $\mu$, of $T$ fulfill
\begin{align}
(1-t)^2\leq \mu \leq (1+t)^2 \label{eq:mu_upper_lower}
\end{align}
\end{lemma}

\begin{proof}
By Gershgorin's circle theorem~\citep{gershgorin1931uber}, 
we know that $\mu$ fulfills i) $|1-\mu|\leq t$, \ie  
$1-t\leq t \leq \mu\leq 1+t$ or ii) $|1+t^2 - \mu|\leq 2t$, \ie 
$1-2t+t^2\leq \mu \leq 1+2t+t^2$. 
For $t\in[0,1]$ the first condition implies the second, 
so \eqref{eq:mu_upper_lower} must hold.
\end{proof}

\noindent\textbf{Case I:} For $t<1$, the statement~\eqref{eq:sigmaE}
follows from Lemma~\ref{lem:mu_upper_lower} in combination with \eqref{eq:sigma_from_lambda}.

\noindent\textbf{Case II:} For $t=1$ the matrix simplifies to 
$T=\begin{pmatrix}
    1 & -1 & 0 & \dots &0 \\
    -1 & 2 & -1 & \dots & 0\\
    0 & \vdots & \ddots & \vdots & \vdots \\
    0 & \dots & -1 & 2 & -1\\
    0 & \dots & 0 & -1 & 2\\
\end{pmatrix}$. Note that $T$ is not exactly Toeplitz, because of the top 
left entry, so closed-form expressions for the eigenvalues of tridiagonal 
Toeplitz matrices do not apply to it. %
Instead, we can compute its eigenvalues explicitly.
Matrices of this form have been studied by \citet{elliottsingular1953}; for completeness, we provide a full proof here.

Let $\mu$ be an eigenvalue of $T$ with eigenvector $\Psi=(\Psi_0,\dots,\Psi_{n-1})$. 
From the eigenvector equation $T\Psi=\mu\Psi$ we obtain 
\begin{align}
    \mu\Psi_0 &= \Psi_0 - \Psi_1 \\
    \mu\Psi_k &= -\Psi_{k-1} + 2\Psi_k - \Psi_{k+1} \qquad\text{for $k=1,\dots,n-2$}\\
    \mu\Psi_{n-1} &= -\Psi_{n-2} + 2\Psi_{n-1}\\
\intertext{which yields a linear recurrence relation}
    \Psi_{k+1} &= (2-\mu)\Psi_{k} - \Psi_{k-1} \quad\text{for $k=1,\dots,n-2$} \label{eq:psirecurrence}
\intertext{with two boundary conditions}
    \Psi_1 &= (1-\mu)\Psi_0  \label{eq:psi_first}\\
    \Psi_{n-2} &= (2-\mu)\Psi_{n-1}. \label{eq:psi_last}
\end{align}
We solve the recurrence relation using the polynomial method~\citep{greene1990mathematics}.
The characteristic polynomial of~\eqref{eq:psirecurrence} is $P(z)=z^2+(\mu-2)z+1$. Its roots are 
\begin{align}
    r_{\pm} &= \frac{2-\mu}{2}\pm \sqrt{\Big(\frac{2-\mu}{2}\Big)^2-1} 
     = \frac{(2-\mu)\pm i\sqrt{4-(2-\mu)^2}}{2} = e^{\pm i\theta}
\end{align}
for some value $\theta\in[0,2\pi)$. 
Note that the expression under the second square root is positive, because of Lemma~\ref{lem:mu_upper_lower}. 
The last equation is a consequence of, $|r_{\pm}|^2 = \frac{1}{4} \left( (2-\mu)^2 +(4-(2-\mu)^2)\right) = 1$.
Consequently, 
\begin{align}
\mu &= 2-2\Re(e^{i\theta}) = 2-2\cos\theta.
\end{align}

From standard results on linear recurrence, it follows that any 
solution to \eqref{eq:psirecurrence} has the form $\Psi_j = c_1 (r_{+})^j + c_2 (r_{-})^j$ 
for some constants $c_1,c_2\in\mathbb{C}$. 
The fact that $\Psi_j$ must be real-valued implies that $c_1=c_2=:\alpha e^{i\phi}$
for some values $\alpha\in\R,\phi\in[0,2\pi)$. Dropping the normalization constant 
(which we could recover later if needed), we obtain 
\begin{align}
    \Psi_j &= e^{i(\phi+j\theta)}+e^{-i(\phi+j\theta)} = 2\cos(\phi+j\theta).
\end{align}
Next, we use the boundary conditions to establish values for $\phi$ and $\theta$. 

Equation~\eqref{eq:psi_last} can be rewritten as 
\begin{align}
\cos(\phi+(n-2)\theta) &= 2\cos(\theta)\cos((n-1)\theta)
\intertext{which, using $2\cos(\alpha+\beta) = \cos(a+b) + \cos(a-b)$, simplifies to}
0 &= \cos(\phi+n\theta)
\end{align}
Consequently, $\phi+n\theta = \frac{1}{2}\pi + k\pi$ must hold for some $k\in\mathbb{N}$. 

Equation~\eqref{eq:psi_first} can be rewritten as 
\begin{align}
    \cos(\phi+\theta) &= (2\cos(\theta)-1)\cos(\phi)
    \intertext{which simplifies to}
    \cos(\phi) &= \cos(\theta-\phi)
\end{align}
One solution to this would be $\theta=0$, but that would implies $\mu=0$, 
which is inconsistent with $T$ being an invertible matrix. 
So instead, it must hold that $\phi=\frac{\theta}{2}+k\pi$ for some $k\in\mathbb{N}$.

Combining both conditions and solving for $\theta$ we obtain 
\begin{align}
\theta = \frac{\frac{1}{2}+ k}{n+\frac{1}{2}}\pi = \frac{1}{2}\pi + k\pi
\qquad\text{for some $k\in\mathbb{N}$.}
\end{align}
Each such value $\theta_k$ for $k\in\{0,\dots,n-1\}$ yields an eigenvector 
with associated eigenvalue $\mu=2-2\cos\theta_k = 4\sin^2(\theta_k/2)$.
Now, \eqref{eq:sigmaE} follows from this in combination with \eqref{eq:sigma_from_lambda}.
\end{proof}

\subsection{Proof of Theorem  \ref{thm:approximation_streaming}}\label{sec:proof_approximation_streaming}

\errorstreaming*

\begin{proof}

The proof consists of a combination of Lemmas~\ref{lem:Ecal_C_alpha_beta} and~\ref{lem:sum_ci_bounds}. Because in the single participation $k=1$, so we need just the first column of matrix $C_{\alpha, \beta}$:
\begin{align}
    \Ecal(C_{\alpha,\beta},C_{\alpha,\beta}) &\leq \sum_{i=0}^{n-1} c^2_i 
    \leq\begin{dcases} \frac{1+\log n}{(1 - \beta)^2}
    \quad&\text{for $\alpha=1$,}
    \\
    \frac{1}{(\alpha-\beta)^2}\log\frac{1}{1-\alpha^2} \quad&\text{otherwise.}
    \end{dcases}
\end{align}
which proves the upper bounds. For the lower bounds, for any $\alpha\leq 1$:
\begin{align}
 \Ecal(C_{\alpha,\beta},C_{\alpha,\beta}) &\geq 
    \frac{1}{n}\sum_{j=0}^{n-1} \sum_{i=0}^j c^2_i 
    \geq \frac{1}{n}\sum_{j=0}^{n-1} c_0 = 1.\label{eq:C_alpha_F_norm_lower_bound}   
\end{align}
Also, for $\alpha=1$:
\begin{align}
\label{eq:C_beta_F_norm_lower_bound}
    \Ecal(C_{1,\beta},C_{1,\beta}) &\geq 
    \frac{1}{n}\sum_{j=1}^{n} \sum_{i=0}^{j-1} c^2_i 
    \geq 
    \frac{1}{4n}\sum_{j=1}^{n} \log(j+1)
    =
    \frac{\log(\,(n+1)!\,)}{4n}
    \geq
    \frac{\log(n+1)-1}{4},
\end{align}
because $\log(\,(n+1)!\,)\geq (n+1)\log(n+1) - n$.
\end{proof}

\subsection{Proof of Theorem~\ref{thm:streaming_baselines}}\label{sec:streaming_baselines_proof}

\errorbaselines*

\begin{proof}
For $\alpha = 1$, by Lemma~\ref{lem:sens_n_identity} we have:
\begin{align}
\sens_{1,b}^2(A_{1,\beta}) 
&= \sum_{i=0}^{n-1} a_i^2 
= \sum_{i=0}^{n-1} \Big(\frac{1-\beta^{i + 1}}{1-\beta}\Big)^2
= 
\frac{n}{(1-\beta)^2} - \frac{2\beta}{(1-\beta)^2}\sum_{i=0}^{n-1} \beta^i
+ \frac{\beta^2}{(1-\beta)^2}\sum_{i=0}^{n-1} \beta^{2i}
\\
&=
\frac{n}{(1-\beta)^2} - \frac{2\beta^{n+1}}{(1-\beta)^3}
+ \frac{\beta^{2n+2}}{(1-\beta)^2(1-\beta^2)}
=
\frac{n}{(1-\beta)^2}(1 + o(1)).
\end{align}

For $\alpha<1$:
\begin{align}
\sens_{1,b}^2(A_{\alpha, \beta}) &= \sum_{i=0}^{n-1} a_i^2 
= \sum_{i=0}^{n-1} \frac{(\alpha^{i+1} - \beta^{i+1})^2}{(\alpha-\beta)^2} 
\\
&= \frac{1}{\alpha-\beta}
\Big[\alpha^2\sum_{i=0}^{n-1} (\alpha^2)^i
- 2\alpha\beta\sum_{i=0}^{n-1} (\alpha\beta)^i
+ \beta^2\sum_{i=0}^{n-1} (\beta^2)^i\Big]
\\
&=\frac{1}{(\alpha - \beta)^2}\left[\frac{\alpha^2}{1 - \alpha^2} - \frac{2\alpha\beta}{1 - \alpha\beta} + \frac{\beta^2}{1 - \beta^2} \right] (1 + o(1)) \\
& =  \frac{1 + \alpha\beta}{(1 - \alpha\beta)(1 - \alpha^2)(1 - \beta^2)}(1  + o(1)).
\end{align}

Together with 
\begin{align}
 \|A_\beta\|_F^2/n &= \frac{1}{n}\sum_{j = 0}^{n - 1}(n - j)\left[\sum_{i = 0}^{j}\beta^i\right]^2 = \frac{1}{n}\sum_{j = 0}^{n - 1}(n - j) \left[\frac{1 - \beta^{j +1 }}{1 - \beta}\right]^2 \\
 &= \frac{1}{n(1 - \beta)^2}\sum_{j = 0}^{n - 1}(n - j) (1 - 2\beta^{j + 1} + \beta^{2j + 2}) \\
 &= \frac{(n + 1)}{2(1 - \beta)^2} + O(1) = \frac{n}{2(1 - \beta)^2}(1 + o(1))
\end{align}
as $\beta < 1$ and the sum $\sum_{j = 0}^{n - 1}j\beta^{j + 1}$ is uniformly bounded by $\beta\sum_{j = 0}^\infty j \beta^j  = \frac{\beta^2}{(1 - \beta)^2}$

For $\alpha < 1$:
\begin{align}
 \|A_{\alpha, \beta}\|_F^2/n &= \frac{1}{n} \sum_{j = 0}^{n - 1} (n - j)\frac{(\alpha^{j + 1} - \beta^{j + 1})^2}{(\alpha - \beta)^2} =  \sum_{j = 0}^{n - 1} \frac{(\alpha^{j + 1} - \beta^{j + 1})^2}{(\alpha - \beta)^2} + o(1)\\
 & = \frac{1 + \alpha\beta}{(1 - \alpha\beta)(1 - \alpha^2)(1 - \beta^2)}(1  + o(1))
\end{align}
where the second equality due to the fact that we average over the sequence $jx^{j + 1}$ which converges to $0$ for $|x| < 1$.
\end{proof}

\subsection{Proof of Theorem \ref{thm:bmin_approximation}}\label{sec:proof_bmin_approximation}

\approximationbmin*

\begin{proof}

Consider a Lower Triangular Toeplitz (LTT) matrix multiplication:

\begin{equation}
\begin{pmatrix}
    a_1 & 0 & \dots & 0\\
    a_2 & a_1 & \dots & 0\\
    \dots & \dots & \dots & \dots\\
    a_{n} & a_{n - 1} & \dots & a_1
\end{pmatrix} \times
\begin{pmatrix}
    b_1 & 0 & \dots & 0\\
    b_2 & b_1 & \dots & 0\\
    \dots & \dots & \dots & \dots\\
    b_{n} & b_{n - 1} & \dots & b_1
\end{pmatrix} = 
\begin{pmatrix}
    c_1 & 0 & \dots & 0\\
    c_2 & c_1 & \dots & 0\\
    \dots & \dots & \dots & \dots\\
    c_{n} & c_{n - 1} & \dots & c_1
\end{pmatrix},
\end{equation}
where  $c_j = \sum_{i = 1}^{j}a_i b_{n +1 - i}$ so the LTT structure is preserved with multiplication that allows us to work with sequences and their convolutions rather than matrix multiplication. For instance, we would write the previous product in the form:

\begin{equation}
(a_1, \dots, a_n) * (b_1, \dots, b_n) = (c_1, \dots, c_n).
\end{equation}

The inverse of the Lower Triangular Toeplitz matrix remains a Lower Triangular Toeplitz (LTT) matrix because we can find a unique sequence $(c_1, \dots, c_n)$ such that:

\begin{align}
&(c_1, \dots, c_n) * (a_1, \dots, a_n) = (1, 0, \dots, 0)\\
&c_j = -\frac{1}{a_1}\sum_{i = 1}^{j - 1} c_j a_{j  + 1 - i},\; \text{and}\;\; c_1 = \frac{1}{a_1},
\end{align}
with the restriction that $a_1 \ne 0$; otherwise, the original matrix was not invertible.  We consider the banded square root factorization $A_{\alpha, \beta} = B^{|p|}_{\alpha, \beta} C^{|p|}_{\alpha, \beta}$ which is characterized by the following identity:

\begin{equation}
(b_0, \dots, b_{n - 1}) * (1, c_1, \dots, c_{p - 1}, 0, \dots, 0) = (1, c_1,\dots, c_{n - 1}) * (1, c_1, \dots, c_{n - 1}).
\end{equation}

We will bound the Frobenius norm of the LTT matrix $(b_0, \dots, b_{n - 1})$. By the uniqueness of the solution, we obtain that for the first $p$ values we have $b_i = c_i$. For the next $p$ values we have the following formula:

\begin{align}
b_{p+ j} + \dots + b_{p}c_{j} + \dots + b_{j + 1} c_{p - 1} &= c_{p + j} + \dots + c_{p + 1} c_{p - 1} + \dots  + c_{p + j}\\
b_{p+j} + b_{p + j - 1}c_1 + \dots + b_{p}c_j &= 2\left(c_{p + j} + \dots + c_{p}c_j\right).
\end{align}

By induction argument, we can see that $b_{p + j} = 2 c_{p + j}$ for $ 0 \le j \le p - 1$. 
For the remaining $n - 2p$ values we will prove convergence to a constant.

\begin{equation}
    \sum_{j = 0}^{p - 1} b_{j - i}c_i = a_j = \frac{\alpha^{j + 1} - \beta^{j + 1}}{\alpha - \beta}.
\end{equation}

We make an ansatz for the solution of the linear recurrence in the form:

\begin{equation}
    b_j  = \frac{\alpha^{j + 1}}{(\alpha - \beta)\sum_{i = 0}^{p - 1}c_i\alpha^{-i}} -  \frac{\beta^{j + 1}}{(\alpha - \beta)\sum_{i = 0}^{p - 1}c_i \beta^{-i}} + \alpha^j y_j,
\end{equation}
where $y_j$ represents the error terms, which will be proven to converge to $0$. The sequence $y_j$ satisfies the following recurrence formula:

\begin{equation}
    y_j = -\sum_{i = 1}^{p - 1}y_{j - i}c_i\alpha^{-i}.
\end{equation}

We denote $w_j = c_j \alpha^{-j}$ which is a decreasing sequence because the values correspond to the  $C_{1, \beta/\alpha}$ matrix. We rewrite the recurrence in matrix notation:
\begin{align}
\begin{pmatrix}
    -w_1 & -w_2 & -w_3 &  \dots & -w_{p - 2} & -w_{p - 1}\\
    1 & 0 &  0&  \dots & 0 & 0\\
    0 & 1 &  0&  \dots & 0 & 0\\
    0 & 0 &  1 &  \dots & 0 & 0\\
    \dots & \dots &  \dots &  \dots &  \dots & \dots\\
    0 & 0 & 0 & \dots & 1 & 0
\end{pmatrix} \begin{pmatrix}
    y_{k -1}\\
    y_{k -2}\\
    y_{k - 3}\\
    y_{k - 4}\\
    \dots\\
    y_{k - b}
\end{pmatrix} = 
\begin{pmatrix}
    y_{k}\\
    y_{k -1}\\
    y_{k - 2}\\
    y_{k - 3}\\
    \dots\\
    y_{k - b + 1}
\end{pmatrix}.
\end{align}

To show that the error terms $y_j$ goes to $0$ as $j$ goes to infinity, we first study the 
characteristic polynomial of the associate homogeneous relations: 
\begin{align}
    g(\lambda) = \lambda^{p-1} +  w_1\lambda^{p-2} + \dots + w_{p-2}\lambda + w_{p-1}.
\end{align}
Because $1>w_1> w_2> \dots> w_{b-1} > 0$, it follows from \emph{Schur's (relaxed) stability condition}~\citep[Theorem~1]{nguyen2007relaxed} that all its (complex) roots lie inside of the open unit circle. 
Therefore, all solutions to the homogeneous relation converge to zero at a rate exponential in $j$ and $y_j = o(1)$ and $\sum_{j = 0}^{\infty}y_j^2 = O_{\alpha, \beta, p}(1)$.
Then we can bound the Frobenious norm of the matrix $B_{\alpha, \beta}^{|p|}$ as:
\begin{align}
   \frac1n \|B^{|p|}_{\alpha, \beta}\|_{\Fr}^2 &\le \sum_{j = 0}^{n - 1} b_j^2 \le \sum_{j = 0}^{p - 1} c_j^2 + \sum_{j = p}^{n - 1} \frac{\alpha^{2j + 2}}{(\alpha - \beta)^2\left[\sum_{i = 0}^{p - 1}w_i\right]^2} + \alpha^{2j}y_j^2.
\end{align}

We use the following lower bound for the sum of $w_j$:

\begin{align}
    \sum_{j = 0}^{p - 1}w_j \ge \frac{1}{2}\sum_{j = 0}^{p - 1} \frac{1}{\sqrt{j + 1}} \ge \sqrt{p + 1} - 1.
\end{align}

Combining these bounds we can upper bound the Frobenious norm of the matrix $B^{|p|}_{\alpha, \beta}$ the following way:

\begin{align}
    \|B^{|p|}_{\alpha, \beta}\|_F^2/n &\le \begin{dcases}
        \frac{1}{(\alpha - \beta)^2}\log\left(\frac{1}{1 - \alpha^2}\right) + \frac{\alpha^2}{(\sqrt{p + 1} - 1)^2(\alpha - \beta)^2} + O_{\alpha, \beta, p}(1) \qquad &\text{for $\alpha < 1$} \\
        \frac{1  + \log(p)}{(1 - \beta)^2} +\frac{n - p}{(1 - \beta)^2 (\sqrt{p + 1}  - 1)^2} + O_{p, \beta}(1)&\text{for $\alpha = 1$}. \\
    \end{dcases} 
\end{align}

Simplifying for the leading terms in asymptotics, we have:

\begin{align}
    \|B^{|p|}_{\alpha, \beta}\|_F^2/n &= \begin{dcases}
      O_{\alpha, \beta, p}(1) \qquad &\text{for $\alpha < 1$} \\
       O_{\beta}\left(\frac{n}{p}\right) +  O_{p, \beta}(1) &\text{for $\alpha = 1$}. \\
    \end{dcases} 
\end{align}

\paragraph{Sensitivity of $C^{\b{p}}_\beta$.}
For the $b$-min-separation participation sensitivity we have the following bound for any $p\leq b$:
\begin{align}
\text{sens}_{k,b}^2(C^{|p|}_{\alpha, \beta}) 
\leq k \sum_{j = 0}^{p - 1} c_j^2 \le \begin{dcases}
    \frac{k}{(\alpha - \beta)^2}\log\left(\frac{1}{1 - \alpha^2}\right) \qquad &\text{for $\alpha < 1$}\\
    k\frac{1  + \log(p)}{(1 - \beta)^2} &\text{for $\alpha = 1$}.\\
\end{dcases}
\end{align}

Combining sensitivity with the upper bound for the Frobenious norm we obtain:

\begin{align}
    \Ecal(B^{|p|}_{\alpha, \beta}, C^{|p|}_{\alpha, \beta}) = \begin{dcases}
        O_{p, \alpha, \beta}(\sqrt{k})  &\text{for $\alpha < 1$}\\
        O_{\beta}\left( \sqrt{\frac{nk \log p}{p}}\right) + O_{\beta, p}(\sqrt{k}) &\text{for $\alpha = 1$}.\\
    \end{dcases}
\end{align}
\end{proof}

\subsection{Proof of Theorem \ref{thm:bmin_approximation_square_root} for Square Root Factorization}
\label{sec:proof_bmin_approximation_squareroot}

\approximationbminsquareroot*

\begin{proof}
We prove the case without weight decay ($\alpha=1$) and with weight decay ($\alpha<1$) separately. 

\paragraph{Case 1) no weight decay ($\alpha=1$).} 

We start by bounding the $b$-min-separation sensitivity:

\begin{equation}
\sens_{k, b}^2(C_{1, \beta}) = \sum_{i = 0}^{k - 1}\sum_{j = 0}^{k - 1} \langle (C_{1, \beta})_{[\cdot,ib]}, (C_{1, \beta})_{[\cdot,jb]} \rangle.
\label{eq:innerproductC}
\end{equation}

Consider a scalar product for a general pair of indices, $j > i$:
\begin{equation}
\langle (C_{1, \beta})_{[\cdot,i]}, (C_{1, \beta})_{[\cdot,j]} \rangle 
= \sum_{t = 0}^{n - 1 - j} c_t c_{j - i + t}.
\label{eq:innerproductCij}
\end{equation}

Using the bounds on $c_k$ \eqref{lem:c_bounds} for $\alpha=1$ we can lower and upper bound 
this sum by:

\begin{align}
\sum_{t = 0}^{n - 1 - j} c_t c_{j - i + t}
&\leq 
\frac{1}{(1 - \beta)^2}\sum_{t = 1}^{n - j} \frac{1}{\sqrt{t(j - i + t)}} \le \frac{1}{(1 - \beta)^2}\int_{0}^{n - j}\frac{dx}{\sqrt{x(j - i + x)}}
\\
\sum_{t = 0}^{n - 1 - j} c_t c_{j - i + t}
&\geq 
\frac{1}{4}\sum_{t = 1}^{n - j} \frac{1}{\sqrt{t(j - i + t)}} \geq \frac{1}{4}\int_{1}^{n - j}\frac{dx}{\sqrt{x(j - i + x)}}.
\end{align}
We can compute the indefinite integral explicitly:

\begin{equation}
\int\frac{dx}{\sqrt{x(j - i + x)}} = F\left(\frac{j - i}{x}\right) + C
\end{equation}
for $F(a)=2\log\big(\sqrt{\frac{1}{a} + 1} + \sqrt{\frac{1}{a}}\big)$.
In combination, we obtain the upper and lower bound for~\eqref{eq:innerproductC}:

\begin{equation}
   \frac{1}{4} f\left(\frac{j - i}{n - j}\right) - \frac{1}{4} f(j - i) \le  \langle (C_\beta)_{i}, (C_\beta)_{j} \rangle \le \frac{1}{(1 - \beta)^2} f\left(\frac{j - i}{n - j}\right).
\end{equation}

Now we are ready to bound the sensitivity of the matrix $C_{1, \beta}$:

\begin{align}
\sens^2_{k, b}(C_{1, \beta}) &= \sum_{i = 0}^{k - 1}\langle (C_{1, \beta})_{ib}, (C_{1, \beta})_{ib} \rangle + 2\sum_{i = 0}^{k - 1}\sum_{j = i + 1}^{k - 1}\langle (C_{1, \beta})_{ib}, (C_{1, \beta})_{jb} \rangle\\
&\le  \frac{1}{(1 - \beta)^2} \sum_{i = 0}^{k - 1}(\log (n - ib) + 1) + \frac{2}{(1 - \beta)^2} \sum_{i = 0}^{k - 1}\sum_{j = i + 1}^{k - 1} f\left(\frac{j - i}{k - j}\right)\intertext{and, analogously}
\sens^2_{k, b}(C_{1, \beta}) &\ge  \frac{1}{4} \sum_{i = 0}^{k - 1}\log (n - ib) + \frac{1}{2} \sum_{0\leq i<j\leq k-1}\left[ f\left(\frac{j - i}{k - j}\right) - f(b(j - i))\right]
\label{eq:sens_31}
\end{align}

Firstly, using $(\frac{k}{e})^k \le k! \le k^k$, we bound the sum of the logarithms:
\begin{align}
&\sum_{j = 0}^{k - 1} \log(n - j b) = \sum_{j = 1}^{k} [\log b + \log j] = k \log b + \log k! \le k\log b + k\log k = k \log n,\\
&\sum_{j = 0}^{k - 1} \log(n - j b) = k \log b + \log k! \ge k\log b + k\log k - k = k\log n - k.
\end{align}

To upper bound the last term in sensitivity lower bound~\eqref{eq:sens_31}, 
we use the auxiliary inequality $f(a) = 2\log\left(\sqrt{\frac{1}{a} + 1} + \sqrt{\frac{1}{a}}\right) \le \frac{4}{\sqrt{a}}$ to derive: 

\begin{equation}
\label{eq:third_sum}
\frac{1}{2} \sum_{0\leq i<j\leq k-1}f(b(j - i)) 
\le \frac{2}{\sqrt{b}} \sum_{0\leq i<j\leq k-1}\frac{1}{\sqrt{j - i}} 
=  \frac{2}{\sqrt{b}}\sum_{j = 1}^{k - 1} \sum_{t = 1}^{j} \frac{1}{\sqrt{t}} \le \frac{4}{\sqrt{b}}\sum_{j = 1}^{k - 1} \sqrt{j} \le \frac{8}{3\sqrt{b}}k^{3/2}
\end{equation}

To bound the final term we establish the following inequalities for $f(a)$:

\begin{equation}
    f(a) = 2 \log \left(\sqrt{\frac{1}{a} + 1} + \sqrt{\frac{1}{a}}\right) = \log\left(\frac{1}{a} + 1\right) + 2 \log\left(1 + \frac{1}{\sqrt{a + 1}}\right)
\end{equation}

\begin{equation}
\log\left(\frac{1}{a} + 1\right) < f(a) < \log\left(\frac{1}{a} + 1\right) + 2 \log 2
\end{equation}

Then we can bound the first double sum in~\eqref{eq:sens_31} as
\begin{equation}
\sum_{0\leq i<j\leq k-1}  \log \left(\frac{k - i}{j - i}\right) 
\le 
\sum_{0\leq i<j\leq k-1} f\left(\frac{j - i}{k -j}\right) \le 
\sum_{0\leq i<j\leq k-1} \log \left(\frac{k - i}{j - i}\right) + 2k^2\log 2.
\end{equation}

To bound the term $\sum_{0\leq i<j\leq k-1}\log \big(\frac{k - i}{j - i}\big)$ we use the following identities:

\begin{align}
\sum_{0\leq i<j\leq k-1}\log \left(\frac{k - i}{j - i}\right) 
&= \log \prod_{i = 0}^{k - 2} \frac{(k - i)^{k - i - 1}}{(k - i - 1)!}\\
&=\log \frac{k^{k - 1}(k - 1)^{k - 2}\dots 2^1}{1! \cdot 2! \cdot 3!  \dots (k - 1)!} = \log \frac{2^1 \cdot 3^2 \cdot 4^3 \dots k^{k-1}}{1^{k - 1}\cdot 2^{k - 2} \dots (k - 1)^1}
\\
&= \log \prod_{j = 1}^{k - 1} \left(\frac{j + 1}{k - j}\right)^j = \sum_{j = 1}^{k - 1} j \log(j + 1) - \sum_{j = 1}^{k - 1} j \log (k - j)
\\
&= \sum_{j = 1}^{k} (j - 1)\log (j) - \sum_{j = 1}^{k - 1}(k - j)\log(j)
\\
&=2\sum_{j = 1}^{k - 1} j \log (j) - \log k! + 2k\log k - k\log k! 
\end{align}
Now, using that $x\log x$ is a monotonically increasing function,
\begin{align}
\sum\sum_{0\leq i<j\leq k-1} \log \left(\frac{k - i}{j - i}\right) 
&\le 2 \int_{1}^{k}x\log x dx + k \log k+ k - k^2\log k + k^2
\label{eq:upper_40}
\\
&= k^2\log k - \frac{k^2}{2} + k\log k - k -k^2\log k + k^2 
\\
&\le \frac{3}{2}k^2
\end{align}
As a lower bound, we obtain
\begin{align}
\sum_{0\leq i<j\leq k-1} &\log \left(\frac{k - i}{j - i}\right) 
\ge 2\int_{1}^{k - 1}x \log x dx - k \log k + k \log k - k(k - 1)\log (k - 1)
\label{eq:lower_40}
\\
&= (k - 1)^2 \log (k - 1) - \frac{(k - 1)^2}{2} + k\log k - k(k - 1) \log (k - 1)
\\
&= -(k - 1) \log(k - 1) - \frac{(k - 1)^2}{2}  + k \log k 
\\
&\ge -\frac{k^2}{2}
\end{align}

Therefore, combining the upper bound~\eqref{eq:upper_40} and the lower bound~\eqref{eq:lower_40} yields
\begin{align}   
\sum_{0\leq i<j\leq k-1} f\left(\frac{j - i}{k -j}\right) &\le (2\log2 + 3/2)k^2 \le 3k^2,
\\
\sum_{0\leq i<j\leq k-1} f\left(\frac{j - i}{k -j}\right) &\ge (2\log 2 - 1/2)k^2 \ge \frac{4k^2}{5}.
\end{align}

Combining all three terms together we obtain the following bounds for 
the squared sensitivity~\eqref{eq:sens_31}:

\begin{align}
\label{eq:sens_final_bound}
&\sens^2_{k, b}(C_{1, \beta}) \le  \frac{k}{(1 - \beta)^2}(\log n + 1) +\frac{6}{(1 - \beta)^2} k^2\\
&\sens^2_{k, b}(C_{1, \beta}) \ge  \frac{k}{4}(\log n - 1) - \frac{8}{3\sqrt{b}}k^{3/2} + \frac{2}{5}k^2
\end{align}

Now, we recall the bounds for the Frobenius norm of the matrix $C_{1, \beta}$ \ref{lem:sum_ci_bounds}  and \eqref{eq:C_beta_F_norm_lower_bound}:
\begin{align}
\frac{\log (n + 1) - 1}{4} \le  \|C_{1, \beta}\|_F^2/n  &\le \frac{\log n + 1}{(1 - \beta)^2}.
\end{align}

With the auxiliary inequality $\sqrt{\frac{a}{2}} +\sqrt{\frac{b}{2}}  \le \sqrt{a + b} \le \sqrt{a} + \sqrt{b}$ and combining \eqref{eq:sens_final_bound} and the bounds on 
Frobenius norm \eqref{lem:sum_ci_bounds} and \eqref{eq:C_beta_F_norm_lower_bound} 
we get that:

\begin{align}
&\Ecal(C_{1, \beta}, C_{1, \beta}) \le \frac{\sqrt{k}}{(1 - \beta)^2}(\log n + 1) +\frac{\sqrt{5} k}{(1 - \beta)^2}\sqrt{\log n + 1}\\
&\Ecal(C_{1, \beta}, C_{1, \beta}) \ge \frac{1}{4\sqrt{2}}\sqrt{k} (\log n - 1) +  \frac{k}{2\sqrt{5}}  \sqrt{\log(n) - 1} \sqrt{1 - \frac{20}{3\sqrt{n}}}
\end{align}
Making the lower bound well-defined requires $n \ge 45$, otherwise one can
simply take $\Ecal(C_{1, \beta}, C_{1, \beta}) \geq 1$. 
As a final step, we combine both inequalities in the following asymptotic statement:
\begin{equation}
    \Ecal(C_{1, \beta}, C_{1, \beta}) = \Theta_\beta\left(\sqrt{k} \log n + k \sqrt{\log n}\right),
\end{equation}
which concludes the proof of the case without weight decay.

\paragraph{Case 2) with weight decay ($\alpha<1$).} 

As above, we first express the $b$-min-separation sensitivity of the matrix $C_{\alpha, \beta}$ in terms of inner products, 
\begin{equation}
\sens_{k, b}^2(C_{\alpha, \beta}) = \sum_{i = 0}^{k - 1}\sum_{j = 0}^{k - 1} \langle (C_{\alpha, \beta})_{ib}, (C_{\alpha, \beta})_{jb} \rangle. 
\label{eq:sens_C_alpha_b2}
\end{equation}
and then consider a scalar product for a general pair of indexes $j > i$:
\begin{equation}
\langle (C_{\alpha, \beta})_{i}, (C_{\alpha, \beta})_{j} \rangle = \sum_{t = 0}^{n - 1 - j}  c_t c_{j - i + t}.
\label{eq:scalar_product_Calphabeta}
\end{equation}
Now, we use the bounds on $c_t$ from Lemma~\ref{lem:c_bounds} for $\alpha<1$, 
to upper and lower bound this sum with the following expression, 
where $\gamma = \frac{\beta}{\alpha}$:

\begin{align}
    &\langle (C_{\alpha, \beta})_{i}, (C_{\alpha, \beta})_{j} \rangle \le \frac{\alpha^{j - i}}{(1 - \gamma)^2}\sum_{t = 0}^{n - 1 - j} \frac{\alpha^{2t}}{\sqrt{(t + 1)(j - i + t + 1)}} \le \frac{\alpha^{j - i}}{(1 - \gamma)^2(1 - \alpha^2)\sqrt{j - i}}\\
    &\langle (C_{\alpha, \beta})_{i}, (C_{\alpha, \beta})_{j} \rangle \ge \frac{\alpha^{j - i}}{4}\sum_{t = 0}^{n - 1 - j} \frac{\alpha^{2t}}{\sqrt{(t + 1)(j - i + t + 1)}} \ge \frac{\alpha^{j - i}}{4\sqrt{j - i + 1}}.
\end{align}

We substitute these bounds into Equation~\eqref{eq:sens_C_alpha_b2} to 
obtain the following upper bound for sensitivity of matrix $C_{\alpha, \beta}$:

\begin{align}
    \sens^2_{k, b}(C_{\alpha, \beta}) &\le \sum_{i = 0}^{k - 1} \langle (C_{\alpha, \beta})_{ib}, (C_{\alpha, \beta})_{ib} \rangle + \frac{2}{(1 - \gamma)^2 (1 - \alpha^2)\sqrt{b}}\sum_{j > i} \frac{\alpha^{b(j - i)}}{\sqrt{j - i}}\\
    & \le  \frac{k}{(\alpha-\beta)^2}\log\frac{1}{1-\alpha^2}+ \frac{2}{(1 - \gamma)^2 (1 - \alpha^2)\sqrt{b}} \sum_{j > i} \alpha^{b(j - i)}\\ 
    & \le \frac{k}{(\alpha-\beta)^2}\log\frac{1}{1-\alpha^2} + \frac{2k\alpha^b}{(1 - \gamma)^2 (1 - \alpha^2)(1 - \alpha^b)\sqrt{b}},
\end{align}
where the second inequality is due to Equation~\eqref{eq:C_alpha_F_norm_lower_bound}. 

A lower bound for the sensitivity follows directly from Lemma~\ref{lem:sens_inequalities}:
\begin{align}
    \sens^2_{k, b}(C_{\alpha, \beta}) &\ge k\|C_{\alpha, \beta}\|_{\Fr} \geq k.
\end{align}
The Frobenius norm of the matrix $C_{\alpha, \beta}$ is the same as that for one round participation; thus, we could reuse Inequalities~\eqref{eq:C_alpha_F_norm_lower_bound}:

\begin{align}
    1 \le \|C_{\alpha, \beta}\|_F^2 / n \le \frac{1}{(\alpha-\beta)^2}\log\frac{1}{1-\alpha^2}.
\end{align}

By merging the bounds for sensitivity and Frobenius norm, we derive the following bounds for error:

\begin{align}
&\Ecal(C_{\alpha, \beta}, C_{\alpha, \beta}) \le \sqrt{k} \left[\frac{1}{(\alpha-\beta)^2}\log\frac{1}{1-\alpha^2} + \frac{2\alpha^b}{(1 - \gamma)^2 (1 - \alpha^2)(1 - \alpha^b)\sqrt{b}}\right]\\
&\Ecal(C_{\alpha, \beta}, C_{\alpha, \beta}) \ge \sqrt{k}. 
\end{align}

The combination of these results yields the following asymptotic statement:

\begin{equation}
    \Ecal(C_{\alpha, \beta}, C_{\alpha, \beta}) = \Theta_{\alpha, \beta}(\sqrt{k}),
\end{equation}
which concludes the proof.
\end{proof}

\subsection{Proof of Theorem~\ref{thm:bmin_lower_bound}}\label{sec:bmin_lower_bound_proof}

\lowerboundbmin*

\begin{proof}
Let $A_{\alpha,\beta}=BC$ be any factorization with $CC^\top\geq 0$. 
From Lemma~\ref{lem:sens_inequalities} it follows that 
\begin{align}
    \Ecal(B,C) &= \frac{1}{\sqrt{n}}\|B\|_{\Fr}\sens_{k,b}(C)
    \geq\frac{\sqrt{k}}{n}\|B\|_{\Fr}\|C\|_{\Fr}
    \geq \frac{\sqrt{k}}{n}\|A_{\alpha,\beta}\|_*,
\end{align}
where $\|\cdot\|_*$ denotes the \emph{nuclear norm}, and the 
last inequality follows from its variational form, $\|M\|_*=\min_{\{X,Y : XY^T=M\}}\|X\|_{\Fr}\|Y\|_{\Fr}$.
The statement of the Theorem follows by inserting the 
corresponding bounds on $\|A_{\alpha,\beta}\|_*$ from 
Lemma~\ref{thm:Aalphabeta_norms}.
\end{proof}

\subsection{Proof of Theorem~\ref{thm:bmin_baselines}}\label{sec:bmin_approximation_baselines}
\baselinesbmin*

\begin{proof}
\paragraph{Case 1) $A_{\alpha,\beta}=BC$ with $B=A_{\alpha,\beta}$ and $C=\Id$.}
It is easy to check that $\sens_{k,b}(C)=\sqrt{k}$, so $\Ecal(B,C)=\sqrt{\frac{k}{n}}\|A_{\alpha,\beta}\|_{\Fr}$.
Because $A_{\alpha,0} \leq A_{\alpha,\beta}\leq \frac{1}{\alpha-\beta}A_{\alpha,0}$ componentwise, 
we have for $\alpha=1$,
\begin{align}
\frac{n(n+1)}{2} = \|A_{1,0}\|^2_{\Fr}\leq \|A_{1,\beta}\|^2_{\Fr} 
&\leq \frac{1}{(1-\beta)^2}\|A_{1,0}\|^2_{\Fr} = \frac{n(n+1)}{2(1-\beta)^2},
\end{align}
which implies the corresponding statement of the theorem.
For $0<\alpha<1$, we use that $A_{1,0}>\Id$ componentwise, so $\|A_{\alpha,0}\|^2_{\Fr} > \|\Id\|^2_{\Fr}= n$, which conclude the proof of this case.

\paragraph{Case 2) $A_{\alpha,\beta}=BC$ with $B=\Id$ and $C=A_{\alpha,\beta}$.}
We observe that $\|\Id\|_{\Fr}=\sqrt{n}$, so 
\begin{align}
\Ecal(B,C)&=\sens_{k,b}(A_{\alpha,\beta}).
\end{align}
Again, we use the fact that $A_{\alpha,0}\leq A_{\alpha,\beta}\leq \frac{1}{\alpha-\beta}A_{\alpha,0}$.
Now $A_{\alpha,0}$ fulfills the conditions of Theorem~\ref{thm:b-sensitivity}, so from Equation~\eqref{eq:sens_toeplitz} we know 
\begin{align}
    \sens_{k,b}(A_{\alpha,0}) &= 
    \Big\|\sum_{j = 0}^{k-1}(A_{\alpha,0})_{[\cdot, 1 + jb]}\Big\|,
    \label{eq:sens_Aalpha_0}
\end{align}

We first study~\eqref{eq:sens_Aalpha_0} for $\alpha=1$.
Then, from the explicit structure of $A_{1,0}=\toep(1,1,\dots,1)$ 
one sees that the vectors inside the norm have a block structure
\begin{align}
    \sum_{j = 0}^{k-1}(A_{1,0})_{[\cdot, 1 + jb]} &= \begin{pmatrix}v_{1}\\ \vdots\\ v_{k}\\v'\end{pmatrix}
\qquad\text{with}\qquad
v_i = \begin{pmatrix}i\\\vdots\\i\end{pmatrix}\in\R^{b}
\end{align}
for $i=1,\dots,k$, and 
$\displaystyle v' = \begin{pmatrix}k\\\vdots\\k\end{pmatrix}\in\R^{n-bk}$, appears only if $k<\frac{n}{b}$.
Now we check
\begin{align}
\|v'\|^2 + \sum_{i=0}^{k}\|v_i\|^2_2
&= (n-bk)k^2 + b\Big(\sum_{i=1}^{k}i^2\Big) 
\\
&= nk^2 - bk^3 + b\frac{k(k+1)(2k+1)}{6}
\\
&\geq nk^2 - \frac{2}{3}bk^3 \geq \frac{1}{3}nk^2,
\end{align}
because $bk\geq n$. Consequently
\begin{align}
    \sens_{k,b}(A_{1,\beta}) &\geq \frac{k\sqrt{n}}{\sqrt{3}},
\end{align}
which concludes the proof of this case.
For $\alpha<1$, $A_{\alpha,\beta}\geq\Id$ componentwise readily implies 
\begin{align}
    \sens_{k,b}(A_{\alpha,\beta}) \geq \sqrt{k},
\end{align}
which implies the statement of the theorem.
\end{proof}

\end{document}